\documentclass[twoside,11pt]{article}
\usepackage{jmlr2e}

\usepackage{epsf}
\usepackage{epsfig}
\usepackage{amsmath}
\usepackage{subfigure}
\usepackage{amssymb}
\usepackage{algorithm}
\usepackage{multirow}
\usepackage{multicol}
\usepackage{graphicx}
\usepackage{algorithm}
\usepackage{algorithmic}

\usepackage[pagebackref=true,breaklinks=true,letterpaper=true,colorlinks,bookmarks=yes]{hyperref}

\def\A{{\bf A}}
\def\a{{\bf a}}
\def\B{{\bf B}}
\def\bb{{\bf b}}
\def\C{{\bf C}}

\def\D{{\bf D}}
\def\d{{\bf d}}

\def\Hes{{\bf H}}
\def\I{{\bf I}}
\def\K{{\bf K}}
\def\k{{\bf k}}
\def\LL{{\bf L}}
\def\M{{\bf M}}

\def\PP{{\bf P}}
\def\Q{{\bf Q}}

\def\R{{\bf R}}
\def\S{{\bf S}}

\def\T{{\bf T}}
\def\U{{\bf U}}
\def\u{{\bf u}}
\def\V{{\bf V}}

\def\W{{\bf W}}

\def\X{{\bf X}}
\def\x{{\bf x}}
\def\Y{{\bf Y}}
\def\y{{\bf y}}
\def\Z{{\bf Z}}

\def\0{{\bf 0}}
\def\1{{\bf 1}}

\def\JM{{\mathcal J}}

\def\OM{{\mathcal O}}
\def\PM{{\mathcal P}}

\def\ZM{{\mathcal Z}}
\def\RB{{\mathbb R}}
\def\RBmn{{\RB^{m\times n}}}
\def\EB{{\mathbb E}}
\def\PB{{\mathbb P}}

\def\Si{\mbox{\boldmath$\Sigma$\unboldmath}}
\def\si{\mbox{\boldmath$\sigma$\unboldmath}}

\def\Lam{\mbox{\boldmath$\Lambda$\unboldmath}}

\def\De{\mbox{\boldmath$\Delta$\unboldmath}}
\def\Ome{\mbox{\boldmath$\Omega$\unboldmath}}

\def\argmin{\mathop{\rm argmin}}

\def\nnz{\mathrm{nnz}}

\def\tr{\mathrm{tr}}
\def\rk{\mathrm{rank}}

\def\st{\mathsf{s.t.}}
\def\vect{\mathsf{vec}}

\def\nystrom{{Nystr\"{o}m} }

\def\Kens{{\tilde{\K}_{t,c}^{\textrm{ens}}}}

\def\Kss{{\tilde{\K}_{c}^{\textrm{ss}}}}
\def\Kiss{{\tilde{\K}_{c}^{\textrm{iss}}}}

\jmlrheading{17}{2016}{1-49}{5/14; Revised 1/16}{5/16}{Shusen Wang, Luo Luo, and Zhihua Zhang}
\ShortHeadings{SPSD Matrix Approximation vis Column Selection}{Wang, Luo, and Zhang}
\firstpageno{1}

\begin{document}
\title{SPSD Matrix Approximation vis Column Selection: Theories, Algorithms, and Extensions}

\author{\name  Shusen Wang  \email wssatzju@gmail.com \\
        \addr Department of Statistics \\
                University of California at Berkeley \\
                Berkeley, CA 94720 \\
        \AND
        \name Luo Luo \email ricky@sjtu.edu.cn \\
        \name Zhihua Zhang\thanks{Corresponding author.} \email zhihua@sjtu.edu.cn \\
        \addr Department of Computer Science and Engineering \\
                Shanghai Jiao Tong University \\
                800 Dong Chuan Road, Shanghai, China 200240
        }

\editor{Inderjit Dhillon}

\maketitle


\begin{abstract}%
Symmetric positive semidefinite (SPSD) matrix approximation is an important problem with applications in kernel methods. However, existing SPSD matrix approximation methods such as the Nystr\"om method only have weak error bounds. In this paper we conduct in-depth studies of an SPSD matrix approximation model and establish strong relative-error bounds. We call it the prototype model for it has more efficient and effective extensions, and some of its extensions have high scalability. Though the prototype model itself is not suitable for large-scale data, it is still useful to study its properties, on which the analysis of its extensions relies.

This paper offers novel theoretical analysis, efficient algorithms, and a highly accurate extension. First, we establish a lower error bound for the prototype model and  improve the error bound of an existing column selection algorithm to match the lower bound. In this way, we obtain the first optimal column selection algorithm for the prototype model. We also prove that the prototype model is exact under certain conditions. Second, we develop a simple column selection algorithm with a provable error bound. Third, we propose a so-called spectral shifting model to make the approximation more accurate when the eigenvalues of the matrix decay slowly, and the improvement is theoretically quantified. The spectral shifting method can also be applied to improve other SPSD matrix approximation models.
\end{abstract}



\section{Introduction} \label{sec:introduction}

The kernel methods are important tools in machine learning, computer vision, and data mining \citep{scholkopf2002learning,ShaweTaylorBook:2004}.
However, for two reasons, most kernel methods have scalability difficulties.
First, given $n$ data points of $d$ dimension,
generally we need $\OM (n^2 d)$ time to form the $n\times n$ kernel matrix $\K$.
Second, most kernel methods require expensive matrix computations.
For example, Gaussian process regression and classification
require inverting some $n\times n$ matrices which costs $\OM(n^3)$ time and $\OM(n^2)$ memory;
kernel PCA and spectral clustering perform the truncated eigenvalue decomposition
which takes $\tilde\OM(n^2 k)$ time\footnote{The $\tilde\OM$ notation hides
the logarithm terms and the data-dependent spectral gap parameter.}
and $\OM(n^2)$ memory, where $k$ is the target rank of the decomposition.

Besides high time complexities, these matrix operations also have high memory cost
and are difficult to implement in distributed computing facilities.
The matrix decomposition and  (pseudo) inverse operations are generally solved by numerical iterative algorithms,
which go many passes through the matrix until convergence.
Thus, the whole matrix had better been placed in main memory,
otherwise in each iteration there would be a swap between memory and disk, which incurs high I/O costs and can be more expensive than CPU time.
Unless the algorithm is pass-efficient, that is, it goes constant passes through the data matrix,
the main memory should be at least the size of the data matrix.
For two reasons, such iterative algorithms are expensive even if they are  performed in distributed computing facilities such as MapReduce.
First, the memory cost is too expensive for each individual machine to stand.
Second, communication and synchronization must be performed in each iteration of the numerical algorithms,
so the cost of each iteration is high.

Many matrix approximation methods have been proposed to make kernel machines scalable.
Among them the Nystr\"om method \citep{nystrom1930praktische,williams2001using} and random features \citep{rahimi2008weighted}
are the most efficient and widely applied.
However, only weak results are known \citep{drineas2005nystrom,gittens2013revisiting,lopezpaz2014}.
\citet{yang2012nystrom} showed that the Nystr\"om method is likely a better choice than random features, both theoretically and empirically.
However, even the Nystr\"om method cannot attain high accuracy.
The lower bound in \citep{wang2013improving} indicates that the Nystr\"om method costs at least $\Omega (n^2 k/\epsilon)$ time
and $\Omega (n^{1.5} k^{0.5} \epsilon^{-0.5})$ memory to attain $1+\epsilon$ Frobenius norm error bound
relative to the best rank $k$ approximation.

In this paper we investigate a more accurate low-rank approximation model proposed by \cite{halko2011ramdom,wang2013improving},
which we refer to as the prototype model.
For any symmetric positive semidefinite (SPSD) matrix $\K \in \RB^{n\times n}$,
the prototype model first draws a random matrix $\PP \in \RB^{n\times c}$ and forms a sketch $\C = \K \PP$,
and then computes the intersection matrix
\begin{eqnarray} \label{eq:intersection}
\U^\star
\; = \; \argmin_{\U} \| \K - \C \U \C^T \|_F^2
\; = \; \C^\dag \K (\C^\dag)^T
\; \in \; \RB^{c\times c}.
\end{eqnarray}
Finally, the model approximates $\K$ by $\C \U^\star  \C^T$.
With this low-rank approximation at hand,
it takes time $\OM (n c^2)$ to compute the approximate matrix inversion and eigenvalue decomposition.
In the following we discuss how to form $\C$ and compute $\U^\star$.

{\bf Column Selection vs.\ Random Projection.}
Although the sketch $\C = \K \PP$ can be formed by either random projection or column selection,
when applied to the kernel methods, column selection is preferable to random projection.
As aforementioned, suppose we are given $n$ data points of $d$ dimension.
It takes time $\OM (n^2 d)$ to compute the whole of the kernel matrix $\K$,
which is prohibitive when $n$ is in million scale.
Unfortunately, whatever existing random projection technique is employed to form the sketch $\C$,
every entry of $\K$ must be visited.
In contrast, by applying data independent column selection algorithms such as uniform sampling,
we can form $\C$ by observing only $\OM (n c)$ entries of $\K$.
At present all the existing column selection algorithms, including our proposed uniform+adaptive$^2$ algorithm,
cannot avoid observing the whole of $\K$ while keeping constant-factor bound.
Nevertheless, we conjecture that our uniform+adaptive$^2$ algorithm can be adapted to satisfy these two properties simultaneously
(see Section~\ref{sec:open_problems} for discussions in detail).

{\bf The Intersection Matrix.}
With the sketch $\C$ at hand, it remains to compute the intersection matrix.
The most straightforward way is \eqref{eq:intersection}, which minimizes the Frobenius norm approximation error.
However, this approach has two drawbacks.
First, it again requires the full observation of $\K$.
Second, the matrix product $\C^\dag \K$ costs $\OM (n^2 c)$ time.
The prototype model is therefore time-inefficient.
Fortunately, \cite{wang2015towards} recently overcame the two drawbacks by solving \eqref{eq:intersection} approximately rather than optimally.
\cite{wang2015towards} obtained the approximate intersection matrix $\tilde\U$ in $\OM (n c^3 /\epsilon)$ time while keeping
\begin{eqnarray} \label{eq:wzz2015}
 \| \K - \C \tilde\U \C^T \|_F^2
\; \leq \; (1+\epsilon) \min_{\U} \| \K - \C \U \C^T \|_F^2
\end{eqnarray}
with high probability.

With the more efficient solution,
why is it useful to study the exact solution to the prototype model \eqref{eq:intersection}?
On the one hand, from \eqref{eq:wzz2015} we can see that the quality of the approximation depends on the prototype model,
thus improvement of the prototype model directly applies to the more efficient model.
On the other hand, for medium-scale problems where $\K$ does not fit in memory,
the prototype model can produce high quality approximation with reasonable time expense.
The experiment on kernel PCA in Section~\ref{sec:experiment_kpca} shows that the prototype model is far more accurate than the Nystr\"om method.
For the above two reasons, we believe the study of the prototype model is useful.


\subsection{Contributions}

Our contributions mainly include three aspects: theoretical analysis, column selection algorithms,
and extensions.
They are summarized as follows.


\subsubsection{Contributions: Theories}

\cite{kumar2009sampling,talwalkar2010matrix} previously showed that
the \nystrom method is exact when the original kernel matrix is low-rank.
In Section~\ref{sec:exact} we show that the prototype model exactly recovers the original SPSD matrix under the same conditions.

The prototype model with the near-optimal+adaptive column sampling algorithm satisfies $1{+}\epsilon$ relative-error bound when $c = \OM (k/\epsilon^2)$ \citep{wang2013improving}.
It was unknown whether this upper bound is optimal.
In Section~\ref{sec:lower_bounds} we establish a lower error bound for the prototype model.
We show that at least $2 k/ \epsilon$ columns must be chosen to attain $1+\epsilon$ bound.
In Theorem~\ref{thm:near_opt} we improve the upper error bound of the near-optimal+adaptive algorithm to $\OM(k/\epsilon)$,
which matches the lower bound up to a constant factor.


\subsubsection{Contributions: Algorithms}

In Section~\ref{sec:algorithms} we devise a simple column selection algorithm which we call {\it the uniform+adaptive$^2$ algorithm}.
The uniform+adaptive$^2$ algorithm is more efficiently and more easily implemented than the near-optimal+adaptive algorithm of \cite{wang2013improving},
yet its error bound is comparable with the near-optimal+adaptive algorithm.
It is worth  mentioning that our uniform+adaptive${}^2$ algorithm is the adaptive-full algorithm of \cite[Figure 3,][]{kumar2012sampling}
with two rounds of adaptive sampling,
and thus our results theoretically justify the adaptive-full algorithm.


\subsubsection{Contributions: Extension} \label{sec:introduction_extension}

When the spectrum of a matrix decays slowly (that is, the $c+1$ to $n$ largest eigenvalues are not small enough),
all of the low-rank approximations are far from the original kernel matrix.
Inspired by \cite{zhang2014mra},
we propose a new method called {\it spectral shifting (SS)}
to make the approximation still effective even when the spectrum decays slowly.
Unlike the low-rank approximation $\K \approx \C \U \C^T$,
the spectral shifting model approximates $\K$ by $\K \approx \bar{\C} \U^{\textrm{ss}} \bar{\C}^T + \delta^{\textrm{ss}} \I_n$,
where $\C, \bar{\C} \in \RB^{n\times c}$, $\U, {\U^{\textrm{ss}}} \in \RB^{c\times c}$, and $\delta^{\textrm{ss}} \geq 0$.
When the spectrum of $\K$ decays slowly, the term $\delta^{\textrm{ss}} \I_n$ helps to improve the approximation accuracy.
In Section~\ref{sec:sspbs} we describe the spectral shifting method in detail.

We highlight that the spectral shifting method can naturally apply to improve other kernel approximation models such as
the memory efficient kernel approximation (MEKA) model \citep{si2014memory}.
Experiments demonstrate that MEKA can be significantly improved by spectral shifting.


\subsection{Paper Organization}

The remainder of this paper is organized as follows.
In Section~\ref{sec:notation} we define the notation.
In Section~\ref{sec:matrix_sketching_model} we introduce the motivations of SPSD matrix approximation and define the SPSD matrix approximation models.
Then we present our work---theories, algorithms, and extension---respectively in Sections~\ref{sec:theory}, \ref{sec:algorithms}, and \ref{sec:sspbs}.
In Section \ref{sec:algorithm:experiments} we conduct experiments to compare among the column sampling algorithms.
In Section~\ref{sec:sspbs:experiments} we empirically evaluate the proposed spectral shifting model.
All the proofs  are deferred to the appendix.


\section{Notation} \label{sec:notation}

Let $[n]=\{1, \ldots, n\}$, and $\I_n$ be the $n{\times}n$ identity matrix.
For an $m {\times} n$ matrix $\A=[a_{i j}]$, we let $\a_{i:}$ be its $i$-th row,
$\a_{:i}$ be its $i$-th column,
and  use $\a_i$ to denote either row or column when there is no ambiguity.
Let $\A_1 \oplus \A_2 \oplus \cdots \oplus \A_q$ be the block diagonal matrix whose the $i$-th diagonal block is $\A_i$.
Let $\|\A\|_F = (\sum_{i,j} a_{i j}^2)^{1/2}$ be the Frobenius norm and
$\|\A\|_2 = \max_{\x\neq \0} \|\A \x\|_2 / \|\x\|_2$ be the spectral norm.

Letting $\rho=\rk(\A)$, we write the condensed singular value decomposition (SVD) of $\A$ as $\A = \U_\A \Si_\A \V_\A^T$,
where the $(i,i)$-th entry of $\Si_\A \in \RB^{\rho \times \rho}$ is the $i$-th largest singular value of $\A$ (denoted  $\sigma_i(\A)$).
Unless otherwise specified, in this paper ``SVD'' means the condensed SVD.
We also let $\U_{\A,k}$ and $\V_{\A,k}$ be the first $k$ ($<\rho$) columns of $\U_\A$ and $\V_\A$, respectively,
and $\Si_{\A,k}$ be the $k\times k$ top sub-block of $\Si_\A$.
Then the $m\times n$ matrix $\A_k=\U_{\A,k} \Si_{\A,k} \V_{\A,k}^T$ is the ``closest'' rank-$k$ approximation to $\A$.

If $\A$ is normal, we let $\A = \U_\A \Lam_\A \U_\A^T$ be the eigenvalue decomposition,
and denote the $i$-th diagonal entry of $\Lam_\A$ by $\lambda_i (\A)$,
where $|\lambda_1 (\A)| \geq \cdots \geq |\lambda_n (\A)|$.
When $\A$ is SPSD, the SVD and the eigenvalue decomposition of $\A$ are identical.

Based on SVD, the {\it matrix coherence} of the columns of $\A$ relative to the best rank-$k$ approximation
is defined as $\mu_k = \frac{n}{k} \max_{j} \big\| (\V_{\A,k})_{j:} \big\|_2^2$.
Let $\A^\dag = \V_{\A} \Si_\A^{-1} \U_\A^T$ be the {\it Moore-Penrose inverse} of $\A$.
When $\A$ is nonsingular, the Moore-Penrose inverse is identical to the matrix inverse.
Given another ${n\times c}$ matrix $\C$,
we define $\PM_\C (\A) = \C \C^\dag \A$
as the projection of $\A$ onto the column space of $\C$
and $\PM_{\C,k} (\A) = \C \cdot \argmin_{\rk(\X)\leq k} \|\A - \C \X\|_F$
as the rank restricted projection.
It is obvious that $\|\A - \PM_\C (\A) \|_F \leq \|\A - \PM_{\C,k} ( \A)\|_F$.


\section{SPSD Matrix Approximation Models} \label{sec:matrix_sketching_model}

In Section \ref{sec:motivation} we provide motivating examples to show why SPSD matrix approximation is useful.
In Section~\ref{sec:models} we formally describe  low-rank approximation models.
In Section~\ref{sec:models_ss} we describe the spectral shifting model.
In Table~\ref{tab:models} we compare the matrix approximation models defined in Section~\ref{sec:models} and Section~\ref{sec:models_ss}.

\begin{table}[t]\setlength{\tabcolsep}{0.3pt}
\caption{Comparisons among the matrix approximation models in Section~\ref{sec:models} and Section~\ref{sec:models_ss}.
        Here ``\#Entries'' denotes the number of entries of $\K$ required to observe.
        The costs of column selection is not counted;
        they are listed separately in Table~\ref{tab:algorithms}.}
\label{tab:models}
\begin{center}
\begin{tabular}{p{80pt} p{250pt} p{250pt} p{250pt} p{250pt}}
\hline
        &	\multicolumn{1}{c}{Time~~~}
        & \multicolumn{1}{c}{~~~Memory~~~}
        & \multicolumn{1}{c}{~~~\#Entries~~~}
        & \multicolumn{1}{c}{~~~Theory~~~}  \\
\hline
Prototype   & \multicolumn{1}{c}{$\OM( n^2 c )$~~~}
            & \multicolumn{1}{c}{~~~$\OM( n c)$~~~ }
            & \multicolumn{1}{c}{~~~$ n^2 $~~~ }
            & \multicolumn{1}{c}{~~~$1+\epsilon$ relative-error }  \\
Faster      & \multicolumn{1}{c}{$\OM( n c^3 /\epsilon )$~~~}
            & \multicolumn{1}{c}{~~~$\OM( n c)$~~~ }
            & \multicolumn{1}{c}{~~~$ n c^2 /\epsilon $~~~ }
            & \multicolumn{1}{c}{~~~$1+\epsilon$ relative-error}  \\
Nystr\"om   & \multicolumn{1}{c}{$\OM( n c^2 )$~~~}
            & \multicolumn{1}{c}{~~~$\OM( n c)$~~~ }
            & \multicolumn{1}{c}{~~~$ n c $~~~ }
            & \multicolumn{1}{c}{~~~weak }  \\
SS          & \multicolumn{3}{c}{the same to ``prototype''}
            & \multicolumn{1}{c}{~~~stronger than ``prototype'' }  \\
Faster SS   & \multicolumn{3}{c}{the same to ``faster''}
            & \multicolumn{1}{c}{unknown}  \\
\hline
\end{tabular}
\end{center}
\end{table}


\subsection{Motivations} \label{sec:motivation}

Let $\K$ be an $n\times n$ kernel matrix.
Many kernel methods require the eigenvalue decomposition of $\K$ or solving certain linear systems involving $\K$.
\begin{itemize}
\item
    Spectral clustering,
    kernel PCA, and manifold learning
    need to perform the rank $k$ eigenvalue decomposition which costs $\tilde\OM (n^2 k)$ time and $\OM (n^2)$ memory.
\item
    Gaussian process regression and
 classification
    both require solving this kind of linear systems:
    \begin{equation} \label{eq:linear_system}
        (\K + \alpha \I_n) \bb = \y ,
    \end{equation}
    whose solution is $\bb^\star = (\K + \alpha \I_n)^{-1} \y$.
    Here $\alpha$ is a constant.
    This costs $\OM (n^3)$ time and $\OM (n^2)$ memory.
\end{itemize}
Fortunately, if we can efficiently find an approximation in the form
\[
\tilde\K \; =\; \LL \LL^T + \delta \I_n
\; \approx \; \K,
\]
where $\delta \geq 0$ and $\LL \in \RB^{n\times l}$ with $l \ll n$,
then the eigenvalue decomposition and linear systems can be approximately solved in $\OM (n l^2)$ time and $\OM (n l)$ space in the following way.
\begin{itemize}
\item
{Approximate Eigenvalue Decomposition.}
Let $\LL = \U \Si \V^T$ be the SVD and $\U_\perp$ be the orthogonal complement of $\U$.
Then the full eigenvalue decomposition of $\tilde\K$ is
\[
\tilde\K \; = \; \U (\Si^2 + \delta \I_l) \U^T + \U_\perp (\delta \I_{n-l}) \U_\perp^T .
\]
\item
{Approximately Solving the Linear Systems.}
Here we use a more general form: $\tilde\K = \LL \LL^T + \De$,
where $\De$ is a diagonal matrix with positive diagonal entries.
Then
\begin{eqnarray*}
\bb^\star
& = & (\K + \alpha \I_n)^{-1} \y
\; \approx \; (\LL \LL^T + \De + \alpha \I_n)^{-1} \y
\; = \; (\LL \LL^T + \De')^{-1} \y \\
& = & {\De'}^{-1}\y - \underbrace{{\De'}^{-1} \LL}_{n\times l} \underbrace{( \I_l + \LL^T {\De'}^{-1} \LL )^{-1}}_{l\times l} \underbrace{\LL^T {\De'}^{-1}}_{l\times n} \y.
\end{eqnarray*}%
Here the second equality is obtained by letting $\De' = \De + \alpha  \I_n$,
and the third equality follows by the Sherman-Morrison-Woodbury matrix identity.
\end{itemize}
The remaining problem is to find such matrix approximation efficiently while keeping $\tilde\K$ close to $\K$.

\begin{algorithm}[tb]
   \caption{Computing the Prototype Model in $\OM (nc + nd)$ Memory.}
   \label{alg:memory_cost}
\algsetup{indent=2em}
\begin{algorithmic}[1]
   \STATE {\bf Input:} data points $\x_1 , \cdots , \x_n \in \RB^d$, kernel function $\kappa (\cdot , \cdot )$.
   \STATE Compute $\C$ and $\C^\dag$; {\it // In $\OM (nc + nd)$ memory and $\OM (ncd+nc^2)$ time}
   \STATE Form a $c\times n$ all-zero matrix $\D$; {\it // In $\OM (nc )$ memory and $\OM (nc)$ time}
   \FOR{$j = 1$ to $n$ }
   \STATE Form the $j$-th column of $\K$ by $\k_j = [\kappa (\x_1, \x_j) , \cdots, \kappa (\x_n, \x_j)]^T$;
   \STATE Compute the $j$-th column of $\D$ by $\d_j = \C^\dag \k_j$;
   \STATE Delete $\k_j$ from memory;
   \ENDFOR
   \STATE {\it // Now the matrix $\D$ is $\C^\dag \K$}
   \STATE {\it // The loop totaly costs $\OM (n c + nd)$ memory and $\OM (n^2 d + n^2 c)$ time}
   \STATE Compute $\U = \D (\C^\dag)^T$; {\it // In $\OM (n c)$ memory and $\OM (nc^2)$ time}
   \RETURN $\C$ and $\U$ ($=\C^\dag \K (\C^\dag)^T$).
\end{algorithmic}
\end{algorithm}


\subsection{Low-Rank Matrix Approximation Models} \label{sec:models}

We first recall the prototype model introduced previously and then discuss its approximate solutions.
In fact, the famous Nystr\"om method \citep{nystrom1930praktische,williams2001using} is an approximation to the prototype model.
Throughout this paper, we let $\PP \in \RB^{n\times c}$ be random projection or column selection matrix and
$\C = \K \PP$ be a sketch of $\K$.
The only difference among the discussed  models is  their intersection matrices.

{\bf The Prototype Model.}
Suppose we have $\C \in \RB^{n\times c}$ at hand.
It remains to find an intersection matrix $\U \in \RB^{c\times c}$.
Since our objective is to make $\C \U \C^T$ close to $\K$,
it is straightforward to optimize their difference.
The prototype model computes the intersection matrix by
\begin{eqnarray} \label{eq:intersection_proto}
\U^\star
\; = \; \argmin_{\U} \| \K - \C \U \C^T \|_F^2
\; = \; \C^\dag \K (\C^T)^\dag
\; \in \; \RB^{c\times c}.
\end{eqnarray}
With $\C$ at hand, the prototype model still needs one pass through the data,
and it costs $\OM (n^2 c)$  time.
When applied to kernel methods, the memory cost is $\OM (n c + nd)$ (see Algorithm~\ref{alg:memory_cost}),
where $n$ is the number of data points and $d$ is the dimension.
The prototype model has the same time complexity as the exact rank $k$ eigenvalue decomposition,
but it is more memory-efficient and pass-efficient.

\citet{halko2011ramdom} showed that when $\PP$ is a standard Gaussian matrix and $c = \OM (k/\epsilon)$,
the prototype model attains $2+\epsilon$ error relative to $\|\K - \K_k\|_F^2$.
\cite{wang2013improving} showed that when $\C$ contains $c = \OM (k/\epsilon^2)$ columns selected by the near-optimal+adaptive sampling algorithm,
the prototype model attains  $1+\epsilon$ relative error.
In Section~\ref{sec:algorithms:improved} we improve the result to $c = \OM (k/\epsilon)$, which is near optimal.

{\bf Faster SPSD matrix Approximation Model.}
\cite{wang2015towards} noticed that \eqref{eq:intersection_proto} is a strongly over-determined linear system,
and thus proposed to solve \eqref{eq:intersection_proto} by randomized approximations.
They proposed to sample $s = \OM(c \sqrt{n/\epsilon}) \ll n$ columns according to the row leverage scores of $\C$, which costs $\OM (n c^2)$ time.
Let $\S \in \RB^{n\times s}$ be the corresponding column selection matrix.
They proposed the faster SPSD matrix approximation model which computes the intersection matrix by
\begin{eqnarray} \label{eq:intersection_faster}
\tilde\U
\; = \; \argmin_{\U} \big\| \S^T ( \K - \C \U \C^T ) \S \big\|_F^2
\; = \; \underbrace{(\S^T\C)^\dag}_{c\times s} \underbrace{(\S^T \K \S)}_{s\times s} \underbrace{(\C^T \S)^\dag }_{s\times c}
\; \in \; \RB^{c\times c}.
\end{eqnarray}
The faster model visits only $s^2 = \OM (n c^2 /\epsilon) = o( n^2)$ entries of $\K$,
and the time  complexity is $\OM (n c^2 + s^2 c) = \OM (n c^3 /\epsilon)$.
The following error bound is satisfied with  high  probability:
\[
 \| \K - \C \tilde\U \C^T \|_F^2
\; \leq \; (1+\epsilon) \: \min_{\U} \| \K - \C \U \C^T \|_F^2.
\]
This implies that if $\C$ is such a high quality sketch that the prototype model satisfies $1+\epsilon$ relative-error bound,
then the faster SPSD matrix approximation model also satisfies $1+\epsilon$ relative-error bound.

{\bf The Nystr\"om Method.}
The Nystr\"om method is a special case of the faster SPSD matrix  approximation model,
and therefore it is also an approximate solution  to \eqref{eq:intersection_proto}.
If we let the two column selection matrices  $\S$ and $\PP$ be the same,
then \eqref{eq:intersection_faster}  becomes
\begin{eqnarray*}
\tilde\U
\; = \; \argmin_{\U} \big\| \PP^T ( \K - \C \U \C^T ) \PP \big\|_F^2
\; = \; \underbrace{(\PP^T\K \PP)^\dag}_{c\times c} \underbrace{(\PP^T \K \PP)}_{c\times c} \underbrace{(\PP^T \K \PP)^\dag }_{c\times c}
\; = \; \underbrace{(\PP^T \K \PP)^\dag }_{c\times c}.
\end{eqnarray*}
The matrix $(\PP^T \K \PP)^\dag$ is exactly the intersection matrix of the Nystr\"om method.
The Nystr\"om method costs only $\OM (nc^2)$ time, and it can be applied to million-scale problems \citep{talwalkar2013large}.
However, its accuracy is low.
Much work in the literature has analyzed the error bound of the Nystr\"om method,
but only weak results are known \citep{drineas2005nystrom,Shawe-taylor05onthe,kumar2012sampling,jin2012improved,gittens2013revisiting}.
\citet{wang2013improving} even showed that the Nystr\"om method cannot attain $1+\epsilon$ relative-error bound unless $c \geq \Omega(\sqrt{n k/\epsilon})$.
Equivalently, to attain $1+\epsilon$ bound, the Nystr\"om would take $\Omega (n^2 k /\epsilon)$ time and $\Omega (n^{1.5} k^{0.5} \epsilon^{-0.5})$ memory.


\subsection{Spectral Shifting Models} \label{sec:models_ss}

We propose a more accurate SPSD matrix approximation method called the spectral shifting model.
Here we briefly describe the model and its fast solution.
The theoretical analysis is left to Section~\ref{sec:sspbs}.

{\bf The Spectral Shifting (SS) Model.}
As before, we let $\PP \in \RB^{n\times c}$ be a column selection matrix and $\bar\C = \bar\K \PP$,
where $\bar\K = \K$ or $\bar\K = \K - \bar\delta \I_n$ for some parameter $\bar\delta \geq 0$.
We approximate $\K$ by $\bar\C \U^{\textrm{ss}} \bar\C^T + \delta^{\textrm{ss}}  \I_n$, where
\begin{equation}  \label{eq:def_ss_nystrom0}
\big(\U^{\textrm{ss}} , \delta^{\textrm{ss}} \big)
\; = \;
\argmin_{\U, \delta} \big\| \K - \bar{\C} \U \bar{\C}^T - \delta \I_n \big\|_F^2.
\end{equation}
This optimization problem has closed-form solution (see Theorem~\ref{thm:ss_closed_form})
\begin{eqnarray*}
\delta^{\textrm{ss}} & = & \frac{1}{n- \rk(\bar{\C})}\Big( \tr(\K) - \tr\big({\bar{\C}}^\dag\K {\bar{\C}} \big) \Big), \\
\U^{\textrm{ss}} & = & {\bar{\C}}^\dag \K ( {\bar{\C}}^\dag)^T - \delta^{\textrm{ss}}({\bar{\C}}^T {\bar{\C}})^{\dag},
\end{eqnarray*}
which can be computed in $\OM (n^2 c)$ time and $\OM (n c)$ memory.
Later we will show that the SS model is more accurate than the prototype model.

{\bf Faster Spectral Shifting Model.}
The same idea of \cite{wang2015towards} also applies to the SS model \eqref{eq:def_ss_nystrom}.
Specifically, we can draw another  column selection matrix $\S \in \RB^{n\times s}$  and solve
\begin{eqnarray*}
\big(\tilde\U^{\textrm{ss}} , \tilde\delta^{\textrm{ss}} \big)
& = &
\argmin_{\U, \delta} \big\| \S^T \big( \K - \bar{\C} \U \bar{\C}^T - \delta \I_n \big) \S \big\|_F^2 \\
& = &
\argmin_{\U, \delta} \big\| \S^T \K \S - (\S^T \bar{\C}) \U (\S^T \bar{\C})^T  - \delta \I_s \big\|_F^2 .
\end{eqnarray*}
Similarly, it has closed-form solution
\begin{eqnarray*}
\tilde\delta^{\textrm{ss}}
& = & \frac{1}{s- \rk(\S^T \bar{\C})}\Big[ \tr \big(\S^T \K \S \big)
    - \tr \Big((\S^T \bar{\C})^\dag (\S^T \K \S) (\S^T \bar{\C}) \Big)  \Big], \nonumber \\
\tilde\U^{\textrm{ss}}
& = & (\S^T \bar{\C})^\dag (\S^T \K \S) ( {\bar{\C}}^T \S)^\dag - \tilde\delta^{\textrm{ss}}({\bar{\C}}^T \S \S^T{\bar{\C}})^{\dag}.
\end{eqnarray*}
In this way, the time cost is merely $\OM (s^2 c)$.
However,  the theoretical properties of this model are yet unknown.
We do not conduct theoretical or empirical study of this  model;
we leave it as a  future work,


\section{Theories} \label{sec:theory}

In Section~\ref{sec:exact} we show  that the prototype model is exact when $\K$ is low-rank.
In Section~\ref{sec:lower_bounds} we provide a lower  error bound of the prototype model.


\subsection{Theoretical Justifications} \label{sec:exact}

Let $\PP$ be a column selection matrix, $\C = \K \PP$ be a sketch, and $\W = \PP^T \K \PP$ be the corresponding submatrix.
\cite{kumar2009sampling,talwalkar2010matrix} showed that the \nystrom method is exact when $\rk(\W) = \rk(\K)$.
We present a similar result in Theorem~\ref{thm:exact}.

\begin{theorem}\label{thm:exact}
The following three statements are equivalent:
\emph{(i)} $\rk(\W)=\rk(\K)$,
\emph{(ii)} $\K = \C \W^\dag \C^T$,
\emph{(iii)} $\K = \C \C^\dag \K (\C^\dag)^T \C^T$.
\end{theorem}

Theorem~\ref{thm:exact} implies that the prototype model and the \nystrom method are equivalent when $\rk(\W) = \rk(\K)$;
that is, the kernel matrix $\K$ is low rank.
However, it holds in general that $\rk(\K) \gg c \geq \rk(\W)$, where the two models are not equivalent.


\subsection{Lower Bound} \label{sec:lower_bounds}

\citet{wang2013improving}  showed that with $c = \OM (k/\epsilon^2)$ columns chosen by the near-optimal+adaptive sampling algorithm,
the prototype model satisfies $1+\epsilon$ relative-error bound.
We establish a lower error bound in Theorem~\ref{thm:lower},
which shows that at least $c \geq 2 k \epsilon^{-1}$ columns must be chosen to attain the $1+\epsilon$ bound.
This indicates there exists a gap between the upper bound in \citep{wang2013improving} and our lower bound,
and thus  there is room of improvement.
The  proof of Theorem~\ref{thm:lower} is left to Appendix~\ref{sec:app:lower_bound}.

\begin{theorem}[Lower Bound of the Prototype Model] \label{thm:lower}
Whatever column sampling algorithm is used,
there exists an $n\times n$ SPSD matrix $\K$ such that the error incurred by the prototype model obeys:
\begin{eqnarray}
\big\| \K - \C \U^{\star} \C^T \big\|_F^2 & \geq & \frac{n-c}{n-k} \Big( 1+ \frac{2 k}{c} \Big) \|\K - \K_k\|_F^2. \nonumber
\end{eqnarray}
Here $k$ is an arbitrary target rank, $c$ is the number of selected columns,
and $\U^{\star} = \C^\dag \K (\C^\dag)^T$.
\end{theorem}


\section{Column Sampling Algorithms} \label{sec:algorithms}

In Section~\ref{sec:algorithms:related} we introduce the column sampling algorithms in the literature.
In Section~\ref{sec:algorithms:improved} we improve the bound of the near-optimal+adaptive sampling algorithm \citep{wang2013improving},
and the obtained upper bound matches the lower bound up to a constant factor.
In Section~\ref{sec:algorithm:efficient} we develop a more efficient column sampling algorithm which we call the uniform+adaptive$^2$ algorithm.
In Section~\ref{sec:open_problems} we discuss the possibility of making uniform+adaptive$^2$ more scalable.


\subsection{Related Work} \label{sec:algorithms:related}

Column selection is an important matrix sketching approach that enables expensive matrix computations to be performed on much a smaller matrix.
The column selection problem has been widely studied in the theoretical computer science community
\citep{boutsidis2011NOCshort,mahoney2011ramdomized,Guruswami2012optimal,woodruff2014sketching}
and the numerical linear algebra community \citep{gu1996efficient,stewart1999four},
and numerous algorithms have been devised and analyzed.
Here we focus on some provable algorithms studied in the theoretical computer science community.

The adaptive sampling algorithm devised by \citet{deshpande2006matrix} (see Algorithm~\ref{alg:adaptive}) is the most relevant to this paper.
The adaptive sampling algorithm has strong error bound \citep{deshpande2006matrix,wang2013improving,boutsidis2011NOCshort}
and good empirical performance \citep{kumar2012sampling}.
Particularly, \cite{wang2013improving} proposed an algorithm that combines the near-optimal column sampling algorithm \citep{boutsidis2011NOCshort}
and the adaptive sampling algorithm \citep{deshpande2006matrix}.
They showed that by selecting $c = \OM( k \epsilon^{-2} )$
columns of $\K$ to form $\C$, it holds that
\begin{equation}
\EB \big\|\K - \C \big( \C^\dag \K (\C^\dag)^T \big) \C^T \big\|_F
\;\leq\; (1+\epsilon) \|\K - \K_k \|_F . \nonumber
\end{equation}
This error bound was the tightest among all the feasible algorithms for SPSD matrix approximation.

\begin{algorithm}[t]
   \caption{The Adaptive Sampling Algorithm.}
   \label{alg:adaptive}
\algsetup{indent=2em}
\begin{algorithmic}[1]
   \STATE {\bf Input:} a residual matrix $\B \in \RB^{n\times n}$ and number of selected columns $c$ $(< n)$.
   \STATE Compute sampling probabilities $p_j = \|\bb_{:j} \|_2^2 / \|\B\|_F^2$ for $j=1, \cdots , n$;
   \STATE Select $c$ indices in $c$ i.i.d.\ trials, in each trial the index $j$ is chosen with probability $p_j$;
   \RETURN an index set containing the indices of the selected columns.
\end{algorithmic}
\end{algorithm}


\subsection{Near Optimal Column Selection for SPSD Matrix Approximation} \label{sec:algorithms:improved}

The error bound of near-optimal+adaptive
can be improved by a factor of $\epsilon$ by exploiting the latest results of \cite{boutsidis2014optimal}.
Using the same algorithm except for different $c_2$ (i.e.\ the number of columns selected by adaptive sampling),
we obtain the following stronger theorem.
Recall from Theorem~\ref{thm:lower} that the lower bound is $c \geq \Omega \big(2 k \epsilon^{-1} (1+o(1) ) \big)$.
Thus the near-optimal+adaptive algorithm is optimal up to a constant factor.
The proof of the theorem is left to Appendix~\ref{thm:near_opt}.

\begin{theorem}[Near-Optimal+Adaptive]
\label{thm:near_opt}
Given a symmetric matrix $\K \in \RB^{n\times n}$ and a target rank $k$,
the algorithm samples totally
$c = 3 k \epsilon^{-1} \big(1+o(1)\big)$
columns of $\K$ to construct the approximation.
Then
\begin{equation}
\EB \big\|\K - \C \big( \C^\dag \K (\C^\dag)^T \big) \C^T \big\|_F
\;\leq\; (1+\epsilon) \|\K - \K_k \|_F . \nonumber
\end{equation}
The algorithm costs $\OM\big( n^2 c+ n k^3 \epsilon^{-2/3} \big)$ time and $\OM(n c)$ memory in
computing $\C$.
\end{theorem}

Despite its optimal error bound, the near-optimal+adaptive algorithm lacks of practicality.
The implementation is complicated and difficult.
Its main component---the near-optimal algorithm \citep{boutsidis2011NOCshort}---is highly iterative and therefore not suitable for parallel computing.
Every step of the near-optimal algorithm requires the full observation of $\K$ and there is no hope to avoid this.
Thus we propose to use uniform sampling to replace the near-optimal algorithm.
Although the obtained uniform+adaptive$^2$ algorithm also has quadratic time complexity and requires the full observation of $\K$,
there may be some way to making it more efficient. See the discussions in Section~\ref{sec:open_problems}.

\begin{algorithm}[t]
   \caption{The Uniform+Adaptive${}^2$ Algorithm.}
   \label{alg:uniform_adaptive2}
\algsetup{indent=2em}
\begin{algorithmic}[1]
   \STATE {\bf Input:} an $n\times n$ symmetric matrix $\K$, target rank $k$, error parameter $\epsilon \in (0, 1]$, matrix coherence $\mu$.
   \STATE {\bf Uniform Sampling.} Uniformly sample
            \vspace{-1mm}
            \[c_1 = 20 \mu k \log \big(20 k \big)
            \vspace{-1mm}
            \]
            columns of $\K$ without replacement to construct $\C_1$; \label{alg:practical:uniform}
   \STATE {\bf Adaptive Sampling.} Sample
            \vspace{-1mm}
            \[c_2 = 17.5 k / \epsilon
            \vspace{-1mm}
            \]
            columns of $\K$ to construct $\C_2$
            using the adaptive sampling algorithm (Algorithm \ref{alg:adaptive}) according to the residual $\K - \PM_{\C_1} (\K)$; \label{alg:practical:adapt1}
   \STATE {\bf Adaptive Sampling.} Sample
            \vspace{-1mm}
            \[c_3 = 10 k / \epsilon
            \vspace{-1mm}
            \] columns of $\K$ to construct $\C_3$ \label{alg:practical:adapt2}
            using the adaptive sampling algorithm (Algorithm \ref{alg:adaptive}) according to the residual $\K - \PM_{[\C_1,\;\C_2]} ( \K)$;
   \RETURN $\C = [\C_1 , \C_2 , \C_3]$.
\end{algorithmic}
\end{algorithm}


\subsection{The Uniform+Adaptive$^2$ Column Sampling Algorithm} \label{sec:algorithm:efficient}

In this paper we propose a column sampling algorithm
which is efficient, effective, and very easy to be implemented.
The algorithm consists of a uniform sampling step and two adaptive sampling steps, so we call it {\it the uniform+adaptive${}^2$ algorithm}.
The algorithm is described in Algorithm~\ref{alg:uniform_adaptive2} and analyzed in Theorem~\ref{thm:efficient}.
The proof is left to Appendix~\ref{sec:app:thm:efficient}.

It is worth  mentioning that our uniform+adaptive${}^2$ algorithm is a special instance of
the adaptive-full algorithm of \cite[Figure 3]{kumar2012sampling}.
The adaptive-full algorithm consists of random initialization and multiple adaptive sampling steps.
Using multiple adaptive sampling steps can surely reduce the approximation error.
However, the update of sampling probability in each step is expensive, so we choose to do only two steps.
The adaptive-full algorithm of \cite[Figure 3]{kumar2012sampling} is merely a heuristic scheme without theoretical guarantee;
our result provides theoretical justification for the adaptive-full algorithm.

\begin{theorem}[Uniform+Adaptive${}^2$] \label{thm:efficient}
Given an $n {\times} n$ symmetric matrix $\K$ and a target rank $k$,  let $\mu_k$ denote the matrix coherence of $\K$.
Algorithm~\ref{alg:uniform_adaptive2} samples totally
\[
c = \OM\big( k \epsilon^{-1} + \mu_k  k \log k \big)
\]
columns of $\K$ to construct the approximation.
The error bound
\begin{eqnarray}
\big\|\K - \C \big( \C^\dag \K (\C^{\dag})^T \big) \C^T \big\|_F
\;\leq \; \big(1 + \epsilon \big) \big\| \K - \K_k \big\|_F \nonumber
\end{eqnarray}
holds with probability at least $0.7$.
The algorithm costs $\OM(n^2 c ) $ time and $\OM(n c)$ space in
computing $\C$.
\end{theorem}

\begin{remark}\label{remark:efficient}
Theoretically, Algorithm~\ref{alg:uniform_adaptive2} requires computing the matrix coherence of $\K$
in order to determine $c_1$ and $c_2$.
However, computing the matrix coherence is as hard as computing the truncated SVD;
even the fast approximation approach of \cite{drineas2012fast} is not feasible here because $\K$ is a square matrix.
The use of the matrix coherence here is merely for theoretical analysis;
setting the parameter $\mu$ in Algorithm~\ref{alg:uniform_adaptive2} to be exactly the matrix coherence does not certainly result in the highest accuracy.
Empirically, the resulting approximation accuracy is not sensitive to the value of $\mu$.
Thus we  suggest   setting $\mu$ in Algorithm~\ref{alg:uniform_adaptive2} as a constant (e.g.\ $1$),
rather than actually computing the matrix coherence.
\end{remark}

\begin{table}[t]\setlength{\tabcolsep}{0.3pt}
\caption{Comparisons between the two sampling algorithms.}
\label{tab:algorithms}
\begin{center}
\begin{tabular}{p{80pt} p{200pt} p{200pt}}
\hline
        &	\multicolumn{1}{c}{Uniform+Adaptive$^2$ ~~}	& \multicolumn{1}{c}{~~Near-Optimal+Adaptive}  \\
\hline
Time    & \multicolumn{1}{c}{~~~$\OM( n^2 c )$~~~}
        & \multicolumn{1}{c}{~~~$\OM( n^2 c  +n k^3 \epsilon^{-2/3} )$ }  \\
Memory  & \multicolumn{1}{c}{$\OM\big( n c \big)$}
        & \multicolumn{1}{c}{$\OM\big( n c \big)$} \\
\#Passes& \multicolumn{1}{c}{$2$}
        & \multicolumn{1}{c}{$4$} \\
\#Columns & \multicolumn{1}{c}{$\OM\big( k \epsilon^{-1}+ \mu_k  k \log k \big)$}
        & \multicolumn{1}{c}{$\OM\big( k \epsilon^{-1} )$}\\
Implement  & \multicolumn{1}{c}{Easy to implement}
        & \multicolumn{1}{c}{Hard to implement}  \\
\hline
\end{tabular}
\end{center}
\end{table}

Table~\ref{tab:algorithms} presents comparisons between the near-optimal+adaptive algorithm and our uniform+adaptive$^2$ algorithm
over the time cost, memory cost, number of passes through $\K$,
the number of columns required to attain $1+\epsilon$ relative-error bound, and the hardness of implementation.
Our algorithm is more time-efficient and pass-efficient than the near-optimal+adaptive algorithm,
and the memory costs of the two algorithms are the same.
To attain the same error bound, our algorithm needs to select $c=\OM\big( k \epsilon^{-1}+ \mu_k  k \log k \big)$ columns,
which is a little larger than that of the near-optimal+adaptive algorithm.


\subsection{Discussions} \label{sec:open_problems}

The two algorithms discussed in the above have strong theoretical guarantees,
but they are not efficient enough for large-scale applications.
First, their time complexities are quadratic in $n$.
Second, they require the full observation of $\K$.
In fact, at present no existing column selection algorithm satisfies the three properties simultaneously:
    \vspace{-2mm}
\begin{enumerate}
\item
    the time and memory costs are $\OM (n)$;
    \vspace{-2mm}
\item
    only $\OM (n)$ entries of $\K$ need to be observed;
    \vspace{-2mm}
\item
    relative-error bound holds in expectation or with high probability.
    \vspace{-2mm}
\end{enumerate}
It is interesting to find such an algorithm,
and it remains an open problem.

Nevertheless, uniform+adaptive$^2$ is a promising column selection algorithm
for it may be adapted to satisfy the above three properties.
The drawback of uniform+adaptive$^2$ is that computing the adaptive sampling probability costs quadratic time and requires the full observation of $\K$.
There may be remedies for this problem.
The adaptive-partial algorithm in \citep{kumar2012sampling} satisfies the first two properties,
but it lacks theoretical analysis.
Another possibility is to first uniformly sample $o(n)$ columns and then down-sample to $\OM (k/\epsilon)$ columns by adaptive sampling,
which we describe in Algorithm~\ref{alg:incomplete_uniform_adaptive2}.
In this way, the first two properties can be satisfied, and it may be theoretically explained under the incoherent matrix assumption.
We do not implement such heuristics for their theoretical property is completely unknown and they are beyond the scope of this paper.

\begin{algorithm}[t]
   \caption{The Incomplete Uniform+Adaptive${}^2$ Algorithm.}
   \label{alg:incomplete_uniform_adaptive2}
\algsetup{indent=2em}
\begin{algorithmic}[1]
   \STATE {\bf Input:} part of an $n\times n$ symmetric matrix $\K$.
   \STATE {\bf Uniform Sampling.} Uniformly sample $c_1$ columns of $\K$ without replacement to construct $\C_1 $;
   \STATE {\bf Adaptive Sampling.} Uniformly sample $o (n)$ columns of $\K$ to form $\K'$;
            then sample $c_2 $ columns of $\K'$ to construct $\C_2$
            using the adaptive sampling algorithm (Algorithm \ref{alg:adaptive}) according to the residual $\K' - \PM_{\C_1} (\K')$;
   \STATE {\bf Adaptive Sampling.} Uniformly sample $o (n)$ columns of $\K$ to form $\K''$;
            then sample $c_3 $ columns of $\K''$ to construct $\C_3$
            using the adaptive sampling algorithm (Algorithm \ref{alg:adaptive}) according to the residual $\K'' - \PM_{[\C_1,\;\C_2]} ( \K'' )$;
   \RETURN $\C = [\C_1 , \C_2 , \C_3]$.
\end{algorithmic}
\end{algorithm}

\begin{table*}[t]\setlength{\tabcolsep}{0.3pt}
\caption{A summary of the datasets for kernel approximation.}
\label{tab:datasets}
\begin{center}
\begin{small}
\begin{tabular}{l c c c c c c }
\hline
	{\bf Dataset}  &~~MNIST~~&~~Letters~&~PenDigit~&~Cpusmall~&~Mushrooms\\
	\hline
    {\bf \#Instance}  & $60,000$        &$15,000$ & 10,992  & $8,192$ &  $8,124$ \\
    {\bf \#Attribute} & $780$           & $16$    &  16     & $12$    &   $112$  \\
$\gamma$ ($\eta=0.5$) &  $\approx 1.50$ & $0.155$ & $0.045$ & $0.031$ &  $0.850$ \\
$\gamma$ ($\eta=0.9$) &  $\approx 2.30$ & $0.290$ & $0.073$ & $0.057$ &  $1.140$ \\
\hline
\end{tabular}
\end{small}
\end{center}
\end{table*}

\begin{figure*}[ht]
\begin{center}
\centering
\includegraphics[width=0.98\textwidth]{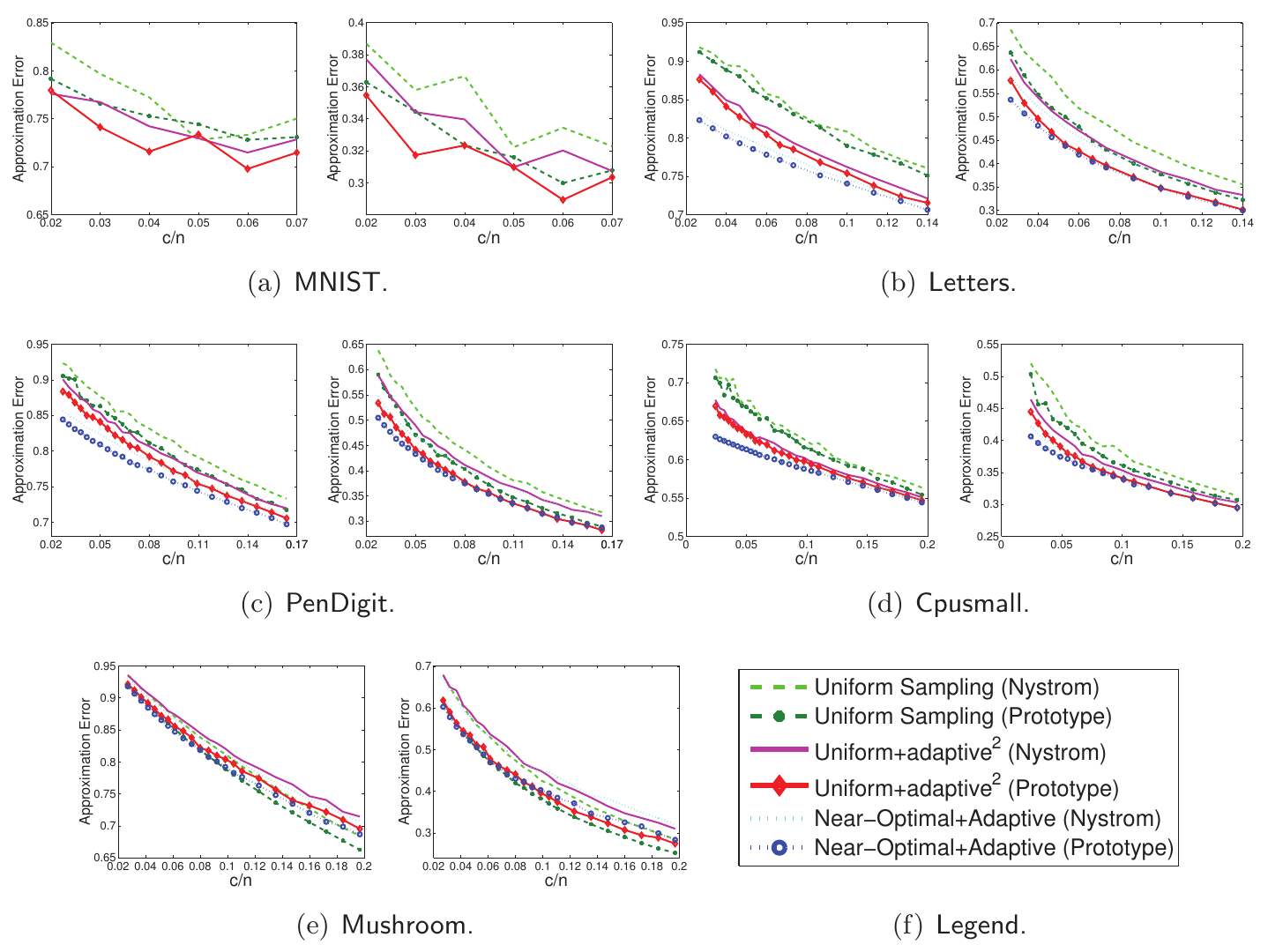}
\end{center}
   \caption{The ratio $\frac{c}{n}$ against the error ratio defined in \eqref{eq:def_approximation_error}.
   In each subfigure, the left corresponds to the RBF kernel matrix with $\eta=0.5$, and the right corresponds to $\eta=0.9$,
   where $\eta$ is defined in (\ref{eq:eigenvalue_ratio}).}
\label{fig:algorithm_memory}
\end{figure*}

\begin{figure*}[ht]
\begin{center}
\centering
\includegraphics[width=0.98\textwidth]{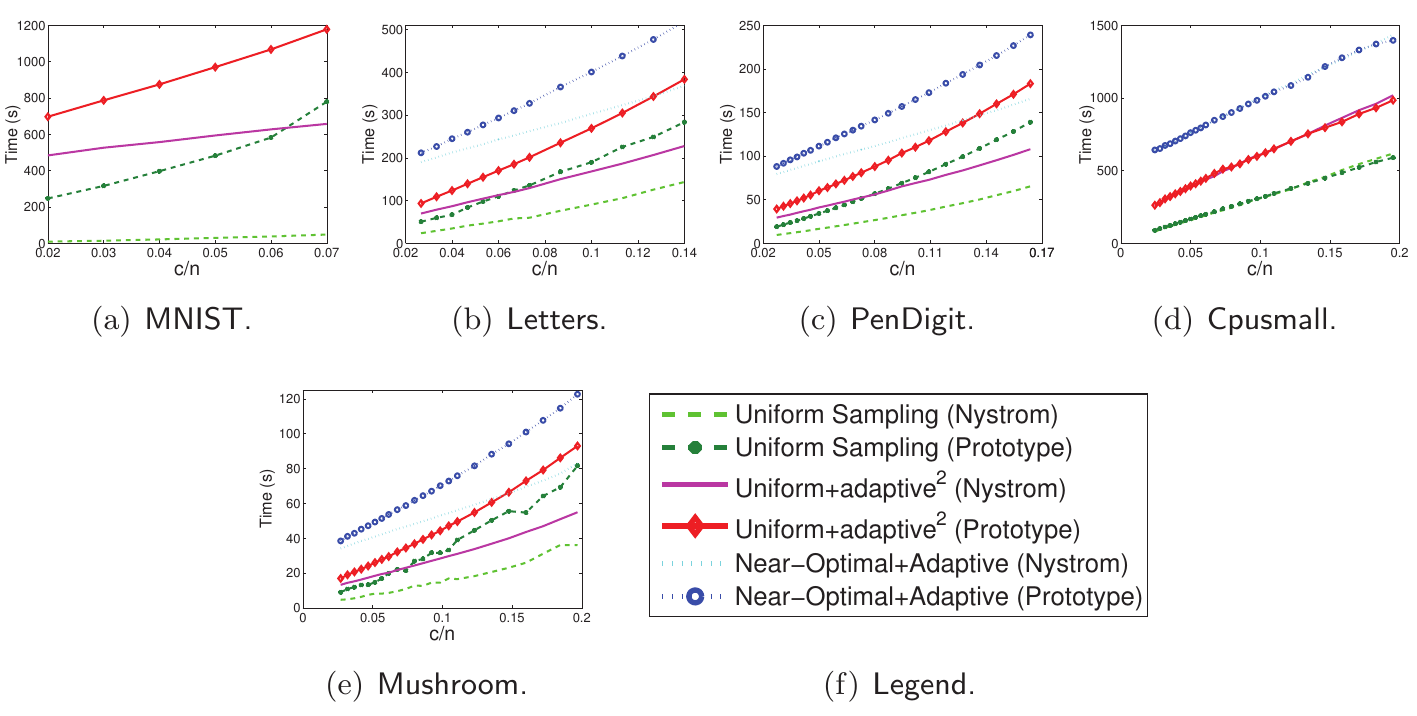}
\end{center}
   \caption{The growth of the average elapsed time in $\frac{c}{n}$.}
\label{fig:algorithm_time}
\end{figure*}


\section{Experiments on the Column Sampling Algorithms}  \label{sec:algorithm:experiments}

We empirically conduct comparison among three column selection algorithms---uniform sampling, uniform + adaptive$^2$, and the near-optimal + adaptive sampling algorithm.


\subsection{Experiment Setting} \label{sec:experiment_setting}

We perform experiments on several datasets collected on the LIBSVM website\footnote{http://www.csie.ntu.edu.tw/{$\sim$}cjlin/libsvmtools/datasets/}
where the data are scaled to [0,1].
We summarize the datasets in Table~\ref{tab:datasets}.

For each dataset, we generate a radial basis function (RBF) kernel matrix $\K$ defined by
$k_{ij} = \exp( -\frac{1}{2 \gamma^2} \|\x_i - \x_j\|_2^2 )$.
Here $\gamma>0$ is the scaling parameter;
the larger the scaling parameter $\gamma$ is, the faster the spectrum of the kernel matrix decays \citep{gittens2013revisiting}.
The previous work has shown that for the same dataset,
with different settings of $\gamma$, the sampling algorithms have very different performances.
Instead of setting $\gamma$ arbitrarily, we set $\gamma$ in the following way.

Letting $p = \lceil 0.05 n \rceil$,
we define
\begin{equation} \label{eq:eigenvalue_ratio}
\eta \; \triangleq \; \frac{ \sum_{i=1}^{p} \lambda_i^2 (\K) }{ \sum_{i=1}^{n} \lambda_i^2 (\K) }
\; = \; \frac{ \|\K_p \|_F^2 }{ \| \K \|_F^2 } ,
\end{equation}
which denotes the ratio of the top $5\%$ eigenvalues of the kernel matrix $\K$ to the all eigenvalues.
In general, a large $\gamma$ results in a large $\eta$.
For each dataset, we use two different settings of $\gamma$ such that $\eta=0.5$ or $\eta = 0.9$.

The models and algorithms are all implemented in MATLAB.
We run the algorithms on a workstation with Intel Xeon
2.40GHz CPUs, 24GB memory, and 64bit Windows Server 2008 system.
To compare the running time, we set MATLAB in single thread mode by the command ``$\mathrm{maxNumCompThreads(1)}$''.
In the experiments we do not keep $\K$ in memory.
We use a variant of Algorithm~\ref{alg:memory_cost}---we compute and store one block, instead of one column, of $\K$ at a time.
We keep at most $1,000$ columns of $\K$ in memory at a time.

We set the target rank $k$ to be $k= \lceil n/100\rceil$ in all the experiments unless otherwise specified.
We evaluate the performance by
\begin{eqnarray} \label{eq:def_approximation_error}
\textrm{Approximation Error} \; = \; {\| \K - \tilde{\K}\|_F} / {\|\K\|_F},
\end{eqnarray}
where $\tilde{\K}$ is the approximation generated by each method.

To evaluate the quality of the approximate rank-$k$ eigenvalue decomposition,
we use {\it misalignment} to indicate the distance between
the true eigenvectors $\U_k$ ($n\times k$) and the approximate eigenvectors $\tilde\V_k$ ($n\times k$):
\begin{eqnarray} \label{eq:def_misalignment}
\textrm{Misalignment}
\; = \;
\frac{1}{k}
\big\| \U_k - \tilde\V_k \tilde\V_k^T \U_k \big\|_F^2
\; \in \; [0, 1].
\end{eqnarray}

\begin{figure*}[ht]
\begin{center}
\centering
\includegraphics[width=0.98\textwidth]{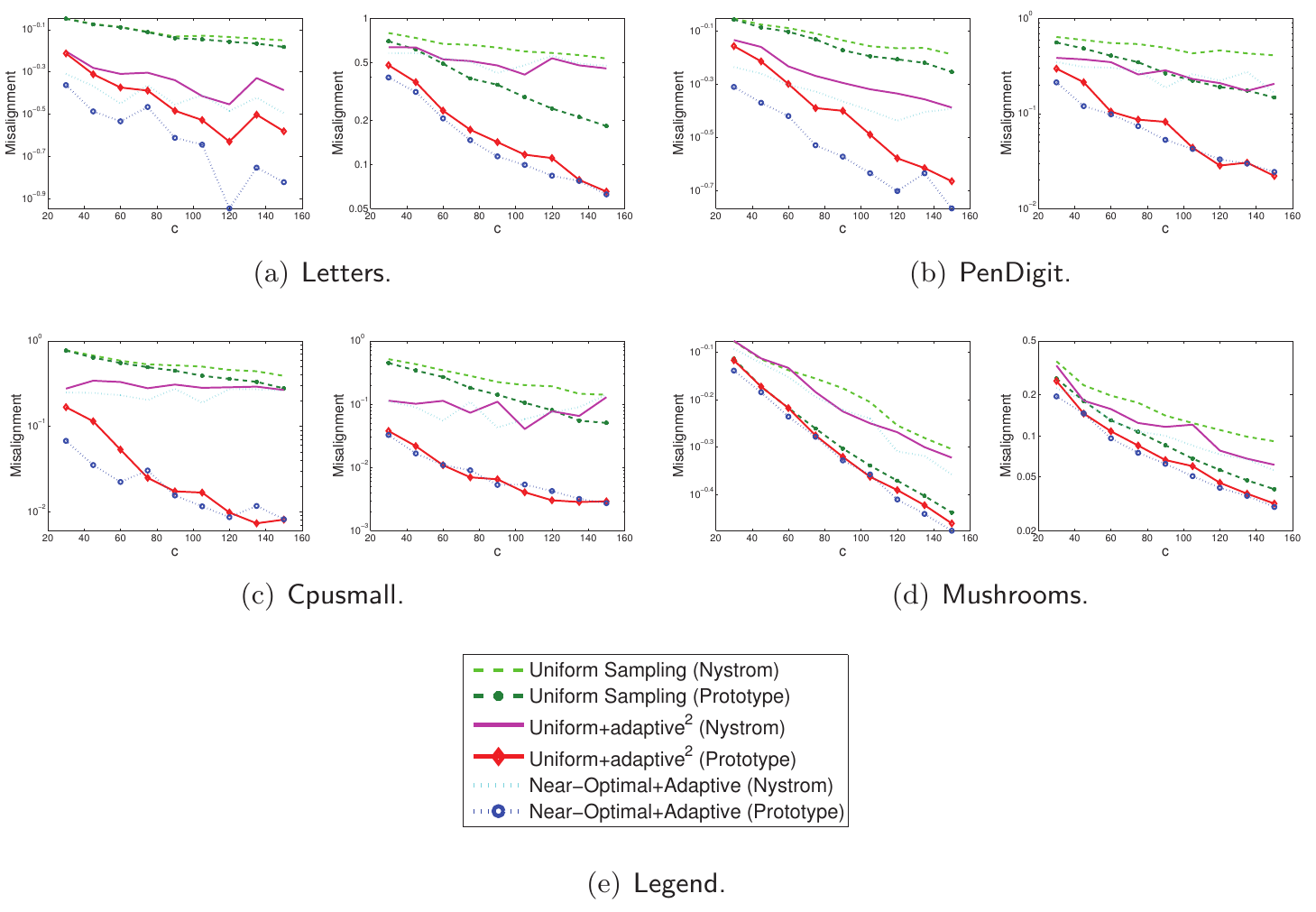}
\end{center}
   \caption{The number of selected columns $c$ against the misalignment (log-scale) defined in \eqref{eq:def_misalignment}.
   In each subfigure, the left corresponds to the RBF kernel matrix with $\eta=0.5$, and the right corresponds to $\eta=0.9$,
   where $\eta$ is defined in (\ref{eq:eigenvalue_ratio}).}
\label{fig:kpca_memory}
\end{figure*}

\begin{figure*}[ht]
\begin{center}
\centering
\includegraphics[width=0.98\textwidth]{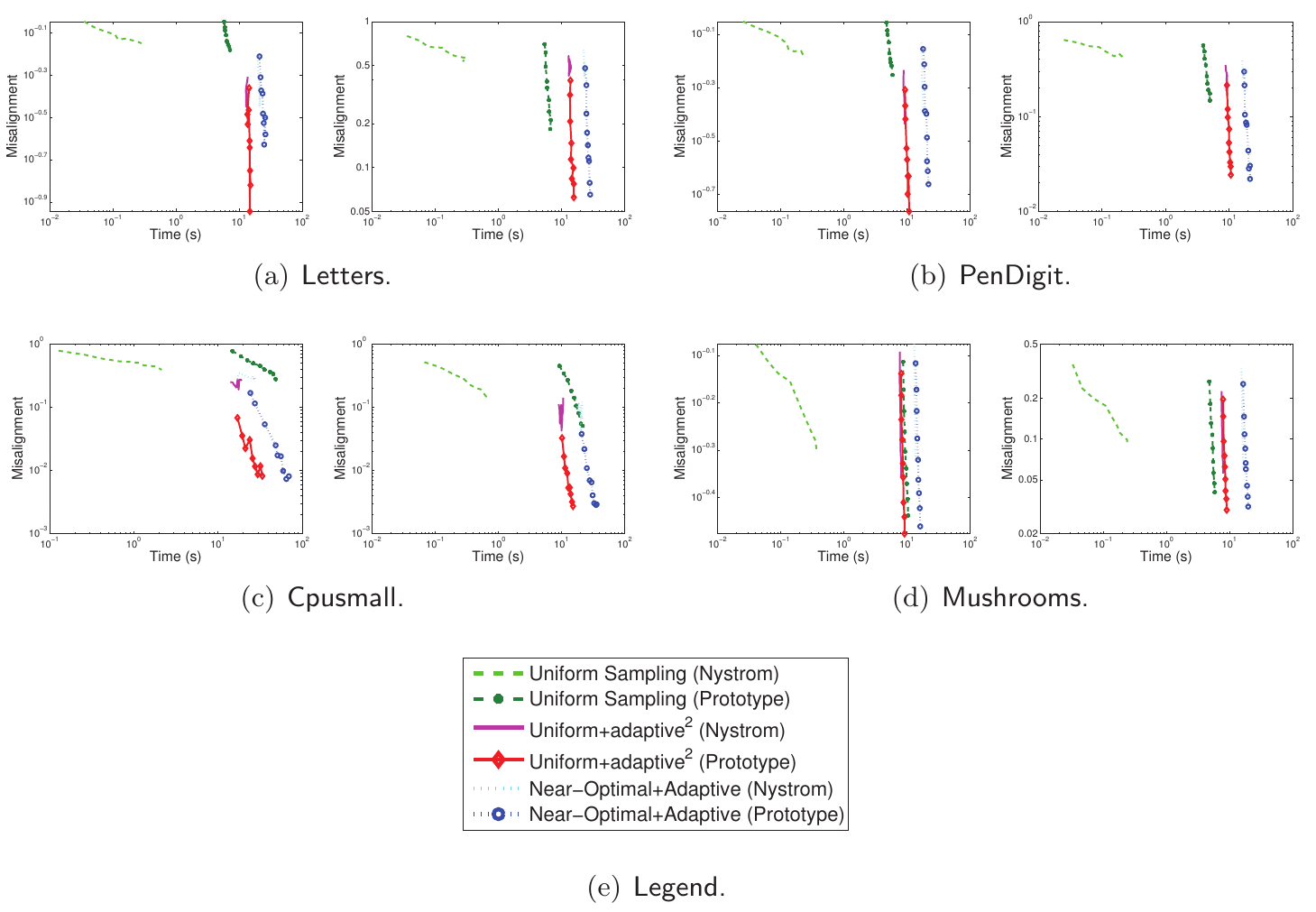}
\end{center}
   \caption{The elapsed time (log-scale) against the misalignment (log-scale) defined in \eqref{eq:def_misalignment}.
   In each subfigure, the left corresponds to the RBF kernel matrix with $\eta=0.5$, and the right corresponds to $\eta=0.9$,
   where $\eta$ is defined in (\ref{eq:eigenvalue_ratio}).}
\label{fig:kpca_time}
\end{figure*}


\subsection{Matrix Approximation Accuracy}

In the first set of experiments, we compare the matrix approximation quality using
the Frobenius norm approximation error defined in \eqref{eq:def_approximation_error} as the metric.

Every time when we do column sampling, we repeat each sampling algorithm $10$ times and record the minimal approximation error of the $10$ repeats.
We report the average elapsed time of the $10$ repeat rather than the total elapsed time because the $10$ repeats can be done in parallel on $10$ machines.
We plot $c$ against the approximation error in Figure~\ref{fig:algorithm_memory}.
For the kernel matrices with $\eta=0.9$, we plot $c$ against the average elapsed time in Figure~\ref{fig:algorithm_time};
for $\eta = 0.5$, the curve of the running time is very similar, so we do not show it.

The results show that our uniform+adaptive$^2$ algorithm achieves accuracy
comparable with the near-optimal+adaptive algorithm.
Especially, when $c$ is large, these two algorithms have virtually the same accuracy,
which agrees with our analysis:
a large $c$ implies a small error term $\epsilon$, and the error bounds of the two algorithms coincide when $\epsilon$ is small.
As for the running time, we can see that our uniform+adaptive$^2$ algorithm is much more efficient than the near-optimal+adaptive algorithm.

Particularly, the MNIST dataset has $60,000$ instances, and the $60,000\times 60,000$ kernel matrix $\K$ does not fit in memory.
The experiment shows that neither the prototype model nor the uniform+adaptive$^2$ algorithm require keeping $\K$ in memory.


\subsection{Kernel Principal Component Analysis} \label{sec:experiment_kpca}

In the second set of experiment,
we apply the kernel approximation methods to approximately compute the rank $k=3$ eigenvalue decomposition of the RBF kernel matrix.
We use the {\it misalignment} defined in \eqref{eq:def_misalignment} as the metric,
which reflects the distance between the true and the approximate eigenvectors.
We report the average misalignment of the 20 repeats.
We do not conduct experiments on the MNIST dataset because the true eigenvectors are too expensive to compute.

To evaluate the memory efficiency, we plot $c$ against the misalignment in Figure~\ref{fig:kpca_memory}.
The results show that the two non-uniform sampling algorithms are significantly better than uniform sampling.
The performance of our uniform+adaptive$^2$ algorithm is nearly the same to the near-optimal+adaptive algorithm.

To evaluate the time efficiency, we plot the elapsed time against the misalignment in Figure~\ref{fig:kpca_time}.
Though the uniform sampling algorithm is the most efficient in most cases, its accuracy is unsatisfactory.
In terms of time efficiency, the uniform+adaptive$^2$ algorithm is better than the near-optimal+adaptive algorithm.

The experiment on the kernel PCA shows that the prototype model with the uniform + adaptive$^2$ column sampling algorithm achieves the best performance.
Though the \nystrom method with uniform sampling is the most efficient,
its resulting misalignment is worse by an order of magnitude.
Therefore, when applied to speedup eigenvalue decomposition, the \nystrom method may not be a good choice,
especially when high accuracy is required.


\subsection{Separating the Time Costs}

Besides the total elapsed time,
the readers may be interested in the time cost of each step, especially when the data do no fit in memory.
We run the Nystr\"om method and the prototype model, each with uniform+adaptive$^2$ column sampling algorithm,
on the MNIST dataset with $\gamma = 2.3$ (see Table~\ref{tab:datasets}).
Notice that the $60,000\times 60,000$ kernel matrix does not fit in memory,
so we keep at most $1,000$ columns of the kernel matrix in memory at a time.
We set $c = 50$ or $500$ and repeat the procedure $20$ times and record the average elapsed time.
In Figure~\ref{fig:kernel_vs_other} we separately show the time costs of the uniform+adaptive$^2$ algorithm and the computation of the intersection matrices.
In addition, we separate the time costs of evaluating the kernel functions and all the other computations (e.g.\  SVD of $\C$ and matrix multiplications).

\begin{figure*}[ht]
\begin{center}
\centering
\subfigure[$c = 50$]{\includegraphics[width=0.48\textwidth]{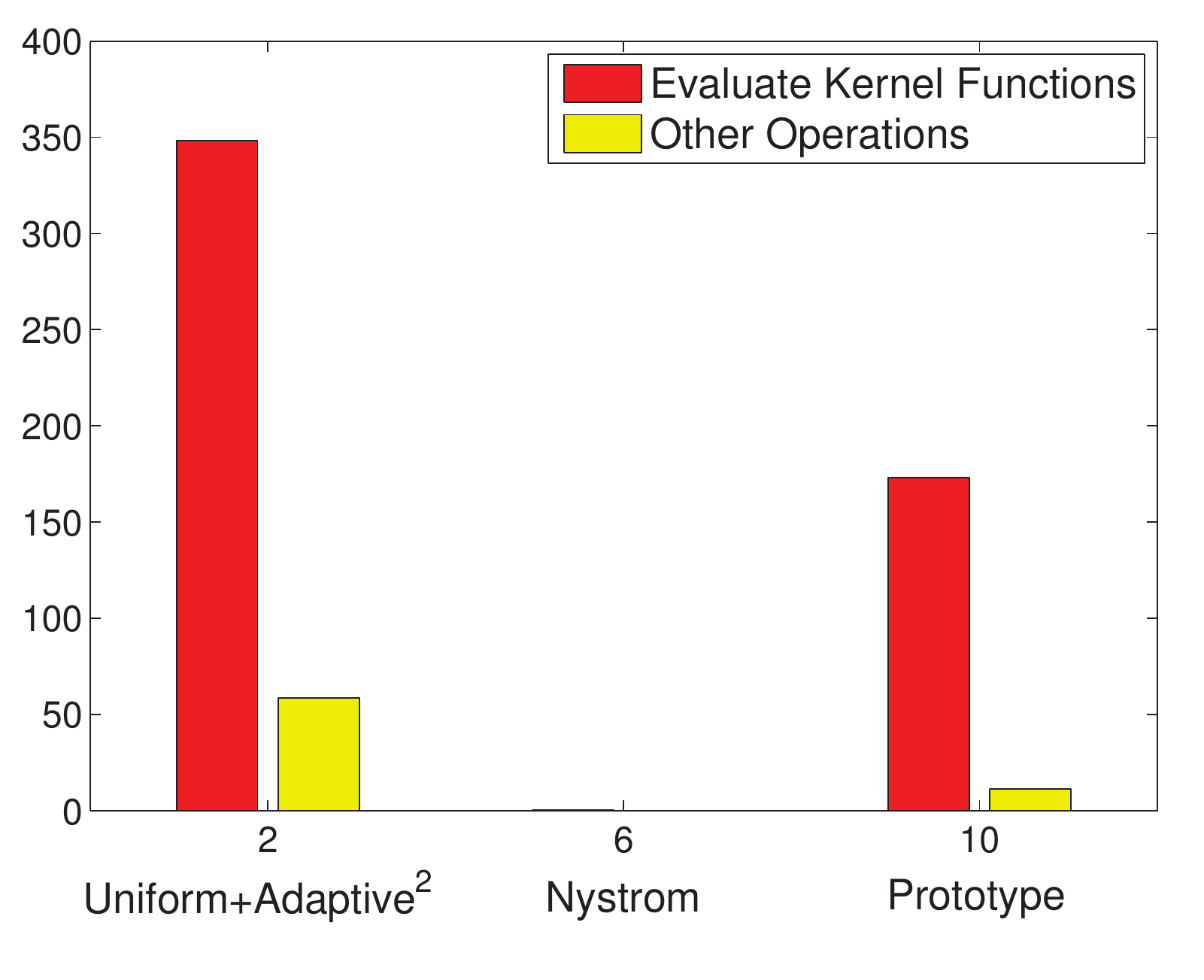}}
\subfigure[$c = 500$]{\includegraphics[width=0.48\textwidth]{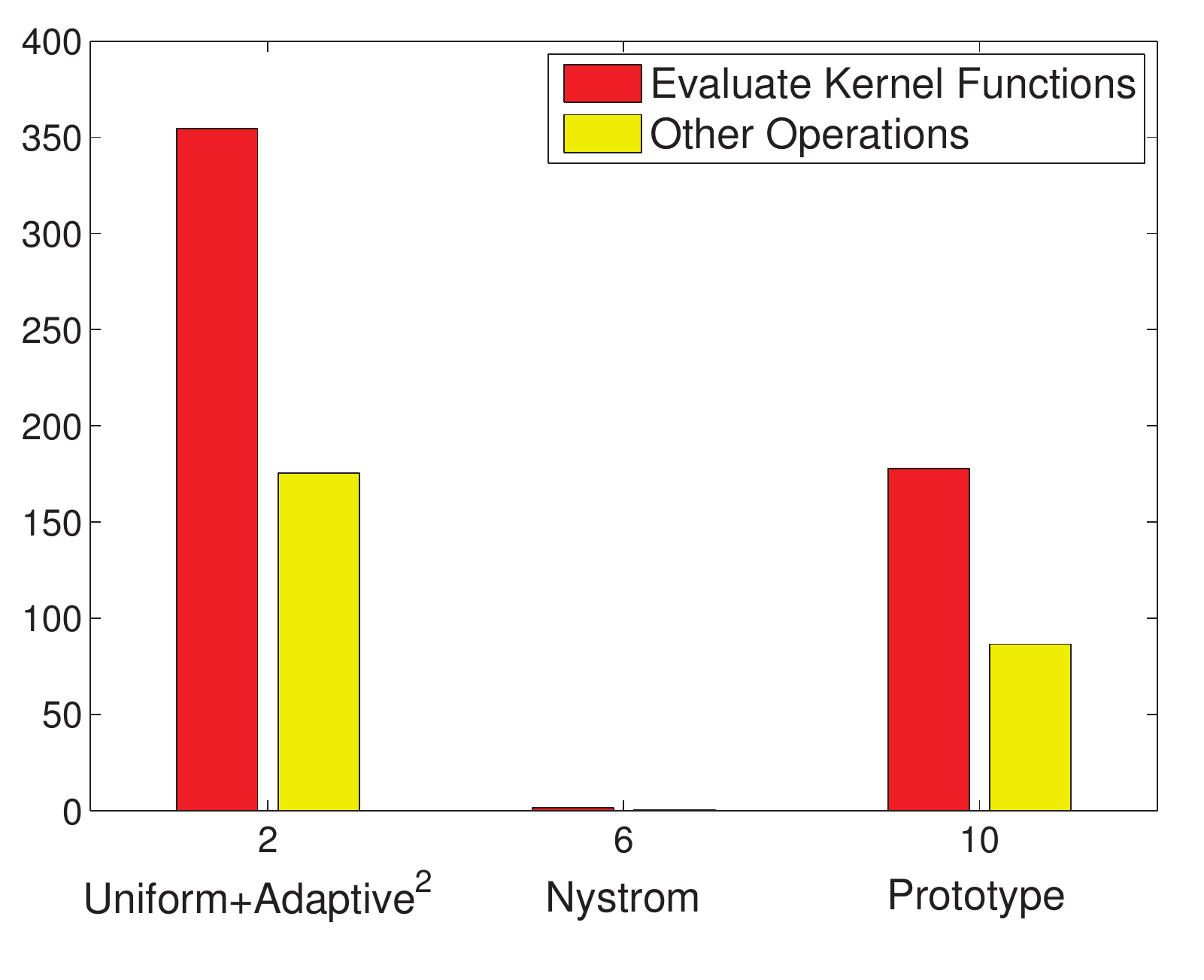}}
\end{center}
   \caption{The time costs (s) of the uniform+adaptive$^2$ algorithm and the computation of the intersection matrices.}
\label{fig:kernel_vs_other}
\end{figure*}

We can see from Figure~\ref{fig:kernel_vs_other} that when $\K$ does not fit in memory,
the computation of the kernel matrix contributes to the most of the computations.
By comparing the two subfigures in Figure~\ref{fig:kernel_vs_other},
we can see that as $c$ increases, the costs of computing the kernel matrix barely change,
but the costs of other matrix operations significantly increase.


\section{The Spectral Shifting Model} \label{sec:sspbs}

All the low-rank approximation methods work well only when the bottom eigenvalues of $\K$ are near zero.
In this section we develop extensions of the three SPSD matrix approximation models to tackle matrices with relatively big bottom eigenvalues.
We call the proposed method the {\it spectral shifting (SS) model} and describe it in Algorithm~\ref{alg:ss_nystrom}.
We show that the SS model has stronger error bound than the prototype model.

In Section~\ref{sec:sspbs:model} we formulate the SS model.
In Section~\ref{sec:sspbs:optimization} we study SS from an optimization perspective.
In Section~\ref{sec:sspbs:analysis} we show that SS has better error bound than the prototype model.
Especially, with the near-optimal+adaptive column sampling algorithm,
SS demonstrates much stronger error bound than the existing matrix approximation methods.
In Section~\ref{sec:sspbs:algorithm} we provide an efficient algorithm for computing the initial spectral shifting term.
In Section~\ref{sec:sspbs:combine} we discuss how to combine spectral shifting with other kernel approximation methods.

\begin{algorithm}[tb]
   \caption{The Spectral Shifting Method.}
   \label{alg:ss_nystrom}
\algsetup{indent=2em}
\begin{algorithmic}[1]
   \STATE {\bf Input:} an $n\times n$ SPSD matrix $\K$, a target rank $k$, the number of sampled columns $c$, the oversampling parameter $l$.
   \STATE // {\it (optional) approximately do the initial spectral shifting} \label{alg:ss_nystrom:delta_begin}
   \STATE $\Ome \longleftarrow n\times l$ standard Gaussian matrix;
   \STATE $\Q \longleftarrow$ the $l$ orthonormal basis of $\K \Ome \in \RB^{n\times l}$; \label{alg:ss_nystrom:Y}
   \STATE $s \longleftarrow$ sum of the top $k$ singular values of $\Q^T \K \in \RB^{l\times n}$;
   \STATE $\tilde{\delta} = \frac{1}{n-k} \big(\tr(\K) - s\big) \approx \bar{\delta}$; \label{alg:ss_nystrom:delta_end}
   \STATE $\bar{\K} \leftarrow \K - \tilde{\delta} \I_n \in \RB^{n\times n}$;
   \STATE // {\it perform sketching, e.g.\ random projection or column selection}
   \STATE $\bar{\C} = \bar\K \PP$, where $\PP$ is an $n\times c$ random projection or selection matrix;
   \STATE Optional: replace $\bar\C$ by its orthonormal bases;
   \STATE // {\it compute the spectral shifting parameter and the intersection matrix}
   \STATE $\delta^{\textrm{ss}} \longleftarrow \frac{1}{n-\rk(\bar{\C})}\Big( \tr(\K) - \tr\big(\bar{\C}^\dag\K \bar{\C} \big) \Big)$;
   \STATE $\U^{\textrm{ss}} \longleftarrow \bar{\C}^\dag\K(\bar{\C}^\dag)^T - \delta^{\textrm{ss}}(\bar{\C}^T \bar{\C})^{\dag}$;
   \RETURN the approximation $\tilde{\K}_c^{\textrm{ss}} = \bar{\C} \U^{\textrm{ss}} \bar{\C}^T + \delta^{\textrm{ss}} \I_n$.
\end{algorithmic}
\end{algorithm}


\subsection{Model Formulation} \label{sec:sspbs:model}

The spectral shifting (SS) model is defined by
\begin{eqnarray} \label{eq:def_sspbs}
\Kss \; = \; \bar{\C} \U^{\textrm{ss}} \bar{\C}^T + \delta^{\textrm{ss}} \I_n .
\end{eqnarray}
Here $\delta^{\textrm{ss}}\geq 0$ is called the spectral shifting term.
This approximation is computed in three steps.
Firstly, (approximately) compute the initial spectral shifting term
\begin{eqnarray} \label{eq:initial_ss_term}
\bar{\delta}
\; = \; \frac{1}{n-k} \bigg( \tr(\K) - \sum_{j=1}^k \sigma_j (\K) \bigg) ,
\end{eqnarray}
and then perform spectral shifting $\bar{\K} = \K - \bar{\delta} \I_n$, where $k \leq c$ is the target rank.
This step is optional.
Due to Theorem~\ref{thm:ss_closed_form} and Remark~\ref{remark:ss_closed_form},
SS is better than the prototype model even if $\bar{\delta} = 0$;
however, without this step, the quantity of the improvement contributed by SS is unknown.
Secondly, draw a column selection matrix $\PP$ and form the sketch $\bar\C = \bar\K \PP$.
Finally, with $\bar{\C}$ at hand, compute $\delta^{\textrm{ss}}$ and $\U^{\textrm{ss}}$ by
\begin{eqnarray}  \label{eq:ss_closed_form}
\delta^{\textrm{ss}} & = & \frac{1}{n- \rk(\bar{\C})}\Big( \tr(\K) - \tr\big({\bar{\C}}^\dag\K {\bar{\C}} \big) \Big), \nonumber \\
\U^{\textrm{ss}} & = & {\bar{\C}}^\dag \K ( {\bar{\C}}^\dag)^T - \delta^{\textrm{ss}}({\bar{\C}}^T {\bar{\C}})^{\dag}.
\end{eqnarray}
We will show that $\Kss$ is positive (semi)definite if $\K$ is positive (semi)definite.

However, when the bottom eigenvalues of $\K$ are small, the computed spectral shifting term is small,
where there is little difference between SS and the prototype model,
and the spectral shifting operation is not advised.


\subsection{Optimization Perspective} \label{sec:sspbs:optimization}

The SS model is an extension of the prototype model from the optimization perspective.
Given an SPSD matrix $\K$, the prototype model computes the sketch $\C = \K \PP$ and the intersection matrix
\begin{equation}  \label{eq:def_pbs_intersection}
\U^\star \; = \; \C^\dag \K (\C^\dag)^T
\; = \; \argmin_{\U} \| \K - \C \U \C^T \|_F^2.
\end{equation}
Analogously, with the sketch $\bar\C = \bar{\K} \PP = \K \PP - \bar{\delta} \PP$ at hand,
SS is obtained by solving
\begin{equation}  \label{eq:def_ss_nystrom}
\big(\U^{\textrm{ss}} , \delta^{\textrm{ss}} \big)
\; = \;
\argmin_{\U, \delta} \big\| \K - \bar{\C} \U \bar{\C}^T - \delta \I_n \big\|_F^2,
\end{equation}
obtaining the intersection matrix $\U^{\textrm{ss}}$ and the spectral shifting term $\delta^{\textrm{ss}}$.
By analyzing the optimization problem \eqref{eq:def_ss_nystrom},
we obtain the following theorem.
Its proof is in Appendix~\ref{sec:proof_ss_closed_form}.

\begin{theorem} \label{thm:ss_closed_form}
The pair ($\delta^{\textrm{ss}}, \U^{\textrm{ss}}$) defined in (\ref{eq:ss_closed_form})
is the global minimizer of problem (\ref{eq:def_ss_nystrom}),
which indicates that using any other ($\delta, \U$) to replace $(\delta^{\textrm{ss}}, \U^{\textrm{ss}})$
results in a larger approximation error.
Furthermore, if $\K$ is positive (semi)definite,
then the approximation $\bar{\C} \U^{\textrm{ss}} \bar{\C}^T + \delta^{\textrm{ss}} \I_n$ is also positive (semi)definite.
\end{theorem}

\begin{remark} \label{remark:ss_closed_form}
The optimization perspective indicates the superiority of SS.
Suppose we skip the initial spectral shifting step and simply set $\bar\delta = 0$.
Then $\bar\C = \C$.
If the constraint $\delta = 0$ is to the optimization problem \eqref{eq:def_ss_nystrom},
then \eqref{eq:def_ss_nystrom} will become identical to the prototype model \eqref{eq:def_pbs_intersection}.
Obviously, adding this constraint will make the optimal objective function value get worse,
so the optimal objective function value of \eqref{eq:def_ss_nystrom} is always less than or equal to \eqref{eq:def_pbs_intersection}.
Hence, without the initial spectral shifting step, SS is still more accurate than the prototype model.
\end{remark}


\subsection{Error Analysis} \label{sec:sspbs:analysis}

The following theorem indicates that the SS model with any spectral shifting term $\delta \in (0, \bar{\delta}]$
has a stronger bound than the prototype model.
The proof is in Appendix~\ref{sec:app:ss}.

\begin{theorem} \label{thm:superiority}
Suppose there is a sketching matrix $\PP \in \RB^{n\times c}$ such that
for any $n \times n$ {\it symmetric} matrix $\A$ and target rank $k$ ($\ll n$),
by forming $\C = \A \PP$,
the prototype model satisfies the error bound
\[
\big\| \A - \C \C^\dag \A (\C^\dag)^T \C^T \big\|_F^2
\; \leq \; \eta \, \big\|\A - \A_k \big\|_F^2
\]
for certain $\eta > 0$.
Let $\K$ be any $n \times n$ SPSD matrix, $\tilde\delta \in (0, \bar{\delta}]$ be the initial spectral shifting term
where $\bar{\delta}$ is defined in (\ref{eq:initial_ss_term}),
$\bar{\K} = \K - \tilde\delta \I_n$, $\bar\C = \bar\K \PP$,
and $\Kss$ be the SS model defined in \eqref{eq:def_sspbs}.
Then
\[
\big\| \K - \Kss \big\|_F^2
\; \leq \;
\eta \, \big\| \bar{\K} - \bar{\K}_k \|_F^2
\; \leq \; \eta \, \big\| \K - \K_k \|_F^2 .
\]
\end{theorem}

We give an example in Figure~\ref{fig:ss} to illustrate the intuition of spectral shifting.
We use the toy matrix $\K$: an $n\times n$ SPSD matrix whose the $t$-th eigenvalue is $1.05^{-t}$.
We set $n=100$ and $k=30$, and hence $\bar{\delta}=0.064$.
From the plot of the eigenvalues we can see that
the ``tail'' of the eigenvalues becomes thinner after the spectral shifting.
Specifically, $\|\K- \K_k\|_F^2 = 0.52$ and $\|\bar{\K}- \bar{\K}_k\|_F^2 \leq 0.24$.
From Theorem~\ref{thm:superiority} we can see that if $\| \K - \C \C^\dag \K (\C^\dag)^T \C^T \|_F \leq 0.52\eta$,
then $\| \K - \Kss \|_F \leq 0.24 \eta$.
This indicates that SS has much stronger error bound than the prototype model.

\begin{figure}[!ht]
\begin{center}
\centering
\subfigure[{Before spectral shifting.}]{\includegraphics[width=0.4\textwidth]{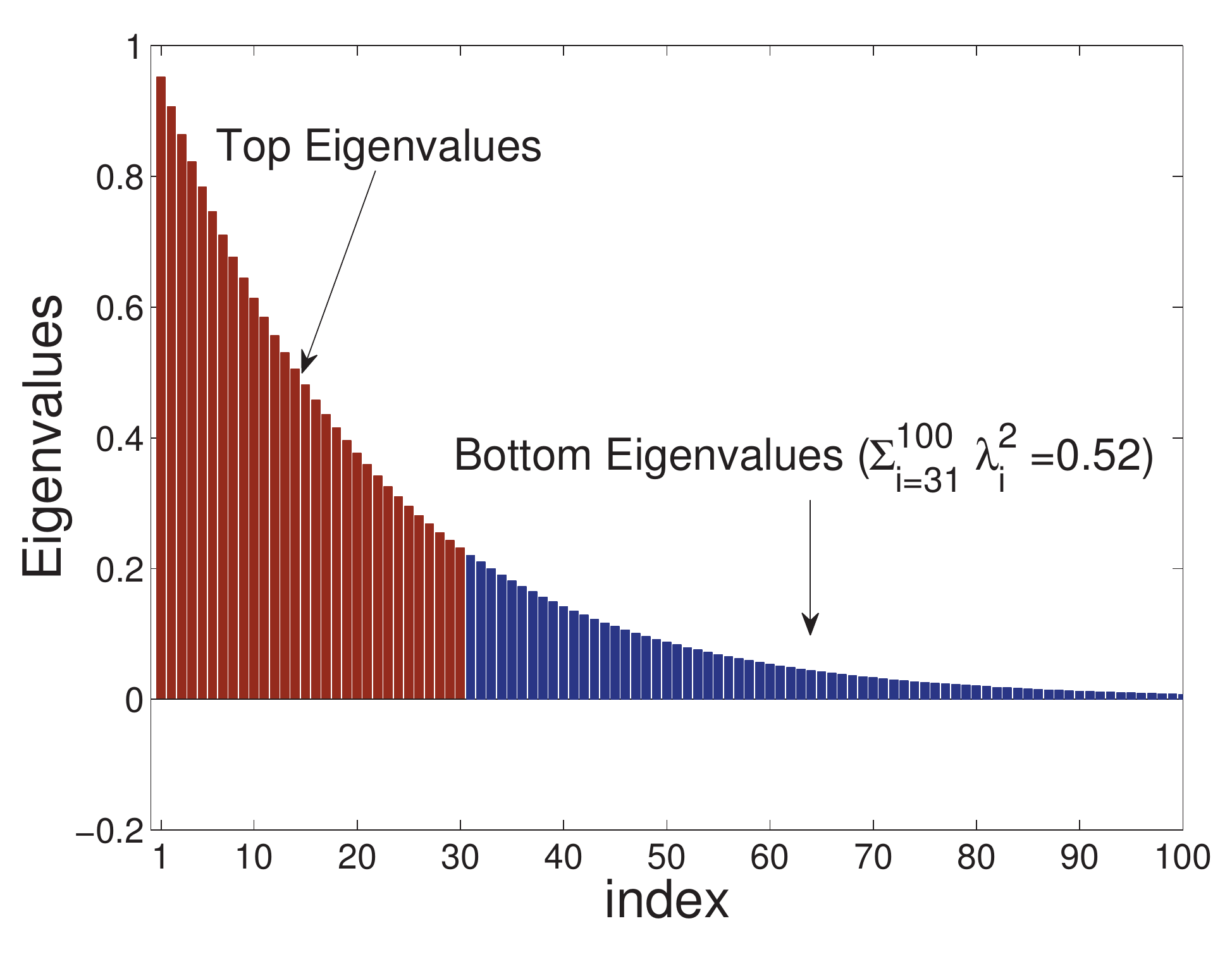}\label{fig:ss:a}}
\subfigure[{After spectral shifting.}]{\includegraphics[width=0.4\textwidth]{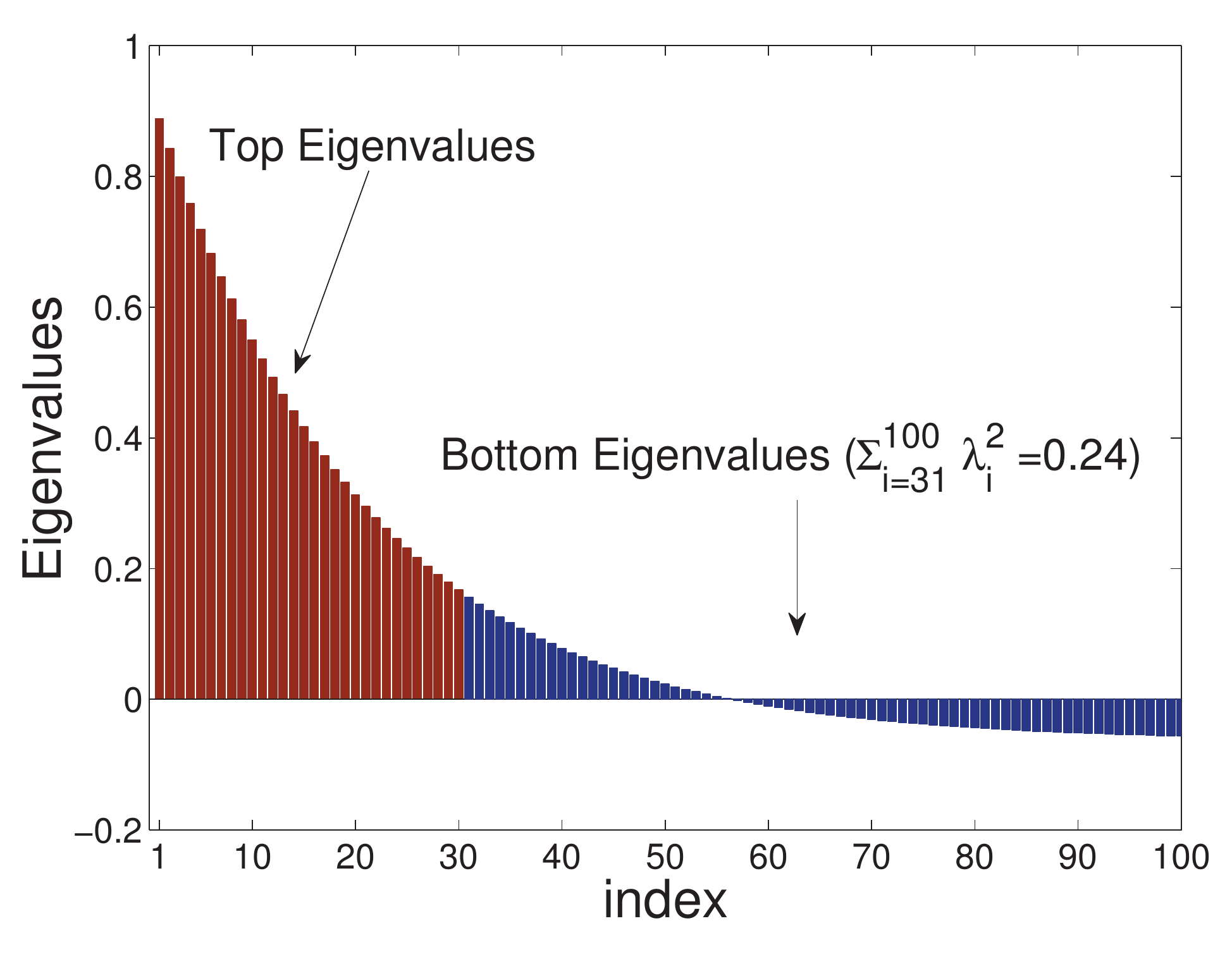}\label{fig:ss:b}}
\end{center}
   \caption{We plot the eigenvalues of $\K$ in Figure \ref{fig:ss:a} and $\bar{\K}=\K-\bar{\delta} \I_{100}$ in Figure \ref{fig:ss:b}.}
\label{fig:ss}
\end{figure}

The following theorem shows an error bound of the SS model,
which is stronger than the prototype model, especially when the bottom $n-k$ eigenvalues of $\K$ are big.
The proof is in Appendix~\ref{sec:app:ss}.

\begin{theorem} \label{thm:ss_nystrom_bound}
Suppose there is a sketching matrix $\PP \in \RB^{n\times c}$ such that
for any $n \times n$ {\it symmetric} matrix $\A$ and target rank $k$ ($\ll n$),
by forming the sketch $\C = \A \PP$,
the prototype model satisfies the error bound
\[
\big\| \A - \C \C^\dag \A (\C^\dag)^T \C^T \big\|_F^2
\; \leq \; \eta \, \big\|\A - \A_k \big\|_F^2
\]
for certain $\eta > 0$.
Let $\K$ be any $n \times n$ SPSD matrix,
$\bar{\delta}$ defined in (\ref{eq:initial_ss_term}) be the initial spectral shifting term,
and $\Kss$ be the SS model defined in \eqref{eq:def_sspbs}.
Then
\[
\big\|\K - \Kss \big\|_F^2
\: \leq \:
\eta \bigg( \big\|\K - \K_k \big\|_F^2 - \frac{\big[ \sum_{i=k+1}^n \lambda_{i} (\K) \big]^2}{n-k}  \bigg).
\]
\end{theorem}

If $\bar\C$ contains the columns of $\bar{\K}$ sampled by the near-optimal+adaptive algorithm in Theorem~\ref{thm:near_opt},
which has the strongest bound,
then the error bound incurred by SS is given in the following corollary.

\begin{corollary}\label{cor:ss_nystrom_bound}
Suppose we are given any SPSD matrix $\K$ and
we sample $c=\OM(k /\epsilon)$ columns of $\bar{\K}$ to form $\bar{\C}$
using the near-optimal+adaptive column sampling algorithm (Theorem~\ref{thm:near_opt}).
Then the inequality holds:
\[
\EB \big\|\K - \tilde{\K}_{c}^{\textrm{ss}} \big\|_F^2
 \leq
(1+\epsilon) \bigg( \|\K - \K_k \|_F^2 - \frac{\big[ \sum_{i=k+1}^n \lambda_{i} (\K) \big]^2}{n-k}  \bigg).
\]
\end{corollary}

Here we give an example to demonstrate the superiority of SS over the prototype model, the Nystr\"om method, and even the truncated SVD of the same scale.

\begin{example} \label{exa:example1}
Let $\K$ be an $n\times n$ SPSD matrix such that
$\lambda_1 (\K) \geq \cdots \geq \lambda_k (\K) > \theta = \lambda_{k+1} (\K) = \cdots = \lambda_n (\K)>0$.
By sampling $c = \OM(k )$ columns by the near-optimal+adaptive algorithm (Theorem~\ref{thm:near_opt}),
we have that
\begin{eqnarray}
\big\|\K - \tilde{\K}_{c}^{\textrm{ss}} \big\|_F^2
\; = \; 0  \nonumber
\end{eqnarray}
and that
\begin{eqnarray}
(n-c) \theta^2
\; = \;
\big\|\K - \K_c \big\|_F^2
\; \leq \;
\big\|\K - \tilde{\K}_{c}^{\textrm{proto}} \big\|_F^2
\; \leq \;
\big\|\K - \tilde{\K}_{c}^{\textrm{nys}} \big\|_F^2 . \nonumber
\end{eqnarray}
Here $\tilde{\K}_{c}^{\textrm{proto}}$ and $\tilde{\K}_{c}^{\textrm{nys}}$ respectively denote the approximation formed by the prototype model and the Nystr\"om method.
In this example the SS model is far better than the other models if we set $\theta$ as a large constant.
\end{example}


\subsection{Approximately Computing $\bar{\delta}$} \label{sec:sspbs:algorithm}

The SS model uses $\bar{\delta}$ as the initial spectral shifting term.
However, computing $\bar{\delta}$ according to (\ref{eq:initial_ss_term})
requires the partial eigenvalue decomposition which costs $\OM(n^2 k)$ time and $\OM(n^2)$ memory.
For large-scale data, one can simply set $\bar{\delta} = 0$;
Remark~\ref{remark:ss_closed_form} shows that SS with this setting still works better than the prototype model.
For medium-scale data, one can approximately compute $\bar{\delta}$ by the algorithm devised and analyzed in this subsection.

We depict the algorithm in Lines \ref{alg:ss_nystrom:delta_begin}--\ref{alg:ss_nystrom:delta_end} of Algorithm~\ref{alg:ss_nystrom}.
The performance of the approximation is analyzed in the following theorem.

\begin{theorem} \label{thm:delta}
Let $\bar{\delta}$ be defined in (\ref{eq:initial_ss_term})
and $\tilde{\delta}$, $k$, $l$, $n$ be defined in Algorithm~\ref{alg:ss_nystrom}.
The following inequality holds:
\[
\EB \big[ { \big|\bar{\delta} - \tilde{\delta}\big|} \, \big/ \, {\bar{\delta}} \big] \;\leq \; {k} / {\sqrt{l}},
\]
where the expectation is taken w.r.t.\ the Gaussian random matrix $\Ome$ in Algorithm~\ref{alg:ss_nystrom}.
Lines~\ref{alg:ss_nystrom:delta_begin}--\ref{alg:ss_nystrom:delta_end} in Algorithm~\ref{alg:ss_nystrom} compute $\tilde{\delta}$
in $\OM (n^2 l)$ time and $\OM(n l)$ memory.
\end{theorem}

Here we empirically evaluate the accuracy of the approximation to $\bar{\delta}$
(Lines \ref{alg:ss_nystrom:delta_begin}--\ref{alg:ss_nystrom:delta_end} in Algorithm~\ref{alg:ss_nystrom}) proposed in Theorem~\ref{thm:delta}.
We use the RBF kernel matrices with the scaling parameter $\gamma$ listed Table~\ref{tab:datasets}.
We use the error ratio $|\bar{\delta} - \tilde{\delta}| / \bar{\delta}$ to evaluate the approximation quality.
We repeat the experiments 20 times and plot $l/k$ against the average error ratio in Figure~\ref{fig:delta}.
Here $\tilde{\delta}$, $l$, and $k$ are defined in Theorem~\ref{thm:delta}.
We can see that the approximation of $\bar{\delta}$ has high quality:
when $l = 4k$, the error ratios are less than $0.03$ in all cases,
no matter whether the spectrum of $\K$ decays fast or slow.

\begin{figure*}[ht]
\begin{center}
\centering
\subfigure[{Letters}]{\includegraphics[width=0.48\textwidth]{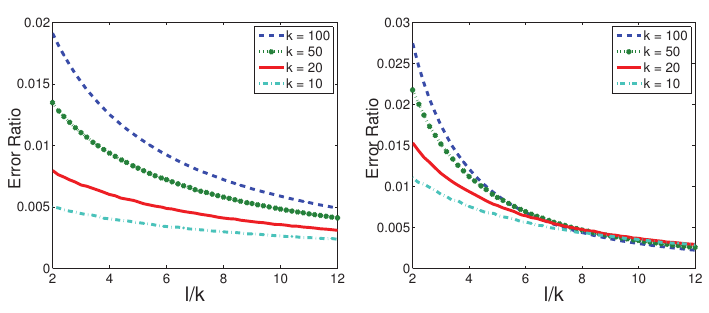}}~~
\subfigure[{PenDigit}]{\includegraphics[width=0.48\textwidth]{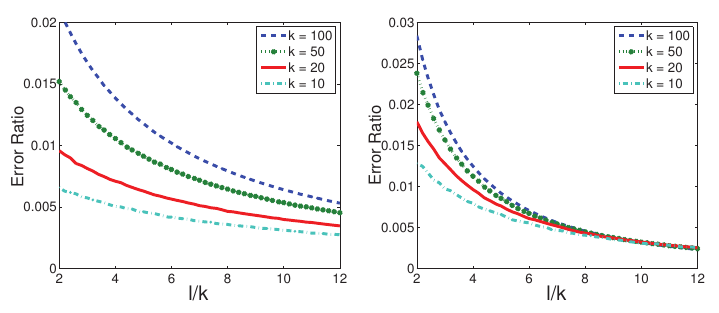}}
\end{center}
   \caption{The ratio $\frac{l}{k}$ against the error $|\bar{\delta} - \tilde{\delta}| / \bar{\delta}$.
   In each subfigure, the left corresponds to the RBF kernel matrix with $\eta=0.5$, and the right corresponds to $\eta=0.9$,
   where $\eta$ is defined in (\ref{eq:eigenvalue_ratio}).}
\label{fig:delta}
\end{figure*}


\subsection{Combining with Other Matrix Approximation Methods} \label{sec:sspbs:combine}

There are many other matrix approximation approaches such as
the ensemble \nystrom method \citep{kumar2012sampling} and  MEKA \citep{si2014memory}.
In fact, the key components of the ensemble \nystrom method and MEKA are the \nystrom method,
which can be straightforwardly replaced by other matrix approximation methods such as the SS model.

{\bf The ensemble \nystrom method} improves the \nystrom method by running the \nystrom method
$t$ times and combine the samples to construct the kernel approximation:
\begin{equation*}
\Kens \; = \;
\sum_{i=1}^t \mu^{(i)} \C^{(i)} {\W^{(i)}}^\dag {\C^{(i)}}^T,
\end{equation*}
where $\mu^{(1)} , \cdots , \mu^{(t)}$ are the weights of the samples,
and a simple but effective strategy is to set the weights as $\mu^{(1)} = \cdots = \mu^{(t)} = \frac{1}{t}$.
However, the time and memory costs of computing $\C$ and $\U$ are respectively $t$ times as much as
that of its base method (e.g.\ the \nystrom method),
and the ensemble \nystrom method needs to use the Sherman-Morrison-Woodbury formula
$t$ times to combine the samples.
When executed on a single machine, the accuracy gained by the ensemble may not worth the $t$ times more time and memory costs.

{\bf MEKA} is reported to be the state-of-the-art kernel approximation method.
It exploits the block-diagonal structure of kernel matrices,
and outputs an ${n\times c}$ sparse matrix $\C$ and a $c\times c$ small matrix $\U$ such that $\K \approx \C \U \C^T$.
MEKA first finds the blocks by clustering the data into $b$ clusters
and permutes the kernel matrix accordingly.
It then approximates the diagonal blocks by the \nystrom method,
{\it which can be replaced by other methods.}
It finally approximates the off-diagonal blocks using the diagonal blocks.
If the kernel matrix is partitioned into $b^2$ blocks,
only $b$ blocks among the $b^2$ blocks of $\C$ are nonzero,
and the number of nonzero entries of $\C$ is at most $\nnz(\C) = n c / b$.
MEKA is thus much more memory efficient than the \nystrom method.
If we use the SS model to approximate the diagonal blocks,
then the resulting MEKA approximation will be in the form $\K \approx \C \U \C^T + \delta_1 \I \oplus \cdots \oplus \delta_b \I$,
where $\delta_i$ corresponds to the $i$-th diagonal block.

\begin{figure*}[ht]
\begin{center}
\centering
\includegraphics[width=0.98\textwidth]{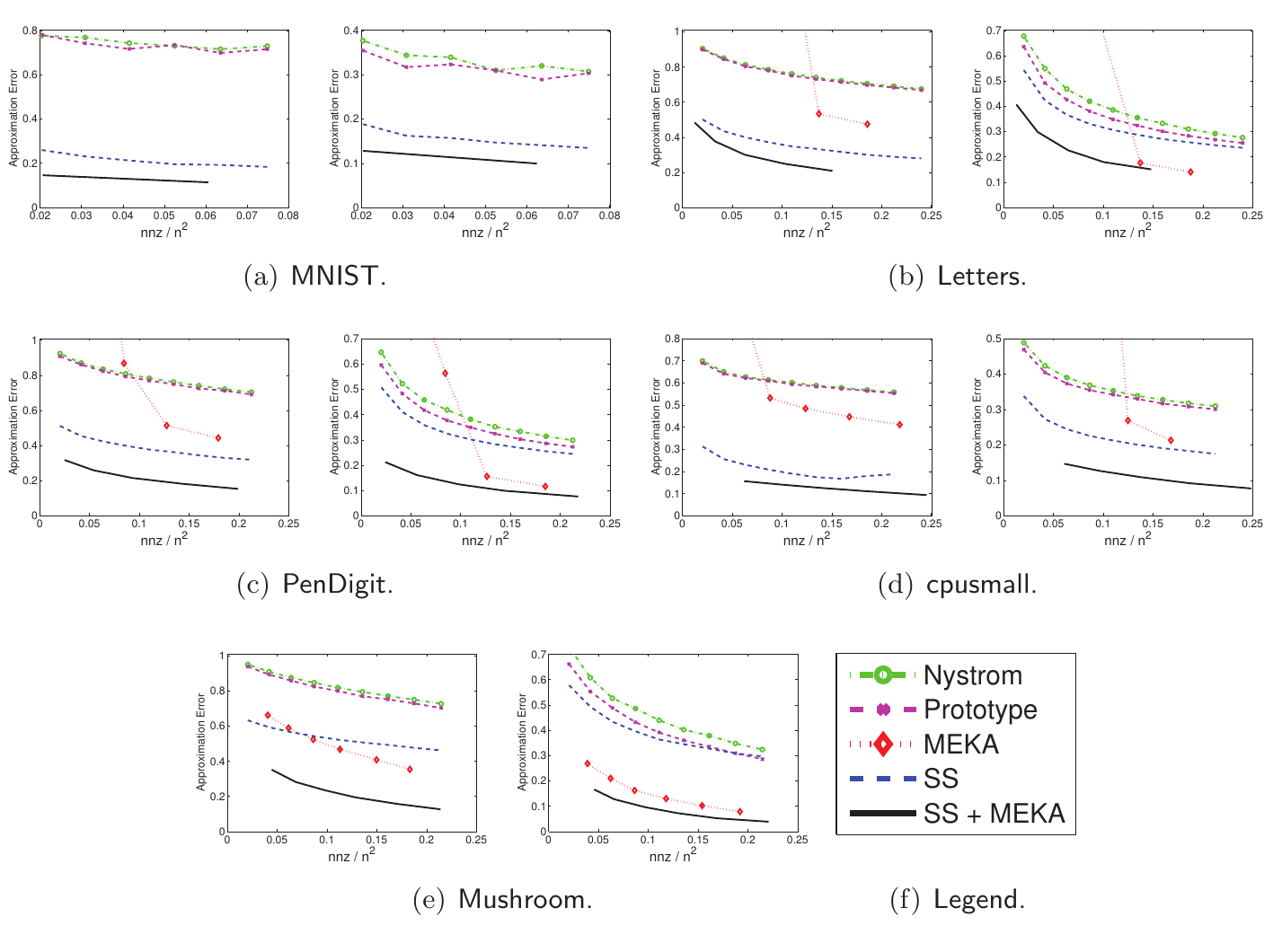}
\end{center}
   \caption{The memory cost against approximation error.
    Here ``nnz'' is number of nonzero entries in the sketch, namely, $\nnz (\C) + \nnz (\U)$.
   In each subfigure, the left corresponds to the RBF kernel matrix with $\eta=0.5$, and the right corresponds to $\eta=0.9$,
   where $\eta$ is defined in (\ref{eq:eigenvalue_ratio}).}
\label{fig:ss_memory}
\end{figure*}

\begin{figure*}[ht]
\begin{center}
\centering
\includegraphics[width=0.98\textwidth]{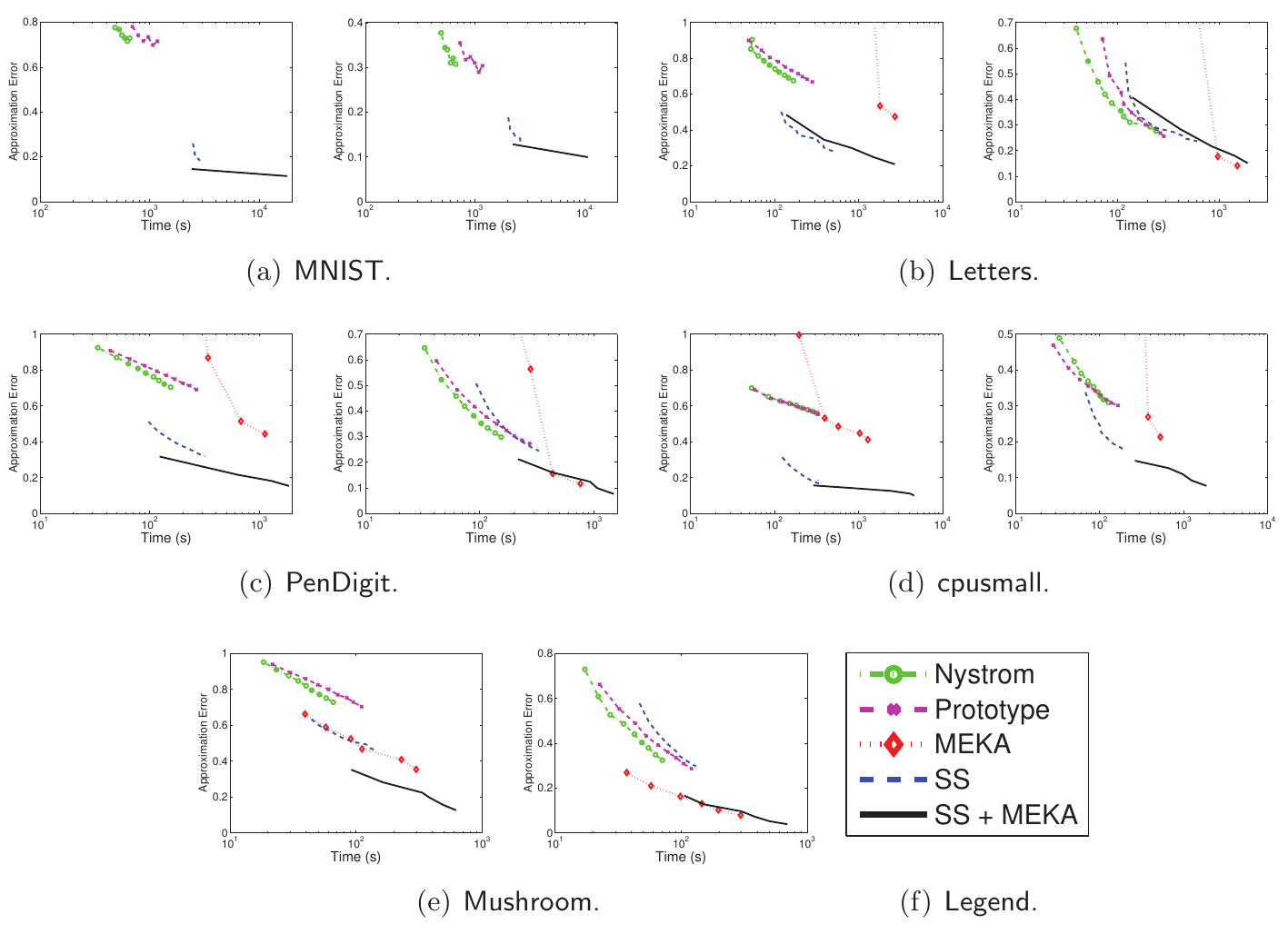}
\end{center}
   \caption{The elapsed time (log-scale) against the approximation error.
   In each subfigure, the left corresponds to the RBF kernel matrix with $\eta=0.5$, and the right corresponds to $\eta=0.9$,
   where $\eta$ is defined in (\ref{eq:eigenvalue_ratio}).}
\label{fig:ss_time}
\end{figure*}


\section{Empirically Evaluating the Spectral Shifting Model} \label{sec:sspbs:experiments}

We empirically evaluate the spectral shifting method by comparing  the following kernel approximation models
in terms of approximation quality and the generalization performance on Gaussian process regression.
\begin{itemize}
\item   The \nystrom method with the uniform+adaptive$^2$ algorithm.
\vspace{-2mm}
\item   The prototype model with the uniform+adaptive$^2$ algorithm.
\vspace{-2mm}
\item   The memory efficient kernel approximation (MEKA) method \citep{si2014memory},
        which approximates the diagonal blocks of the kernel matrix by the \nystrom method.
        We use the code released by the authors with default settings.
\vspace{-2mm}
\item   The spectral shifting (SS) model with the uniform+adaptive$^2$ algorithm.
\vspace{-2mm}
\item   SS+MEKA: the same to MEKA except for using SS, rather than the \nystrom method, to approximate the diagonal blocks.
\end{itemize}
Since the experiments are all done on a single machine,
the time and memory costs of the ensemble \nystrom method \citep{kumar2012sampling} are $t$ times larger.
It would be unfair to directly do  comparison with the ensemble \nystrom method.

\begin{figure*}[ht]
\begin{center}
\centering
\subfigure[$\eta \approx 0.5$]{\includegraphics[width=0.24\textwidth]{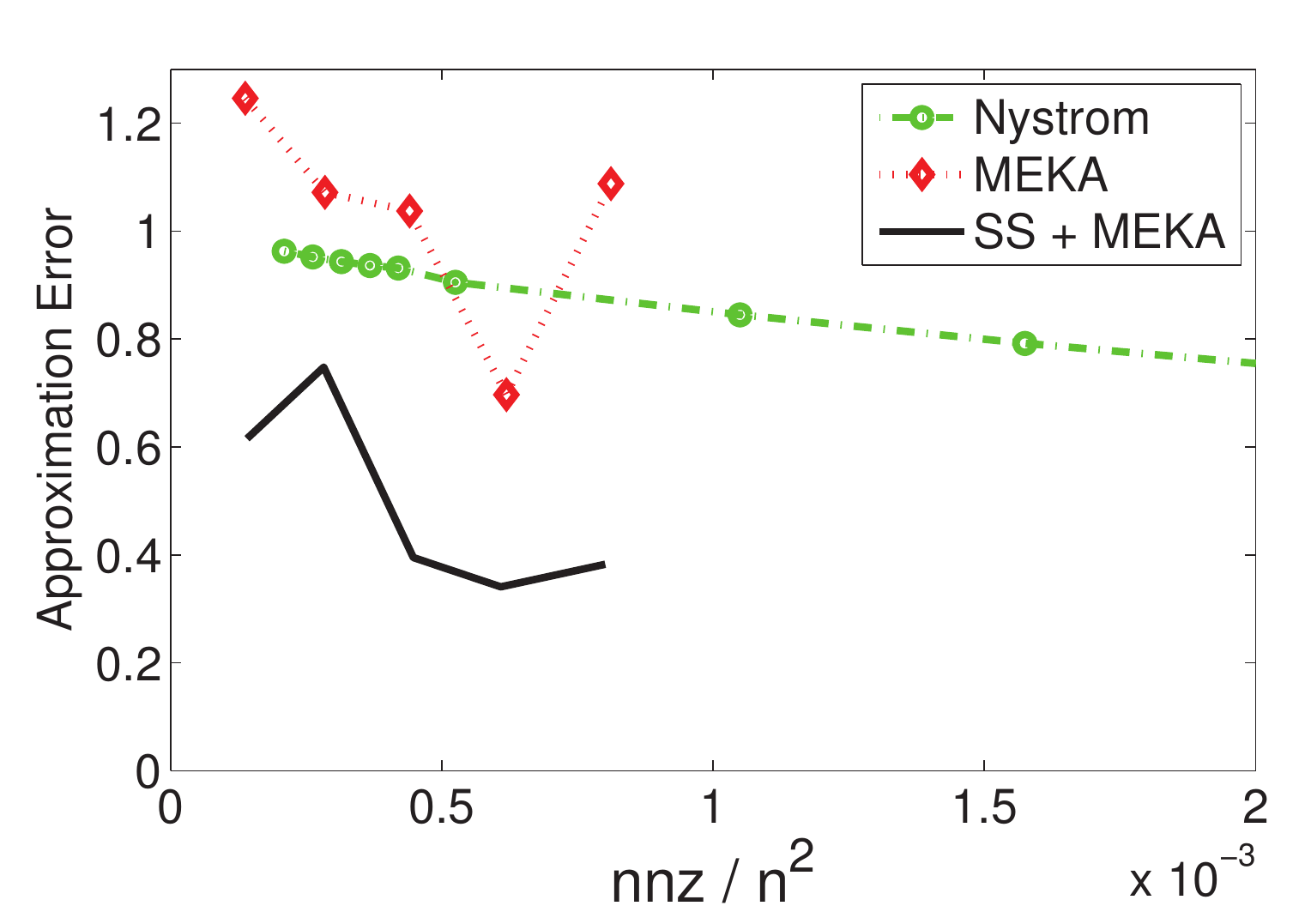}
                            \includegraphics[width=0.24\textwidth]{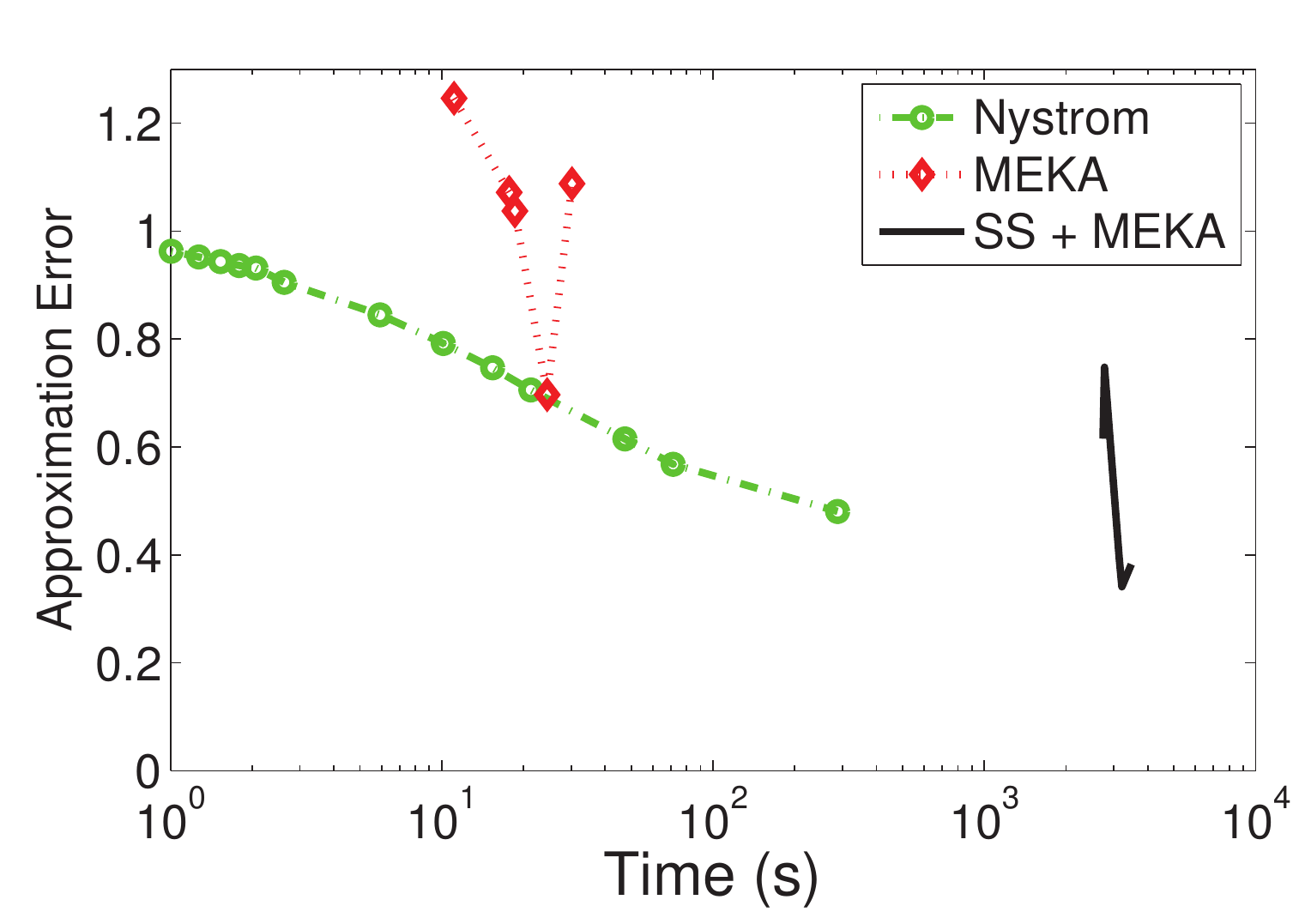}}
\subfigure[$\eta \approx 0.9$]{\includegraphics[width=0.24\textwidth]{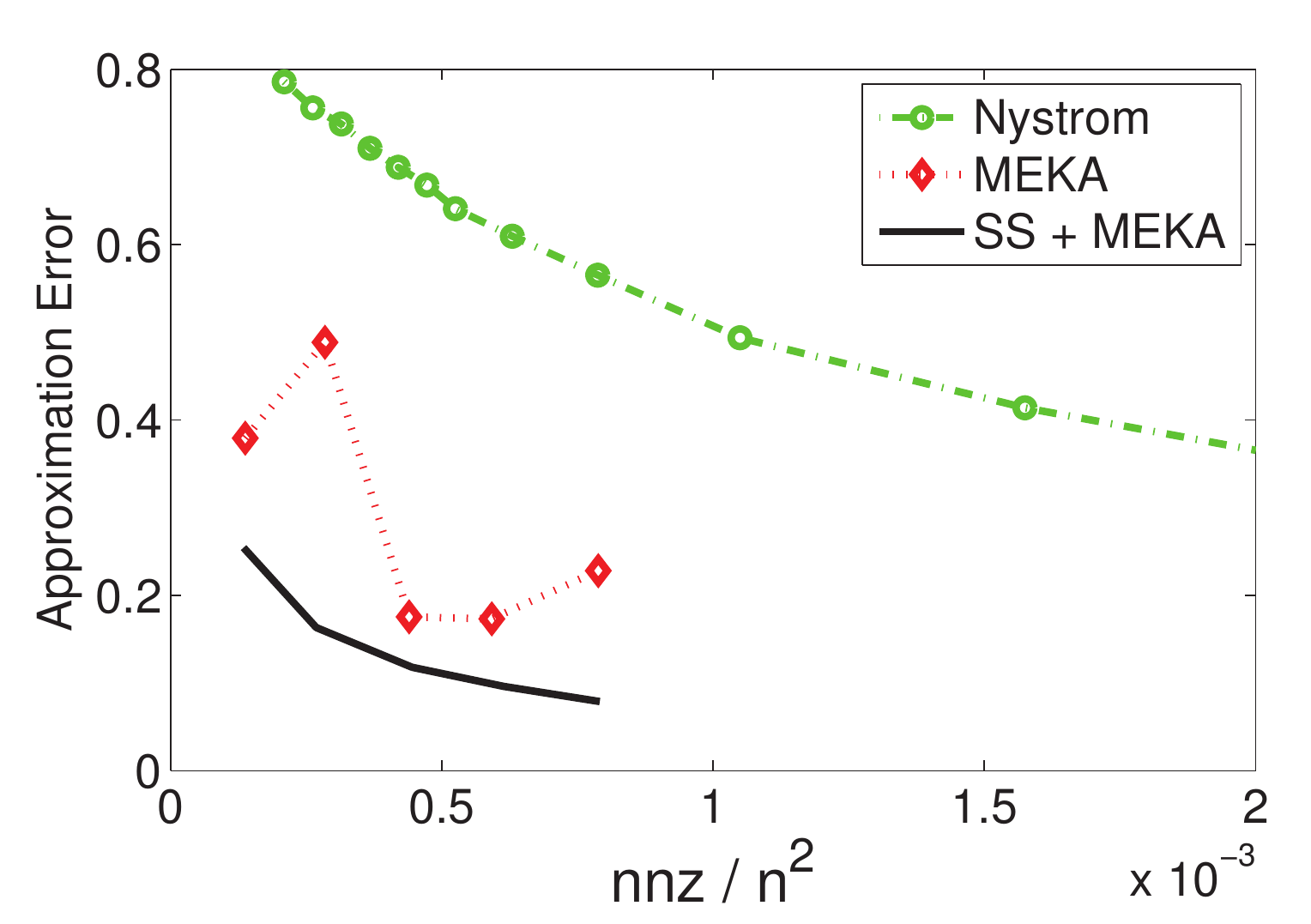}
                            \includegraphics[width=0.24\textwidth]{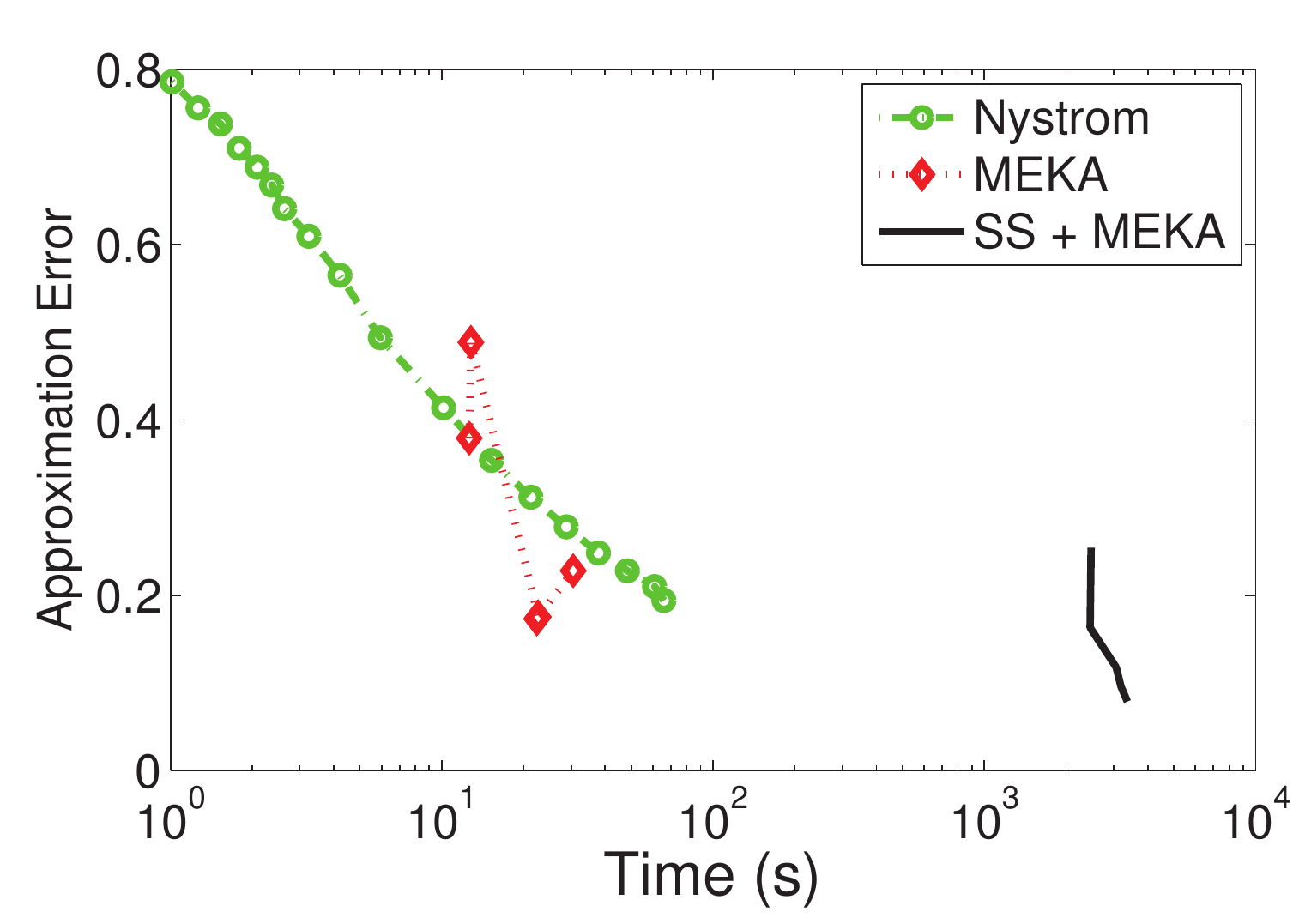}}
\end{center}
   \caption{The experiments on the covtype dataset. We plot the memory cost against approximation error and plot the elapsed time (log-scale) against the approximation error.}
\label{fig:ss_covtype}
\end{figure*}


\subsection{Experiments on Approximation Accuracy}

We conduct experiments using the same setting as in Section~\ref{sec:experiment_setting}.
To demonstrate the effect of spectral shifting,
we only consider a kernel matrix with slowly decaying spectrum.
Thus we set $\eta$ (defined in \eqref{eq:eigenvalue_ratio}) relatively small: $\eta = 0.5$ and $\eta = 0.9$.
When $\eta$ is near one, the spectral shifting term becomes near zero,
and spectral shifting makes no difference.

We present the results in two ways:
(1) Figure~\ref{fig:ss_memory} plots the memory usage against the approximation error,
where the memory usage is proportional to the number of nonzero entries of $\C$ and $\U$;
(2) Figure \ref{fig:ss_time} plots the elapsed time against the approximation error.
The results have the following implications.
\begin{itemize}
\item
Using the spectral shifting technique is better than without using it:
SS and SS+MEKA are more accurate than the prototype model and MEKA, respectively.
The advantage of spectral shifting is particularly obvious when the spectrum of $\K$ decays slowly, i.e.\ when $\eta = 0.5$.
\item
SS and SS+MEKA are the best among the compared methods according to our experiments.
Especially, if a high-quality approximation in limited memory is desired, SS+MEKA should be the best choice.
\item
The kernel matrix of the MNIST dataset is $60,000\times 60,000$, which does not fit in the $24$GB memory.
This shows that the compared methods and sampling algorithm do not require keeping $\K$ in memory.
\end{itemize}

In addition, we conduct experiments on the large-scale dataset---covtype---which has 581,012 data instances.
We compare among the Nystr\"om method, MEKA, and MEKA+SS.
The experiment setting is slightly different from the ones in the above.
For each of the four methods, we use uniform sampling to select columns.
As for the MEKA based methods, we set the number of clusters to be 30 (whereas the default is 10) to increase scalability.
As for the spectral shifting method, we do not perform the initial spectral shifting.
The results are plotted in Figure~\ref{fig:ss_covtype}.
The results show the effectiveness and scalability of the spectral shift method.

\begin{table*}[t]\setlength{\tabcolsep}{0.3pt}
    \caption{Summary of datasets for Gaussian process regression}\label{tab:expGPR}
\begin{footnotesize}
    \begin{center}
        \begin{tabular}{ c ccccccc}
            \hline
                        &~~Plant~~&~~White Wine~~&~~Red Wine~~&~~Concrete~~&~~Energy (Heat)~~&~~Energy (Cool)~~&~~Housing~~\\
            \hline
   {\bf \#Instance}     & 9,568 &   4,898   &   1,599   &    1,030  &     768       &     768       &  506  \\
   {\bf \#Attribute}    &    4  &   11      &    11     &    8      &      8        &      8        & 13  \\
            $\gamma$    & 0.1   &    1      &     1     &    1      &      0.5      &      0.5      &  1    \\
            $\sigma^2$  & 0.1   &   0.01    &   0.01    &   0.0002  &     0.0005    &    0.001      & 0.005 \\
            \hline
        \end{tabular}
       \vskip -0.4cm
    \end{center}
\end{footnotesize}
\end{table*}


\subsection{Experiments on Gaussian Process Regression}

We apply the kernel approximation methods to Gaussian process regression (GPR).
We assume that the training set is $\{(\x_1, y_1), . . . ,(\x_n, y_n)\}$,
where $\x_i \in \RB^d$ are input vectors and $y_i\in\RB$ are the corresponding outputs.
GPR is defined as
\begin{eqnarray*}
y = u + f(\x) + \epsilon, \quad \epsilon \sim  \mathcal{N}(0, \sigma^2),
\end{eqnarray*}
where $f(\x)$ follows a Gaussian process with mean function 0 and kernel function $\kappa (\cdot,\cdot)$.
Furthermore, we define the kernel matrix $\K \in \RB^{n\times n}$, where the $(i,j)$-th entry of $\K$ is $\kappa(\x_i, \x_j)$.

For a test input $\x_*$, the prediction is given by
\begin{eqnarray*}
    \hat{y}_* = {\bf k}^T(\x_*)(\K + \sigma^2\I)^{-1}\y,
\end{eqnarray*}
where $\y = [y_1,\dots,y_n]^T$ and $\k (\x_*) = [\kappa(\x_*, \x_1), \dots, \kappa(\x_*, \x_n)]^T$.
We apply different kernel approximation methods to approximate $\K$
and approximately compute $(\K + \sigma^2\I)^{-1}\y$ according to Section~\ref{sec:motivation}.
We evaluate the generalization performance using the mean squared error:
\begin{eqnarray*}
   \mbox{MSE} = \frac{1}{m}\sum_{i=1}^m (y_{i*}-{\hat y}_{i*})^2,
\end{eqnarray*}
where $y_{i*}$ is the real output of the $i$-th test sample and $m$ is the number of test samples.

\begin{figure*}[ht]
\begin{center}
\centering
\includegraphics[width=0.98\textwidth]{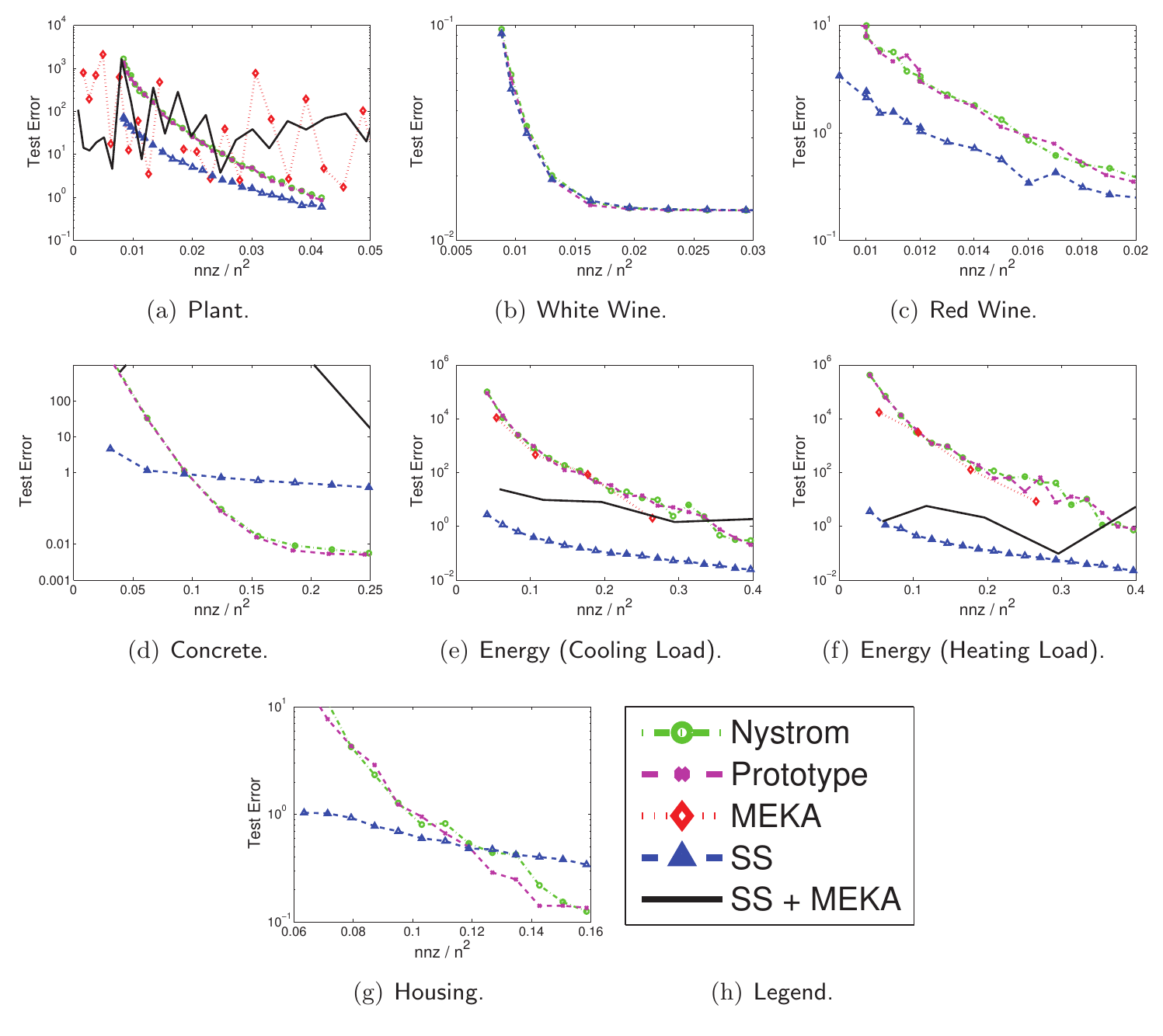}
\end{center}
   \caption{The results on Gaussian process regression.
    In the figures ``nnz'' is number of nonzero entries in the sketch, that is, $\nnz (\C) + \nnz (\U)$.}
\label{fig:gp_nnz}
\end{figure*}

\begin{figure*}[ht]
\begin{center}
\centering
\includegraphics[width=0.98\textwidth]{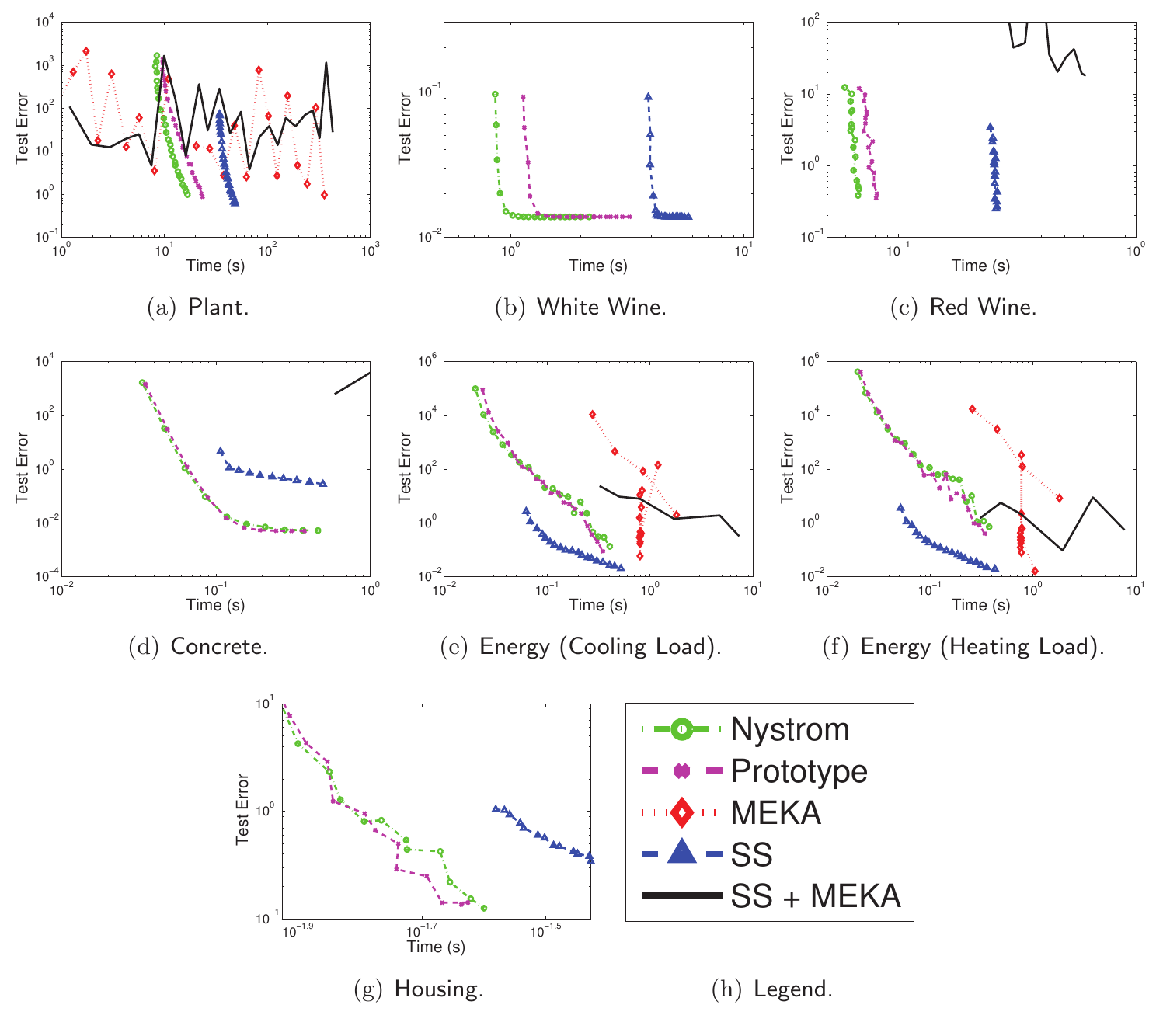}
\end{center}
   \caption{The results on Gaussian process regression.}
\label{fig:gp_time}
\end{figure*}

We conduct experiment on seven datasets summarized in Table~\ref{tab:expGPR}.
We use the Gaussian RBF kernel and tune two parameters: the variance $\sigma^2$ and the kernel scaling parameter $\gamma$.
Recall that there are seven compared methods and eight datasets, so
the time cost would be daunting if we use cross-validation to find $\sigma$ and $\gamma$ for each method on each dataset.
Thus we perform a five-fold cross-validation without using kernel approximation to pre-determine the two parameters $\sigma$ and $\gamma$,
and the same parameters are used for all the kernel approximation methods.
We list the obtained parameters  in Table~\ref{tab:expGPR}.

For each of the compared methods, we randomly hold $80\%$ samples for training and the rest for test;
we repeat this procedure $50$ times and record the average MSE,
the average elapsed time,
and the average of the number of nonzero entries in the sketch.
We plot $\frac{\nnz(\C) + \nnz (\U)}{n^2}$ against MSE in Figure~\ref{fig:gp_nnz}
and  the elapsed time against MSE in Figure~\ref{fig:gp_time}.

Using the same amount of memory,
our SS model achieves the best performance on the Plant, Red Wine, Energy (Cool), and Energy (Heat) datasets.
On the Concrete and Housing datasets, using spectral shifting leads better results unless $c$ is unreasonably large (e.g.\  $c > 0.1 n$).
However, using the same amount of time, the compared methods have competitive performance.

MEKA and SS+MEKA in general work very well,
but they are numerically unstable on some of the randomly partitioned training data.
On the White Wine, Housing, and Concrete datasets,
MATLAB occasionally reports errors of numerical instability in approximating the off-diagonal blocks of $\K$ by solving some linear systems;
when such happens, we do not report the corresponding test errors in the figures.
MEKA and SS+MEKA are also sometimes unstable on the Plant and Red Wine datasets, though MATLAB does not report error.
Although MEKA and SS+MEKA have good performance in general,
the numerical instability makes MEKA and SS+MEKA perform very poorly on a few among the 50 randomly partitioned training/test data,
and consequently the average test errors become large.
This problem can get avoided by repeating MEKA multiple times and choose the one that is stable.
However, this will significantly increase the time cost.


\section{Conclusions}

We have provided an in-depth study of the prototype model for SPSD matrix approximation.
First, we have shown that with $c=\OM (k/\epsilon)$ columns sampled by the near-optimal+adaptive algorithm,
the prototype model attains $1+\epsilon$ Frobenius norm relative-error bound.
This upper bound matches the lower bound up to a constant factor.
Second, we have devised a simple column selection algorithm called {\it uniform+adaptive$^2$}.
The algorithm is efficient and very easy to implement, and it has near-optimal relative-error bound.
Third, we have proposed an extension called the spectral shifting (SS) model.
We have shown that SS has much stronger error bound than the low-rank approximation models,
especially when the bottom eigenvalues are not sufficiently small.

Although the prototype model is not very time-efficient,
we can resort to the approximate method provided by \citet{wang2015towards} to obtain a faster solution to the prototype model.
This faster solution requires only linear time and linear memory.
The theoretical analysis of the fast solution heavily relies on that of the prototype model,
so our established results are useful even if the prototype model itself is not the working horse in real-world applications.
In addition, we have  shown that the fast solution can also be naturally incorporated with the spectral shifting method.


\acks{We thank the anonymous reviewers for their helpful suggestions.
Wang has been supported by Baidu Scholarship.
Luo and Zhang have been supported by  the National Natural Science Foundation of China (No. 61572017),
Natural Science Foundation of Shanghai City (No. 15ZR1424200), and Microsoft Research Asia Collaborative Research Award.}


\appendix

\section{Proof of Theorem~\ref{thm:exact}}

\begin{proof}
Suppose that $\rk(\W) = \rk(\K)$.
We have that $\rk(\W) = \rk(\C) = \rk(\K)$ because
\begin{equation} \label{eq:thm:connection:4}
\rk(\K) \;\geq\; \rk(\C) \;\geq\; \rk(\W) \textrm{.}
\end{equation}
Thus there exists a matrix $\X$ such that
\[
\left[
  \begin{array}{c}
    \K_{2 1}^T \\
    \K_{2 2} \\
  \end{array}
\right]
\; = \;
\C \X^T
\; = \;
\left[
  \begin{array}{c}
    \W \X^T \\
    \K_{2 1} \X^T \\
  \end{array}
\right] \textrm{,}
\]
and it follows that $\K_{2 1} = \X \W$
and $\K_{2 2} = \K_{2 1} \X^T = \X \W \X^T$.
Then we have that
\begin{eqnarray}
\K
& = & \left[
        \begin{array}{cc}
          \W & (\X \W)^T \\
          \X \W & \X \W \X^T \\
        \end{array}
      \right]
\; = \;
    \left[
        \begin{array}{c}
          \I \\
          \X \\
        \end{array}
      \right]
      \W
      \left[
        \begin{array}{cc}
          \I & \X^T \\
        \end{array}
      \right] \textrm{,} \label{eq:thm:connection:1} \\
\C \W^\dag \C^T
& = &
    \left[
        \begin{array}{c}
          \W \\
          \X \W \\
        \end{array}
      \right]
      \W^\dag
      \left[
        \begin{array}{cc}
          \W & (\X \W)^T \\
        \end{array}
      \right]
\; = \;
    \left[
        \begin{array}{c}
          \I \\
          \X \\
        \end{array}
      \right]
      \W
      \left[
        \begin{array}{cc}
          \I & \X^T \\
        \end{array}
      \right] \textrm{.} \label{eq:thm:connection:2}
\end{eqnarray}
Here the second equality in (\ref{eq:thm:connection:2})
follows from $\W \W^\dag \W = \W$.
We obtain that $\K = \C \W^\dag \C$.
Then we show that $\K = \C \C^\dag \K (\C^\dag)^T \C^T$.

Since $\C^\dag = (\C^T \C)^\dag \C^T$, we have that
\[
\C^\dag
\;=\;
\big( \W (\I + \X^T \X) \W \big)^\dag \W \: [\I \, , \, \X^T] \textrm{,}
\]
and thus
\begin{align}
&\C^\dag \K (\C^\dag)^T \W \nonumber \\
& =\;
\big( \W (\I + \X^T \X) \W \big)^\dag
\W (\I + \X^T \X) \Big[ \W (\I + \X^T \X)  \W \big( \W (\I + \X^T \X) \W \big)^\dag \W \Big]\nonumber\\
&=\;
\big( \W (\I + \X^T \X) \W \big)^\dag
\W (\I + \X^T \X) \W \textrm{,} \nonumber
\end{align}
where the second equality follows from Lemma~\ref{lem:thm:connection}
because $(\I + \X^T \X)$ is positive definite.
Similarly we have
\begin{align}
\W \C^\dag \K (\C^\dag)^T \W
\;=\;
\W \big( \W (\I + \X^T \X) \W \big)^\dag
\W (\I + \X^T \X) \W
\;=\; \W\textrm{.} \nonumber
\end{align}
Thus we have
\begin{small}
\begin{eqnarray}
\C \C^\dag \K (\C^\dag)^T \C
\;=\;    \left[
        \begin{array}{c}
          \I \\
          \X \\
        \end{array}
      \right]
      \W \C^\dag \K (\C^\dag)^T \W
      \left[
        \begin{array}{cc}
          \I & \X^T \\
        \end{array}
      \right]
\;= \;    \left[
        \begin{array}{c}
          \I \\
          \X \\
        \end{array}
      \right]
      \W
      \left[
        \begin{array}{cc}
          \I & \X^T \\
        \end{array}
      \right] \textrm{.}  \label{eq:thm:connection:3}
\end{eqnarray}
\end{small}
It follows from Equations (\ref{eq:thm:connection:1}) (\ref{eq:thm:connection:2}) (\ref{eq:thm:connection:3})
that $\K = \C \W^\dag \C^T = \C \C^\dag \K (\C^\dag)^T \C^T$.

Conversely, when $\K = \C \W^\dag \C^T$, it holds that $\rk(\K) \leq \rk(\W^\dag) = \rk(\W)$.
It follows from (\ref{eq:thm:connection:4}) that $\rk(\K) = \rk(\W)$.

When $\K = \C \C^\dag \K (\C^\dag)^T \C^T$,
we have $\rk(\K) \leq \rk(\C)$.
Thus there exists a matrix $\X$ such that
\[
\left[
  \begin{array}{c}
    \K_{2 1}^T \\
    \K_{2 2} \\
  \end{array}
\right]
\; = \;
\C \X^T
\; = \;
\left[
  \begin{array}{c}
    \W \X^T \\
    \K_{2 1} \X^T \\
  \end{array}
\right] \textrm{,}
\]
and therefore $\K_{2 1} = \X \W$.
Then we have that
\[
\C \;=\;
\left[
  \begin{array}{c}
    \W \\
    \K_{2 1} \\
  \end{array}
\right]
\;=\;
\left[
  \begin{array}{c}
    \I \\
    \X \\
  \end{array}
\right]
\W \textrm{,}
\]
so $\rk(\C) \leq \rk(\W)$.
Apply (\ref{eq:thm:connection:4}) again we have $\rk(\K) = \rk(\W)$.
\end{proof}

\begin{lemma} \label{lem:thm:connection}
$\X^T \V \X \big( \X^T \V \X\big)^\dag \X^T = \X^T$ for any positive definite matrix $\V$.
\end{lemma}

\begin{proof}
The positive definite matrix $\V$ have a decomposition $\V= \B^T \B$ for some nonsingular matrix $\B$.
It follows that
\begin{eqnarray}
\X^T \V \X \big( \X^T \V \X\big)^\dag \X^T
& = & (\B \X)^T \Big(\B \X \big( (\B \X)^T (\B \X) \big)^\dag \Big) (\B \X)^T \B (\B^T \B)^{-1} \nonumber \\
& = & (\B \X)^T \big((\B \X)^T\big)^{\dag} (\B \X)^T (\B^T)^{-1}
\; = \; (\B \X)^T (\B^T)^{-1}
\; = \; \X^T \textrm{.} \nonumber
\end{eqnarray}
\end{proof}


\section{Proof of Theorem~\ref{thm:lower}} \label{sec:app:lower_bound}

In Section~\ref{sec:proof:lower_bound:lemma} we provide several key lemmas,
and then in Section~\ref{sec:proof:lower_bound:theorem} we prove Theorem~\ref{thm:lower} using
Lemmas \ref{lem:nystrom_residual_new} and~\ref{lem:lower_bound_modified}.


\subsection{Key Lemmas} \label{sec:proof:lower_bound:lemma}

Lemma \ref{lem:pinv_partitioned_matrix} provides a useful tool for expanding the Moore-Penrose inverse of partitioned matrices,
and the lemma will be used to prove Lemma~\ref{lem:lower_bound_modified} and Theorem~\ref{thm:lower}.

\begin{lemma}  \emph{\citep[see][Page~179]{adi2003inverse}} \label{lem:pinv_partitioned_matrix}
Given a matrix $\X\in \RBmn$ of rank $c$,  let it have a nonsingular $c\times c$ submatrix $\X_{1 1}$.
By rearrangement of columns and rows by permutation matrices $\PP$ and $\Q$,
the submatrix $\X_{1 1}$ can be bought to the top left corner of $\X$, that is,
\[
\PP \X \Q \; =\; \left[
             \begin{array}{cc}
               \X_{1 1} & \X_{1 2} \\
               \X_{2 1} & \X_{2 2} \\
             \end{array}
           \right].
\]
Then the Moore-Penrose inverse of $\X$ is
\begin{align}
\X^\dag \; =\;
\Q \left[
     \begin{array}{c}
       \I_c \\
       \T^T \\
     \end{array}
   \right]
\big( \I_c + \T \T^T \big)^{-1} \X_{1 1}^{-1}   \big( \I_c + \S^T \S \big)^{-1}
\left[
  \begin{array}{cc}
    \I_c & \S^T  \\
  \end{array}
\right] \PP ,\nonumber
\end{align}
where $\T = \X_{1 1}^{-1} \X_{1 2}$ and
$\S = \X_{2 1} \X_{1 1}^{-1}$.
\end{lemma}

Lemmas~\ref{lem:nystrom_residual_new} and \ref{lem:lower_bound_modified} will be used to prove Theorem~\ref{thm:lower}.

\begin{lemma} \emph{\citep[Lemma~19]{wang2013improving}}
\label{lem:nystrom_residual_new}
Given $n$ and $k$, we let $\B$ be an $\frac{n}{k} \times \frac{n}{k}$ matrix
whose diagonal entries equal to one and off-diagonal entries equal to $\alpha \in [0, 1)$.
Let $\A$ be the $n\times n$ block-diagonal matrix
\begin{eqnarray}\label{eq:construction_bad_nystrom_blk}
\A \; =\;  \underbrace{\B \oplus \B \oplus \cdots \oplus \B}_{k \textrm{ blocks}} .
\end{eqnarray}
Let $\A_k$ be the best rank-$k$ approximation to $\A$.
Then
\begin{equation}
\| \A - \A_k \|_F \; =\; (1-\alpha) \sqrt{n-k}  \textrm{.} \nonumber
\end{equation}
\end{lemma}

\begin{lemma} \label{lem:lower_bound_modified}
Let $\B$ be the $n\times n$ matrix with diagonal entries equal to one and off-diagonal entries equal to $\alpha$
and $\tilde\B$ be its rank $c$ approximation formed by the prototype model.
Then
\begin{align}
\|\B - \tilde{\B} \|_F^2
\; \geq \; (1 - \alpha)^2 (n-c)\bigg(1+\frac{2}{c}- (1-\alpha) \frac{1+o(1)}{\alpha c n /2}  \bigg) \textrm{.} \nonumber
\end{align}
\end{lemma}

\begin{proof}
Without loss of generality, we assume the first $c$ column of $\B$ are selected to construct $\C$.
We partition $\B$ and $\C$ as:
\[
\B = \left[
       \begin{array}{cc}
         \W & \B_{2 1}^T \\
         \B_{2 1} & \B_{2 2} \\
       \end{array}
     \right]
\qquad
\textrm{ and }
\qquad
\C = \left[
       \begin{array}{c}
         \W  \\
         \B_{2 1} \\
       \end{array}
     \right] \textrm{.}
\]
Here the matrix $\W$ can be expressed by $\W = (1-\alpha) \I_{c} + \alpha \1_{c} \1_{c}^T$.
We apply the Sherman-Morrison-Woodbury matrix identity
\[
(\X + \Y \Z \R)^{-1}  = \X^{-1} - \X^{-1} \Y (\Z^{-1} + \R \X^{-1} \Y)^{-1} \R \X^{-1}
\]
to compute $\W^{-1}$, and it follows that
\begin{equation} \label{eq:lower_bound_improved:pinv_W}
\W^{-1} \; = \; \frac{1}{1-\alpha} \I_c - \frac{\alpha}{(1-\alpha)(1-\alpha+c\alpha)} \1_{c} \1_{c}^T \textrm{.}
\end{equation}
We expand the Moore-Penrose inverse of $\C$ by Lemma~\ref{lem:pinv_partitioned_matrix} and obtain
\[
\C^\dag = \W^{-1} \big( \I_c + \S^T \S \big)^{-1}
\left[
  \begin{array}{cc}
    \I_c & \S^T \\
  \end{array}
\right]
\]
where
\[
\S = \B_{2 1} \W^{-1} = \frac{\alpha}{1-\alpha + c\alpha} \1_{n-c} \1_c^T.
\]
It is easily verified that $\S^T \S = \big( \frac{\alpha}{1-\alpha + c\alpha} \big)^2 (n-c) \1_c \1_c^T$.

Now we express the approximation by the prototype model in the partitioned form:
\begin{align}
&\tilde{\B}
\; = \; \C \C^\dag \B \big(\C^\dag\big)^T \C^T \nonumber \\
&= \; \left[
      \begin{array}{c}
        \W \\
        \B_{2 1} \\
      \end{array}
    \right] \W^{-1} \big( \I_c + \S^T \S \big)^{-1}
    \left[
      \begin{array}{c c}
        \I_c & \S^T \\
      \end{array}
    \right]
    \B
 \left[
      \begin{array}{c}
        \I_c \\
        \S \\
      \end{array}
    \right]
     \big( \I_c + \S^T \S \big)^{-1} \W^{-1}
     \left[
      \begin{array}{c}
        \W \\
        \B_{2 1} \\
      \end{array}
    \right]^T \nonumber \\
&=\; \left[
      \begin{array}{c}
        \big( \I_c + \S^T \S \big)^{-1} \\
        \B_{2 1}\W^{-1} \big( \I_c + \S^T \S \big)^{-1} \\
      \end{array}
    \right]
    \left[
      \begin{array}{c c}
        \I_c & \S^T \\
      \end{array}
    \right]
    \B
  \left[
      \begin{array}{c}
        \I_c \\
        \S \\
      \end{array}
    \right]
     \left[
      \begin{array}{c}
        \big( \I_c + \S^T \S \big)^{-1} \\
        \B_{2 1}\W^{-1} \big( \I_c + \S^T \S \big)^{-1} \\
      \end{array}
    \right]^T \textrm{.} \label{eq:lower_bound_improved:partitioned_Bimp}
\end{align}
We then compute the submatrices $\big(\I_c + \S^T \S\big)^{-1}$ and $\B_{2 1} \W^{-1} \big(\I_c + \S^T \S \big)^{-1}$ respectively as follows.
We apply the Sherman-Morrison-Woodbury matrix identity
to compute $\big(\I_c + \S^T \S\big)^{-1}$. We obtain
\begin{eqnarray} \label{eq:lower_bound_improved:IcSS}
\big(\I_c + \S^T \S\big)^{-1}
\; = \; \bigg( \I_c + \Big( \frac{\alpha}{1-\alpha + c\alpha} \Big)^2 (n-c) \1_c \1_c^T \bigg)^{-1}
\;= \; \I_c - \gamma_1 \1_c \1_c^T \textrm{,}
\end{eqnarray}
where
\begin{eqnarray}
\gamma_1 = \frac{n-c}{n c + \big( \frac{1-\alpha}{\alpha}\big)^2 + \frac{2(1-\alpha)c}{\alpha}} . \nonumber
\end{eqnarray}
It follows from (\ref{eq:lower_bound_improved:pinv_W}) and (\ref{eq:lower_bound_improved:IcSS}) that
\begin{align}
\W^{-1} \big(\I_c + \S^T \S \big)^{-1}
\;=\; (\gamma_2 \I_c - \gamma_3 \1_c \1_c^T) ( \I_c - \gamma_1 \1_c \1_c^T )
\; =\; \gamma_2 \I_c + (\gamma_1 \gamma_3 c - \gamma_1 \gamma_2 - \gamma_3) \1_c \1_c^T , \nonumber
\end{align}
where
\[
\gamma_2 = \frac{1}{1-\alpha}  \quad \textrm{ and } \quad
\gamma_3 = \frac{\alpha}{(1-\alpha)(1-\alpha + \alpha c)}.
\]
It follows that
\begin{align} \label{eq:lower_bound_improved:BWIcSS}
\B_{2 1} \W^{-1} \big(\I_c + \S^T \S \big)^{-1}
\; =\; \alpha \big( \gamma_1 \gamma_3 c^2 - \gamma_3 c - \gamma_1 \gamma_2 c + \gamma_2 \big) \1_{n-c} \1_c^T
\; \triangleq \; \gamma \1_{n-c} \1_c^T ,
\end{align}
where
\begin{eqnarray}
\gamma &=& \alpha \big( \gamma_1 \gamma_3 c^2 - \gamma_3 c - \gamma_1 \gamma_2 c + \gamma_2 \big)
\;=\; \frac{\alpha (\alpha c - \alpha + 1)}{2 \alpha c - 2 \alpha - 2 \alpha^2 c + \alpha^2 + \alpha^2 c n + 1}\textrm{.}
\end{eqnarray}

Since $\B_{2 1}=\alpha \1_{n-c} \1_c^T$ and $\B_{2 2} = (1-\alpha)\I_{n-c}+ \alpha\1_{n-c} \1_{n-c}^T$,
it is easily verified that
\begin{align} \label{eq:lower_bound_improved:middle_term}
\left[
  \begin{array}{cc}
    \I_c & \S^T \\
  \end{array}
\right]
\B
\left[
  \begin{array}{c}
    \I_c \\
    \S \\
  \end{array}
\right]
\;  = \;
\left[
  \begin{array}{cc}
    \I_c & \S^T \\
  \end{array}
\right]
\left[
       \begin{array}{cc}
         \W & \B_{2 1}^T \\
         \B_{2 1} & \B_{2 2} \\
       \end{array}
\right]
\left[
  \begin{array}{c}
    \I_c \\
    \S \\
  \end{array}
\right]
\; = \; (1-\alpha) \I_c + \lambda \1_c \1_c^T \textrm{,}
\end{align}
where
\begin{small}
\[
\lambda = \frac{\alpha (3 \alpha n - \alpha c - 2 \alpha + \alpha^2 c - 3 \alpha^2 n + \alpha^2 + \alpha^2 n^2 + 1)}{(\alpha c - \alpha + 1)^2}
\]
\end{small}

It follows from (\ref{eq:lower_bound_improved:partitioned_Bimp}), (\ref{eq:lower_bound_improved:IcSS}),
(\ref{eq:lower_bound_improved:BWIcSS}), and (\ref{eq:lower_bound_improved:middle_term}) that
\begin{eqnarray}
\tilde{\B}
\; = \;
\left[
  \begin{array}{c}
    \I_c - \gamma_1 \1_c \1_c^T \\
    \gamma \1_{n-c} \1_c^T \\
  \end{array}
\right]
\Big( (1-\alpha) \I_c + \lambda \1_c \1_c^T \Big)
\left[
  \begin{array}{c}
    \I_c - \gamma_1 \1_c \1_c^T \\
    \gamma \1_{n-c} \1_c^T \\
  \end{array}
\right]^T
\; \triangleq \;
\left[
  \begin{array}{cc}
    \tilde{\B}_{1 1} & \tilde{\B}_{2 1}^T \\
    \tilde{\B}_{2 1} & \tilde{\B}_{2 2} \\
  \end{array}
\right], \nonumber
\end{eqnarray}
where
\begin{eqnarray}
\tilde{\B}_{1 1}
& = & (1-\alpha) \I_c + \big[ (1-\gamma_1 c)   (\lambda - \lambda \gamma_1 c - (1-\alpha) \gamma_1) - (1-\alpha) \gamma_1 \big] \1_c \1_c^T \nonumber \\
& = & (1-\alpha) \I_c + \eta_1 \1_c \1_c^T , \nonumber \\
\tilde{\B}_{2 1}  & = & \tilde{\A}_{1 2}^T
\; = \; \gamma (1-\gamma_1 c) (1-\alpha + \lambda c) \1_{n-c} \1_c^T
\; = \; \eta_2  \1_{n-c} \1_c^T  , \nonumber \\
\tilde{\B}_{2 2}
& = & \gamma^2 c (1-\alpha + \lambda c) \1_{n-c} \1_{n-c}^T
\; = \; \eta_3 \1_{n-c} \1_{n-c}^T  , \nonumber
\end{eqnarray}
where
\begin{eqnarray}
\eta_1
& = &
 (1-\gamma_1 c)(\lambda - \lambda \gamma_1 c - (1-\alpha) \gamma_1) - (1-\alpha) \gamma_1  , \nonumber\\
\eta_2
& = & \gamma (1-\gamma_1 c) (1-\alpha + \lambda c),   \nonumber \\
\eta_3
& = & \gamma^2 c (1-\alpha + \lambda c) ,\nonumber
\end{eqnarray}
By dealing with the four blocks of $\tilde{\B}$ respectively, we finally obtain that
\begin{eqnarray}
\|\B - \tilde{\B} \|_F^2
& = & \|\W - \tilde{\B}_{1 1}\|_F^2 + 2\|\B_{2 1} - \tilde{\B}_{2 1}\|_F^2 + \|\B_{2 2} - \tilde{\B}_{2 2}\|_F^2  \nonumber\\
& = &  c^2 (\alpha - \eta_1)^2 + 2 c (n-c) (\alpha - \eta_2)^2  \nonumber\\
& & + (n-c)(n-c-1)(\alpha - \eta_3)^2 + (n-c) (1 - \eta_3)^2 \nonumber\\
& = & (n-c)(\alpha-1)^2\bigg(1+\frac{2}{c}- \big(1+o(1)\big) \frac{1-\alpha}{\alpha c n /2}  \bigg) \textrm{.} \nonumber
\end{eqnarray}
\end{proof}


\subsection{Proof of the Theorem} \label{sec:proof:lower_bound:theorem}

Now we prove Theorem~\ref{thm:lower} using Lemma~\ref{lem:lower_bound_modified} and Lemma~\ref{lem:nystrom_residual_new}.
Let $\C$ consist of $c$ columns sampled from $\A$ and $\hat{\C}_i$ consist of $c_i$ columns sampled from the $i$-th block diagonal matrix in $\A$.
Without loss of generality, we assume $\hat{\C}_i$ consists of the first $c_i$ columns of $\B$.
Then the intersection matrix $\U$ is computed by
\begin{eqnarray}
\U
& = & \C^\dag \A \big(\C^T\big)^\dag
\; = \; \big[ \hat{\C}_1 \oplus \cdots \oplus \hat{\C}_k  \big]^\dag
        \big[ \B \oplus \cdots \oplus \B \big]
        \big[  \hat{\C}_1^T \oplus \cdots \oplus \hat{\C}_k^T  \big]^\dag \nonumber\\
& = &  \hat{\C}_1^\dag \B \big(\hat{\C}_1^\dag\big)^T \oplus \cdots \oplus\hat{\C}_k^\dag \B \big(\hat{\C}_k^\dag\big)^T \textrm{.} \nonumber
\end{eqnarray}
Let $\tilde\A$ be the approximation formed by the prototype model. Then
\begin{align}
\tilde{\A}
\; = \; \C \U \C^T
\; = \; \hat{\C}_1 \hat{\C}_1^\dag \B \big(\hat{\C}_1^\dag\big)^T \hat{\C}_1^T \oplus
        \cdots \oplus \hat{\C}_k \hat{\C}_k^\dag \B \big(\hat{\C}_k^\dag\big)^T \hat{\C}_k^T   \textrm{,} \nonumber
\end{align}
and thus the approximation error is
\begin{eqnarray}
\big\|\A - \tilde{\A} \big\|_F^2
& = & \sum_{i=1}^{k} \Big\| \B \;-\; \hat{\C}_i \hat{\C}_i^\dag \B \big(\hat{\C}_i^\dag\big)^T \hat{\C}_i^T \Big\|_F^2 \nonumber\\
& \geq & (1 - \alpha)^2 \sum_{i=1}^{k}(p - c_i)\bigg(1+\frac{2}{c_i}- (1-\alpha) \Big(\frac{1+o(1)}{\alpha c_i p /2} \Big) \bigg) \nonumber \\
& = & (1-\alpha)^2 \bigg( \sum_{i=1}^{k}(p - c_i) + \sum_{i=1}^{k}\frac{2(p - c_i)}{c_i}\Big( 1 - \frac{(1-\alpha)(1+o(1))}{\alpha  p} \Big) \bigg) \nonumber \\
& \geq &(1-\alpha)^2 (n - c) \bigg( 1 + \frac{2k}{c}\Big( 1 - \frac{k(1-\alpha)(1+o(1))}{\alpha n} \Big) \bigg)\textrm{,} \nonumber
\end{eqnarray}
where the former inequality follows from Lemma~\ref{lem:lower_bound_modified},
and the latter inequality follows by minimizing over $c_1 , \cdots , c_k$.
Finally the theorem follows by setting $\alpha \rightarrow 1$ and applying Lemma~\ref{lem:nystrom_residual_new}.


\section{Proof of Theorem~\ref{thm:near_opt}}

\begin{proof}
Sampling $c_1 = 2 k \epsilon^{-1} \big( 1+o(1) \big)$ columns of $\K$ by the near-optimal algorithm of \cite{boutsidis2011NOCshort}
to form $\C_1 \in \RB^{n\times c_1}$,
we have that
\begin{equation*}
\EB \big\| \K - \PM_{\C_1, k} (\K) \big\|_F^2
\; \leq \; (1+\epsilon) \big\| \K - \K_k \big\|_F^2 .
\end{equation*}
Applying Lemma 3.11 of \cite{boutsidis2014optimal},
we can find a much smaller matrix $\U_1\in \RB^{n\times k}$ with orthogonal columns in the column space of $\C_1$
such that
\begin{equation*}
\big\| \K - \U_1 \U_1^T \K \big\|_F^2
\; \leq \;
\big\| \K - \PM_{\C_1, k} (\K) \big\|_F^2.
\end{equation*}
(We do not actually compute $\U_1$ because the adaptive sampling algorithm does not need to know $\U_1$.)
Because the columns of $\U_1$ are all in the columns space of $\C_1$,
we have that $\big\| \K - \C_1 \C_1^\dag \K \big\|_F^2 \leq \big\| \K - \U_1 \U_1^T \K \big\|_F^2$.
Combining the above inequalities we obtain
\begin{equation*}
\EB\big\| \K - \C_1 \C_1^\dag \K \big\|_F^2
\; \leq \; \EB \big\| \K - \U_1 \U_1^T \K \big\|_F^2
\; \leq \; \big\| \K - \PM_{\C_1, k} (\K) \big\|_F^2.
\end{equation*}

Given $\C_1$, we use adaptive sampling to select
$c_2 = k \epsilon^{-1}  $ rows of $\K$ to form $\C_2$ and denote $\C = [\C_1 , \C_2]$.
Since the columns of $\U_1$ are all in the columns space of $\C_1$,
Lemma 17 of \cite{wang2013improving} can be slightly modified to show
\begin{eqnarray*}
\big\| \K - \C_1 \C_1^\dag \K (\C^T)^\dag \C^T \big\|_F^2
& \leq & \big\| \K - \U_1 \U_1^T \K (\C^T)^\dag \C^T \big\|_F^2 .
\end{eqnarray*}
By the adaptive sampling theorem of \cite{wang2013improving} we have
\begin{eqnarray} \label{eq:adaptive_sampling_expectation}
\EB \big\| \K - \C_1 \C_1^\dag \K (\C^T)^\dag \C^T \big\|_F^2
& \leq & \EB \big\| \K - \U_1 \U_1^T \K (\C^T)^\dag \C^T \big\|_F^2 \nonumber \\
& \leq &  \big\| \K - \U_1 \U_1^T \K \big\|_F^2 + \frac{k}{c_2}  \big\| \K - \K (\C_1^T)^\dag \C_1^T \big\|_F^2 \nonumber \\
& \leq & \Big(1 + \frac{k}{c_2}  \Big) \big\| \K - \U_1 \U_1^\dag \K \big\|_F^2 \nonumber\\
& \leq & (1+\epsilon) \big\| \K - \PM_{\C_1, k} (\K) \big\|_F^2  ,
\end{eqnarray}
where the expectation is taken w.r.t.\ $\C_2$, and the last inequality follows by setting $c_2 = k/\epsilon$.
Here the trick is bounding $\EB \big\| \K - \U_1 \U_1^T \K (\C^T)^\dag \C^T \big\|_F^2$ rather than
directly bounding $\EB \big\| \K - \C_1 \C_1^\dag \K (\C^T)^\dag \C^T \big\|_F^2$;
otherwise the factor $\frac{k}{c_2}$ would be $\frac{c_1}{c_2} = \frac{2k\epsilon^{-1} (1+o(1))}{c_2}$.
This is the key to the improvement.
It follows that
\begin{eqnarray*}
\EB \big\| \K - \C_1 \C_1^\dag \K (\C^T)^\dag \C^T \big\|_F^2
& \leq & (1 + \epsilon )  \EB \big\| \K - \PM_{\C_1, k} (\K) \big\|_F^2 \\
& \leq & (1+ \epsilon)^2 \big\| \K - \K_k \big\|_F^2  ,
\end{eqnarray*}
where the first expectation is taken w.r.t.\ $\C_1$ and $\C_2$,
and the second expectation is taken w.r.t.\ $\C_1$.
Applying Lemma 17 of \cite{wang2013improving} again, we obtain
\begin{eqnarray*}
\EB \big\| \K - \C \C^\dag \K (\C^T)^\dag \C^T \big\|_F^2
& \leq & \EB \big\| \K - \C_1 \C_1^\dag \K (\C^T)^\dag \C^T \big\|_F^2
\; \leq \; (1+ \epsilon)^2 \big\| \K - \K_k \big\|_F^2 .
\end{eqnarray*}
Hence totally $c =c_1 + c_2= 3 k \epsilon^{-1} \big(1 + o(1) \big)$ columns suffice.
\end{proof}


\section{Proof of Theorem~\ref{thm:efficient}} \label{sec:app:thm:efficient}

In this section we first provide a constant factor bound of the uniform sampling,
and then prove Theorem~\ref{thm:efficient} in the subsequent subsections.


\subsection{Tools for Analyzing Uniform Sampling}

This subsection provides several useful tools for analyzing column sampling.
The matrix $\S \in \RB^{n\times s}$ is a column selection matrix if each column has exactly one nonzero entry;
let $(i_j, j)$ be the position of the nonzero entry in the $j$-th column.
Let us add randomness to column selection.
Suppose we are given the sampling probabilities $p_1 , \cdots , p_n \in (0, 1)$ and $\sum_i p_i = 1$.
In each round we pick one element in $[n]$ such that the $i$-th element is sampled with probability $ p_i$.
We repeat the procedure $s$ times, either with or without replacement, and let $i_1 , \cdots , i_s$ be the selected indices.
For $j = 1$ to $s$, we set
\begin{equation} \label{eq:def_S}
S_{i_j, j} = \frac{1}{\sqrt{s p_{i_j}}} .
\end{equation}
The following lemma shows an important property of arbitrary column sampling.
The proof mirrors the leverage score sampling bound in \citep{woodruff2014sketching}.

\begin{lemma}\label{lem:sampling_property1}
Let $\U\in \RB^{n\times k}$ be any fixed matrix with orthonormal columns.
The column selection matrix $\S \in \RB^{n\times s}$ samples $s$
columns according to arbitrary probabilities $p_1 , p_2 , \cdots , p_n$.
Assume that
\begin{equation*}
\max_{i\in [n]} \frac{\| \u_{i:} \|_2^2}{p_{i}}
\; \leq \; \alpha
\end{equation*}
and $\alpha \geq 1$.
When
$s \; \geq \; \alpha \frac{6 + 2\eta}{3 \eta^2} \log (k/\delta)$,
it holds that
\[
\PB \Big\{ \big\| \I_k - \U^T \S \S^T \U \big\|_2 \; \geq \; \eta \Big\}
\; \leq \; \delta .
\]
\end{lemma}

\begin{proof}
We can express $\I_k = \U^T \U$ as the sum of size $k\times k$ and rank one matrices:
\begin{eqnarray*}
\I_k & = & \U^T \U
\; = \; \sum_{i=1}^n \u_{i:}^T \u_{i:}
\end{eqnarray*}
where $\u_{i:} \in \RB^{1\times k}$ is the $i$-th row of $\U$.
The approximate matrix product can be expressed as
\begin{eqnarray*}
\U^T \S \S^T \U
& = &  \sum_{j=1}^s S_{i_j, j}^2 \u_{i_j:}^T \u_{i_j:}
\; = \; \sum_{j=1}^s \frac{1}{s p_{i_j}} \u_{i_j:}^T \u_{i_j:}.
\end{eqnarray*}
We define the symmetric random matrices
\[
\Z_{i_j}
\; = \; \frac{1}{s p_{i_j}} \u_{i_j:}^T \u_{i_j:} - \frac{1}{s}\I_k
\]
for $j = 1$ to $s$.
Whatever the sampling distribution is, it always holds that
\begin{eqnarray} \label{eq:sampling_mean}
\EB \Z_{i_j}
\; = \; \sum_{q=1}^n p_q \Big(\frac{1}{\sqrt{s p_q}} \Big)^2 \u_{q:}^T \u_{q:} - \sum_{q=1}^n p_q \frac{1}{s}\I_k
\; = \; \frac{1}{s} \I_k - \frac{1}{s} \I_k
\; = \; \0 .
\end{eqnarray}
Thus $\Z_{i_j}$ has zero mean.
Let $\ZM$ be the set
\[
\ZM \; = \;
\bigg\{ \frac{1}{s}\I_k - \frac{1}{s p_{1}} \u_{1:}^T \u_{1:} ,
\; \cdots , \; \frac{1}{s}\I_k - \frac{1}{s p_{n}} \u_{n:}^T \u_{n:} \bigg\} .
\]
Clearly, $\Z_{i_1}, \, \cdots , \, \Z_{i_s} $ are sampled from $\ZM$,
and
\[
\U^T \S \S^T \U - \I_k = \sum_{j=1}^s \Z_{i_j} .
\]
Therefore we can bound its spectral norm using the matrix Bernstein.
Elementary proof of the matrix Bernstein can be found in \citet{tropp2015introduction}.

\begin{lemma} [Matrix Bernstein]
Consider a finite sequence $\{ \Z_i \}$ of independent, random, Hermitian matrices with dimension $k$.
Assume that
\[
\EB \Z_i = \0 \quad \textrm{ and } \quad
\max_i \| \Z_i \|_2 \leq L
\]
for each index $i$.
Introduce the random matrix $\Y = \sum_i \Z_i$.
Let $v (\Y)$ be the matrix variance statistics of the sum:
\[
v (\Y)
\; = \;
\Big\| \EB \Y^2 \Big\|_2
\; = \; \Big\| \sum_i \EB \Z_i^2 \Big\|_2 .
\]
Then
\begin{eqnarray*}
\PB \big\{ \lambda_{\max } (\Y) \geq \eta \big\}
& \leq & k \cdot \exp \bigg( \frac{ -\eta^2 /2 }{ v(\Y) + L \eta /3 } \bigg).
\end{eqnarray*}
\end{lemma}

To apply the matrix Bernstein,
it remains to bound the variance of $ \sum_{j=1}^s \Z_{i_j}$ and to bound  $L =\max_{j\in [s]} \| \Z_{i_j} \|_2$.
We have that
\begin{eqnarray*}
\EB \Z_{i_j}^2
& = & \EB \Big( \frac{1}{s p_{i_j}} \u_{i_j:}^T \u_{i_j:} - \frac{1}{s}\I_k \Big)^2 \\
& = & \frac{1}{s^2} \bigg[  \EB \Big( \frac{1}{p_{i_j}} \u_{i_j:}^T \u_{i_j:} \Big)^2
        - 2 \EB \Big( \frac{1}{p_{i_j}} \u_{i_j:}^T \u_{i_j:} \Big) + \I_k \bigg]\\
& = & \frac{1}{s^2} \bigg[  \sum_{q=1}^n \Big( p_q  \Big(\frac{1}{p_q} \Big)^2 \u_{q:}^T \u_{q:} \u_{q:}^T \u_{q:} \Big)
        - 2 \I_k + \I_k \bigg]\\
& = &  - \frac{1}{s^2} \I_k + \frac{1}{s^2} \sum_{q=1}^n \frac{ \| \u_{q:} \|_2^2}{ p_q }  \u_{q:}^T  \u_{q:} .
\end{eqnarray*}
The variance is defined by
\begin{eqnarray*}
v
& \triangleq & \bigg\| \sum_{j=1}^s \EB \Z_{i_j}^2 \bigg\|_2
\; = \; \frac{1}{s} \bigg\| - \I_k + \sum_{q=1}^n \frac{ \| \u_{q:} \|_2^2}{ p_q }  \u_{q:}^T  \u_{q:} \bigg\|_2
\; \leq \; \frac{1}{s} (\alpha-1).
\end{eqnarray*}
Here the inequality is due to the definition of $\alpha$ and
\[
- \I_k + \sum_{q=1}^n \frac{ \| \u_{q:} \|_2^2}{ p_q }  \u_{q:}^T  \u_{q:}
\; \preceq \;  - \I_k + \sum_{q=1}^n \alpha  \u_{q:}^T  \u_{q:}
\; = \; (-1+\alpha) \I_k.
\]
In addition, $L = \max_{j\in[s]} \| \Z_{i_j} \|_2$ can be bounded by
\begin{align}
& L
\; = \; \max_{j\in[s]} \| \Z_{i_j} \|_2
\; = \; \max_{j\in[s]} \bigg\|  \frac{1}{s} \I_k - \frac{1}{s p_{i_j}} \u_{i_j :}^T \u_{i_j :} \bigg\|_2
\; \leq \; \max_{i\in [n]} \bigg\|  \frac{1}{s} \I_k - \frac{1}{s p_{i}} \u_{i:}^T \u_{i:} \bigg\|_2
\nonumber \\
& \leq \;  \frac{1}{s} + \max_{i\in [n]} \bigg\|  \frac{1}{s p_{i}} \u_{i:}^T \u_{i:} \bigg\|_2
\; = \; \frac{1}{s} +  \max_{i\in [n]} \frac{\| \u_{i:} \|_2^2}{s p_{i}}
\; \leq \; \frac{1}{s} (\alpha + 1) \nonumber.
\end{align}
Finally, the lemma follows by plugging $v$ and $L$ in the matrix Bernstein.
\end{proof}

Theorem~\ref{thm:uniform_sampling_property1} shows an important property of uniform sampling.
It shows that when the number of sampled columns is large enough,
all the singular values of $\U^T \S \S^T \U $ are within $1\pm \eta$ with high probability.

\begin{theorem} \label{thm:uniform_sampling_property1}
Let $\U \in \RB^{n\times k}$ be any fixed matrix with orthonormal columns
and $\mu (\U)$ be its row coherence.
Let $\S \in \RB^{n\times s}$ be an uniformly sampling matrix.
When $s \geq \mu (\U) k \frac{6 + 2\eta}{3 \eta^2} \log (k/\delta)$,
it hold that
\[
\PB \Big\{ \big\| \U^T \S \S^T \U - \I_k \big\|_2 \; \geq \; \eta \Big\}
\; \leq \; \delta .
\]
\end{theorem}

\begin{proof}
Since $\frac{\|\u_{q:}\|_2^2 }{p_q} = n \|\u_{q:}\|_2^2 \leq k \mu(\U)$ for all $q \in [n]$,
the theorem follows from Lemma~\ref{lem:sampling_property1}.
\end{proof}

The following lemma was established by \citet{drineas2008cur}.
It can be proved by writing the squared Frobenius norm as the sum of scalars and then taking expectation.

\begin{lemma} \label{lem:sampling_property2}
Let $\A \in \RB^{n\times k}$ and $\B \in \RB^{n\times d}$ be any fixed matrices and
$\S \in \RB^{n\times s}$ be the column sampling matrix defined in \eqref{eq:def_S}.
Assume that the columns are selected randomly and pairwisely independently.
Then
\[
\EB \Big\| \A^T \B - \A^T \S \S^T \B \Big\|_F^2
\; \leq \;
\frac{1}{s} \sum_{i = 1}^n \frac{ 1 }{ p_i } \big\| \a_{i:} \big\|_2^2 \big\|\bb_{i:} \big\|_2^2 .
\]
\end{lemma}

Theorem~\ref{thm:uniform_sampling_property2} follows from Lemma~\ref{lem:sampling_property2}
and shows another important property of uniform sampling.

\begin{theorem} \label{thm:uniform_sampling_property2}
Let $\U \in \RB^{n\times k}$ be any fixed matrix with orthonormal columns and $\B \in \RB^{n\times d}$ be any fixed matrix.
Use the uniform sampling matrix $\S \in \RB^{n\times s}$ and set
$s \geq \frac{k \mu (\U) }{\epsilon \delta}$.
Then it holds that
\begin{eqnarray*}
\PB \Big\{ \big\| \U^T \B - \U^T \S^T \S \B \big\|_F^2
\; \geq \; \epsilon \|\B\|_F^2 \Big\}
\; \leq \; \delta .
\end{eqnarray*}
\end{theorem}

\begin{proof}
We have shown that $\|\u_{q:}\|_2^2 / p_q \leq k \mu (\U)$ for all $q = 1$ to $n$.
It follows from Lemma~\ref{lem:sampling_property2} that
\[
\EB \Big\| \U^T \B - \U^T \S \S^T \B \Big\|_F^2
\; \leq \;
\frac{1}{s} \sum_{i = 1}^n k \mu(\U) \big\| \bb_{i:} \big\|_2^2
\; = \; \frac{k \mu(\U)}{s} \|\B\|_F^2.
\]
The theorem follows from the setting of $s$ and the Markov's inequality.
\end{proof}

The following lemma is very useful in analyzing randomized SVD.

\begin{lemma}\label{lem:cx}
Let $\M \in \RB^{m\times n}$ be any matrix.
We decompose $\M$ by $\M = \M_1 + \M_2$ such that $\rk (\M_1) = k$.
Let the right singular vectors of $\M_1$ be $\V_1 \in \RB^{n\times k}$.
Let $\S \in \RB^{n\times c}$ be any matrix such that $\rk (\V_1^T \S) = k$
and let $\C = \M \S \in \RB^{m\times c}$. Then
\begin{eqnarray*}
\big\| \M - \C (\V_1^T \S)^\dag \V_{1}^T  \big\|_\xi^2
& \leq & \big\|\M_2\big\|_\xi^2 + \sigma_{\min}^{-2} (\V_1^T \S \S^T \V_1) \; \big\| \M_{2} \S \S^T \V_1 \big\|_\xi^2
\end{eqnarray*}
for $\xi = 2$ or $F$.
\end{lemma}

\begin{proof}
\citet{boutsidis2011NOCshort} showed that
\begin{eqnarray*}
\big\| \M - \C (\V_1^T \S)^\dag \V_{1}^T  \big\|_\xi^2
& \leq & \big\|\M_2\big\|_\xi^2 + \big\| \M_{2} \S (\V_1^T \S)^\dag \big\|_\xi^2 .
\end{eqnarray*}
Since $ \Y^T (\Y \Y^T)^\dag = \Y^\dag$ for any matrix $\Y$,
it follows that
\begin{align*}
&\big\|  \M_{2} \S (\V_1^T \S)^\dag   \big\|_\xi^2
\; = \; \big\|  \M_{2} \S (\V_1^T \S)^T (\V_1^T \S \S^T \V_1 )^\dag   \big\|_\xi^2 \\
& \leq \; \big\|  \M_{2} \S \S^T \V_1 \big\|_\xi^2 \; \big\| (\V_1^T \S \S^T \V_1 )^\dag   \big\|_2^2
\; = \; \big\|  \M_{2} \S \S^T \V_1^T   \big\|_\xi^2 \; \sigma_{\min}^{-2} (\V_1^T \S \S^T \V_1 ),
\end{align*}
by which the lemma follows.
\end{proof}


\subsection{Uniform Sampling Bound} \label{sec:lem:uniform_column}

Theorem~\ref{thm:cx_uniform} shows that uniform sampling can be applied to randomized SVD.
Specifically, if $\C$ consists of $c = \OM (k \mu_k / \epsilon + k \mu_k \log k)$ uniformly sampled columns of $\A$,
then $\|\A - \PM_{\C,k} ( \A ) \|_F^2 \leq (1+\epsilon) \|\A - \A_k\|_F^2$ holds with high probability,
where $\mu_k$ is the column coherence of $\A_k$.

\begin{theorem} \label{thm:cx_uniform}
Let $\A \in \RB^{m\times n}$ be any fixed matrix, $k$ ($\ll m, n$) be the target rank,
and $\mu_k$ be the column coherence of $\A_k$.
Let $\S \in \RB^{m\times c}$ be a uniform sampling matrix
and $\C = \A \S \in \RB^{m\times c}$.
When $c \geq \mu_k k \cdot \max\{ 20 \log (20k) , \, 45 / \epsilon \}$,
\[
\min_{\rk (\X) \leq k} \big\| \A - \C \X \big\|_F^2
\; \leq \; \big( 1 +  \epsilon  \big) \, \| \A - \A_k \|_F^2.
\]
holds with probability at least $0.9$.
\end{theorem}

\begin{proof}
Let $\V_k \in \RB^{n\times k}$ contain the top $k$ right singular vectors of $\A$.
Obviously $\mu_k$ is the row coherence of $\V_k$.
We apply Theorem~\ref{thm:uniform_sampling_property1} with $\delta = 0.05$, $\eta = 1/3$, $s = 20 \mu_k k \log (20 k)$ and obtain
\[
\PB \Big\{ \sigma_{\min} ( \V_k^T \S \S^T \V_k ) \; \leq \; 2/3 \Big\}
\; \leq \; 0.05 .
\]
We then apply Theorem~\ref{thm:uniform_sampling_property2} with $\delta = 0.05$ and $s = 45 \mu_k k / \epsilon$ and obtain
\begin{eqnarray*}
\PB \Big\{ \big\|  \U^T \S^T \S (\A - \A_k) \big\|_F^2
\; \geq \; \frac{4}{9} \epsilon \|\A - \A_k\|_F^2 \Big\}
\; \leq \; 0.05 .
\end{eqnarray*}
It follows from Lemma~\ref{lem:cx} that
\begin{eqnarray*}
\min_{\rk (\X) \leq k}\|\A - \C  \X \|_F^2
& \leq & \|\A - \A_k\|_F^2 + \sigma_{\min}^{-2} ( \V_k^T \S \S^T \V_k ) \big\|  \U^T \S^T \S (\A - \A_k) \big\|_F^2 \\
& \leq & (1+\epsilon)  \|\A - \A_k\|_F^2,
\end{eqnarray*}
where the latter inequality holds with probability at least 0.9 (due to the union bound).
\end{proof}


\subsection{The Uniform+Adaptive Column Selection Algorithm}

In fact, running adaptive sampling only once yields a column sampling algorithm with $1+\epsilon$ bound for any $m\times n$ matrix.
We call it the uniform+adaptive column selection algorithm, which is a part of the uniform+adaptive$^2$ algorithm.
Though the algorithm is irrelevant to this work,
we describe it in the following for it is of independent interest.

\begin{theorem}[The Uniform+Adaptive Algorithm] \label{thm:uniform_adaptive1}
Given an $m\times n$ matrix $\A$,
we sample $c_1 = 20 \mu_k k \log (20 k)$ columns by uniform sampling to form $\C_1$
and sample additional $c_2 = 17.5 k / \epsilon $ columns by adaptive sampling to form $\C_2$.
Let $\C = [\C_1 , \C_2]$.
Then the inequality
\[
\min_{\rk (\X) \leq k} \big\| \A - \C \X \big\|_F^2
\; \leq \;
(1+  \epsilon) \big\| \A - \A_k \big\|_F^2
\]
holds with probability at least 0.8.
\end{theorem}

\begin{proof}
We apply Theorem~\ref{thm:cx_uniform} with $\epsilon = 0.75$ and $c_1 = 20 \mu_k k \log (20 k)$ and obtain that
\[
\|\A - \C_1 \C_1^\dag \A\|_F^2
\; \leq \; 1.75 \|\A - \A_k\|_F^2
\]
holds with probability at least 0.9.

\citet{deshpande2006matrix} showed that
\[
\EB \Big[ \min_{\rk (\X) \leq k} \big\| \A - \C \X \big\|_F^2 - \|\A - \A_k\|_F^2 \Big]
\; \leq \; \frac{k}{c_2} \|\A - \C_1 \C_1^\dag \A\|_F^2 .
\]
It follows from Markov's inequality that
\[
\min_{\rk (\X) \leq k} \big\| \A - \C \X \big\|_F^2 - \|\A - \A_k\|_F^2
\; \leq \;  \frac{k}{\delta_2 c_2} \|\A - \C_1 \C_1^\dag \A\|_F^2
\]
holds with probability at least $1 - \delta_2$.
We let $\delta_2 = 0.1$ and $c_2 = 17.5k/\epsilon$, and it follows that
\[
\min_{\rk (\X) \leq k} \big\| \A - \C \X \big\|_F^2
\; \leq \; (1+\epsilon) \|\A - \A_k\|_F^2
\]
holds with probability at least 0.8.
\end{proof}


\subsection{Proof of Theorem~\ref{thm:efficient}} \label{sec:thm:efficient}

We sample $c_1 = 20 \mu_k k \log (20 k)$ columns by uniform sampling to form $\C_1$
and sample additional $c_2 = 17.5 k / \epsilon $ columns by adaptive sampling to form $\C_2$.
Let $\hat\C = [\C_1 , \C_2]$.
It follows from Theorem~\ref{thm:uniform_adaptive1} that
\[
\big\| \K - \PM_{\hat\C, k} (\K) \big\|_F^2
\; \leq \;
(1+  \epsilon) \big\| \K - \K_k \big\|_F^2
\]
holds with probability at least 0.8.

Let $\C = [\hat\C , \C_3]$ where $\C_3$ consists of $c_3$ columns of $\K$ chosen by adaptive sampling.
Then
\begin{eqnarray*}
\EB  \big\| \K - \C \C^\dag \K (\C^T)^\dag \C^T \big\|_F^2
& \leq &  \EB \big\| \K - \hat\C \hat\C^\dag  \K (\C^T)^\dag \C^T \big\|_F^2   \\
& \leq &  \Big(1+\frac{k}{c_3} \Big) \big\| \K - \PM_{\hat\C, k} (\K) \big\|_F^2,
\end{eqnarray*}
where the former inequality follows from Lemma 17 of \cite{wang2013improving},
and the latter inequality follows from \eqref{eq:adaptive_sampling_expectation}.
It follows from Markov's inequality that
\[
\big\| \K - \C \C^\dag \K (\C^T)^\dag \C^T \big\|_F^2  - \big\| \K - \PM_{\hat\C, k} (\K) \big\|_F^2
\; \leq \; \frac{k}{\delta_3 c_3} \big\| \K - \PM_{\hat\C, k} (\K) \big\|_F^2
\]
holds with probability at least $1-\delta_3$.
We set $\delta_3 = 0.1$ and $c_3 = 10 k   / \epsilon$,
Then
\begin{eqnarray*}
\big\| \K - \C \C^\dag \K (\C^T)^\dag \C^T \big\|_F^2
\; \leq \; (1+ \epsilon) \big\| \K - \PM_{\hat\C, k} (\K) \big\|_F^2
\; \leq \; (1+ \epsilon)^2 \big\| \K - \K_k \big\|_F^2 ,
\end{eqnarray*}
where the former inequality holds with probability $1-\delta_3$,
and the latter inequality holds with probability $0.8 - \delta_3 = 0.7$.


\section{Proof of Theorem~\ref{thm:ss_closed_form}} \label{sec:proof_ss_closed_form}

In Section~\ref{sec:proof_ss_closed_form_solution} we derive the solution to the optimization problem (\ref{eq:def_ss_nystrom}).
In Section~\ref{sec:proof_ss_closed_form_optimal} we prove that the solutions are global optimum.
In Section~\ref{sec:proof_ss_closed_form_spsd} we prove that the resulting solution is positive (semi)definite when $\K$ is positive (semi)definite.


\subsection{Solution to the Optimization Problem (\ref{eq:def_ss_nystrom})} \label{sec:proof_ss_closed_form_solution}

We denote the objective function of the optimization problem (\ref{eq:def_ss_nystrom}) by
\[
f(\U, \delta) \; = \;
\big\| \K - \bar{\C} \U \bar{\C}^T - \delta \I_n \big\|_F^2.
\]
We take the derivative of $f(\U, \delta)$ w.r.t.\ $\U$ to be zero
\begin{eqnarray*}
    \frac{\partial f(\U, \delta)} {\partial \U}
&=& \frac{\partial}{\partial \U} \tr({\bar{\C}}\U{\bar{\C}}^T{\bar{\C}}\U{\bar{\C}}^T - 2\K{\bar{\C}}\U{\bar{\C}}^T + 2\delta{\bar{\C}}\U{\bar{\C}}^T) \\
&=& 2{\bar{\C}}^T{\bar{\C}}\U{\bar{\C}}^T{\bar{\C}} - 2{\bar{\C}}^T\K{\bar{\C}} + 2\delta{\bar{\C}}^T{\bar{\C}}
\; = \; \0,
\end{eqnarray*}
and obtain the solution
\begin{eqnarray*}
{\U^{\textrm{ss}}} &=& ({\bar{\C}}^T{\bar{\C}})^{\dag} ({\bar{\C}}^T\K{\bar{\C}}-\delta^{\textrm{ss}}{\bar{\C}}^T{\bar{\C}}) ({\bar{\C}}^T{\bar{\C}})^{\dag} \\
&=& ({\bar{\C}}^T{\bar{\C}})^{\dag}{\bar{\C}}^T\K{\bar{\C}}({\bar{\C}}^T{\bar{\C}})^{\dag} - \delta^{\textrm{ss}}({\bar{\C}}^T{\bar{\C}})^{\dag}{\bar{\C}}^T{\bar{\C}}({\bar{\C}}^T{\bar{\C}})^{\dag} \\
&=& {\bar{\C}}^\dag\K({\bar{\C}}^\dag)^T - \delta^{\textrm{ss}}({\bar{\C}}^T{\bar{\C}})^{\dag}.
\end{eqnarray*}
Similarly, we take the derivative of $f(\U, \delta)$ w.r.t.\ $\delta$ to be zero
\begin{eqnarray*}
\frac{\partial f(\U, \delta)} {\partial \delta}
&=& \frac{\partial}{\partial \delta} \tr(\delta^2\I_n - 2\delta\K + 2\delta\bar{\C}\U\bar{\C}^T)
\; = \; 2n\delta - 2\tr(\K) + 2\tr(\bar{\C}\U\bar{\C}^T) = 0,
\end{eqnarray*}
and it follows that
\begin{eqnarray*}
    {\delta^{\textrm{ss}}}
&=& \frac{1}{n} \bigg(\tr(\K) - \tr({\bar{\C}}\U^{\textrm{ss}}{\bar{\C}}^T) \bigg) \\
&=& \frac{1}{n}\bigg(\tr(\K) -
    \tr\Big({\bar{\C}}  {\bar{\C}}^\dag \K ({\bar{\C}}^\dag)^T {\bar{\C}}^T\Big)
    + \delta^{\textrm{ss}} \tr \Big( {\bar{\C}} ({\bar{\C}}^T{\bar{\C}})^{\dag}  {\bar{\C}}^T \Big) \bigg) \\
&=& \frac{1}{n}\bigg(\tr(\K) -
    \tr\Big({\bar{\C}}^T {\bar{\C}}  {\bar{\C}}^\dag \K ({\bar{\C}}^\dag)^T\Big)
    + \delta^{\textrm{ss}} \tr \Big( {\bar{\C}}  {\bar{\C}}^{\dag} \Big) \bigg) \\
&=& \frac{1}{n}\bigg(\tr(\K) -
    \tr\Big({\bar{\C}}^T  \K ({\bar{\C}}^\dag)^T\Big)
    + \delta^{\textrm{ss}} \rk(\bar{\C}) \bigg),
\end{eqnarray*}
and thus
\begin{eqnarray*}
\delta^{\textrm{ss}}
&=& \frac{1}{n- \rk(\bar{\C})}\Big( \tr(\K) - \tr\big({\bar{\C}}^\dag\K{\bar{\C}})\big) \Big).
\end{eqnarray*}


\subsection{Proof of Optimality} \label{sec:proof_ss_closed_form_optimal}

The Hessian matrix of $f(\U, \delta)$ w.r.t.\ $(\U, \delta)$ is
\begin{eqnarray*}
    \Hes
&=& \begin{bmatrix}
    \dfrac{\partial^2 f(\U,\delta)} {\partial\vect(\U) \partial\vect(\U)^T} &&
    \dfrac{\partial^2 f(\U,\delta)} {\partial\vect(\U) \partial\delta} \\[0.4cm]
    \dfrac{\partial^2 f(\U,\delta)} {\partial\delta \partial\vect(\U)^T} &&
    \dfrac{\partial^2 f(\U,\delta)} {\partial\delta^2}
  \end{bmatrix}
\;=\; 2\begin{bmatrix}
    ({\bar{\C}}^T{\bar{\C}}) \otimes ({\bar{\C}}^T{\bar{\C}}) & \vect({\bar{\C}}^T{\bar{\C}}) \\
    \vect({\bar{\C}}^T{\bar{\C}})^T & n
  \end{bmatrix}.
\end{eqnarray*}
Here $\otimes$ denotes the Kronecker product, and
$\vect ( \A )$ denotes the vectorization of the matrix $\A$ formed by stacking the columns of $\A$ into a single column vector.
For any $\X\in\RB^{c\times c}$ and $b\in\RB$, we let
\begin{eqnarray*}
    q(\X,b)
&=& \begin{bmatrix} \vect(\X)^T & b \end{bmatrix} \Hes
    \begin{bmatrix} \vect(\X) \\ b \end{bmatrix} \\
&=& \vect(\X)^T\big(({\bar{\C}}^T{\bar{\C}})\otimes({\bar{\C}}^T{\bar{\C}})\big)\vect(\X) + 2b\, \vect({\bar{\C}}^T{\bar{\C}})^T\vect(\X) + nb^2 \\
&=& \vect(\X)^T \vect\big( ({\bar{\C}}^T{\bar{\C}})\X({\bar{\C}}^T{\bar{\C}}) \big) + 2b\, \vect({\bar{\C}}^T{\bar{\C}})^T\vect(\X) + nb^2 \\
&=& \tr(\X^T{\bar{\C}}^T{\bar{\C}}\X{\bar{\C}}^T{\bar{\C}}) + 2b\,\tr({\bar{\C}}^T{\bar{\C}}\X) + nb^2 .
\end{eqnarray*}
Let ${\bar{\C}}\X{\bar{\C}}^T=\Y \in \RB^{n\times n}$. Then
\begin{eqnarray*}
    q(\X,b)
&=& \tr({\bar{\C}}\X^T{\bar{\C}}^T{\bar{\C}}\X{\bar{\C}}^T) + 2b \, \tr({\bar{\C}}\X{\bar{\C}}^T) + nb^2 \\
&=& \tr(\Y^T\Y) + 2b \, \tr(\Y) + n b^2 \\
&=& \sum_{i=1}^n\sum_{j=1}^n y_{ij}^2 + 2b\sum_{l=1}^n y_{l l} + n b^2 \\
&=& \sum_{i\neq j} y_{ij}^2 + \sum_{l=1}^n (y_{l l} + b)^2 \\
&\geq& 0,
\end{eqnarray*}
which shows that the Hessian matrix $\Hes$ is SPSD.
Hence $f(\U^{\textrm{ss}}, \delta^{\textrm{ss}})$ is the global minimum of $f$.


\subsection{Proof of Positive (Semi)Definite} \label{sec:proof_ss_closed_form_spsd}

We denote the thin SVD of ${\bar{\C}}$ by ${\bar{\C}}=\U_{\bar{\C}} \Si_{\bar{\C}}\V_{\bar{\C}}^T$
and let $\U_{\bar\C}^\perp$ be the orthogonal complement of $\U_{\bar\C}$.
The approximation is
\begin{eqnarray} \label{eq:ss_spsd}
\tilde{\K}
& = &{\bar{\C}}\U^{\textrm{ss}}{\bar{\C}}^T + \delta^{\textrm{ss}}\I_n
\; =\; {\bar{\C}} \big( {\bar{\C}}^\dag\K({\bar{\C}}^\dag)^T - \delta^{\textrm{ss}}({\bar{\C}}^T{\bar{\C}})^{\dag} \big) {\bar{\C}}^T + \delta^{\textrm{ss}}\I_n \nonumber \\
&=& {\bar{\C}} \big( {\bar{\C}}^\dag\K({\bar{\C}}^\dag)^T \big){\bar{\C}}^T + \delta^{\textrm{ss}} \big( \I_n - {\bar{\C}}({\bar{\C}}^T{\bar{\C}})^{\dag}{\bar{\C}}^T \big) \nonumber \\
&=& \U_{\bar\C} \U_{\bar\C}^T \K \U_{\bar\C} \U_{\bar\C}^T + \delta^{\textrm{ss}} \big( \I_n - \U_{\bar\C} \U_{\bar\C}^T \big)  \nonumber \\
&=& \U_{\bar\C} \U_{\bar\C}^T \K \U_{\bar\C} \U_{\bar\C}^T + \delta^{\textrm{ss}} \U_{\bar\C}^\perp (\U_{\bar\C}^\perp)^T .
\end{eqnarray}
The first term is SPSD because $\K$ is SPSD. The second term is SPSD if $\delta^{\textrm{ss}}$ is nonnegative.
We have
\begin{align*}
\delta^{\textrm{ss}}
& = \; \tr(\K) - \tr\big( {\bar{\C}}^\dag\K{\bar{\C}} \big)
\;= \; \tr(\K) - \tr\big( {\bar{\C}}{\bar{\C}}^\dag\K \big)
\;=\;\tr\big(\K - \U_{\bar{\C}}\U_{\bar{\C}}^T \K \big)\\
& = \;\tr \big[\U_{\bar{\C}}^\perp (\U_{\bar{\C}}^\perp)^T \K \big]
 = \;\tr \big[\U_{\bar{\C}}^\perp (\U_{\bar{\C}}^\perp)^T  \U_{\bar{\C}}^\perp (\U_{\bar{\C}}^\perp)^T \K \big]
=  \tr \big[\U_{\bar{\C}}^\perp (\U_{\bar{\C}}^\perp)^T \K \U_{\bar{\C}}^\perp (\U_{\bar{\C}}^\perp)^T\big]
\;\geq \; 0.
\end{align*}
To this end, we have shown SPSD.

Then we assume $\K$ is positive definite and prove that its approximation formed by SS is also positive definite.
Since $\U_{\bar{\C}}^\perp (\U_{\bar{\C}}^\perp)^T \K \U_{\bar{\C}}^\perp (\U_{\bar{\C}}^\perp)^T$ is SPSD,
its trace is zero only when it is all-zeros,
which is equivalent to $\K = \U_{\bar\C} \U_{\bar\C}^T \K \U_{\bar\C} \U_{\bar\C}^T$.
However, this cannot hold if $\K$ has full rank.
Therefore $\delta^{\textrm{ss}} > 0$ holds when $\K$ is positive definite.
When  $\K$ is positive definite, the $c\times c$ matrix $\U_{\bar\C}^T \K \U_{\bar\C}$ is also positive definite.
It follows from \eqref{eq:ss_spsd} that
\begin{eqnarray*}
\tilde\K \; = \;
{\bar{\C}}\U^{\textrm{ss}}{\bar{\C}}^T + \delta^{\textrm{ss}}\I_n
\; = \;
\left[
  \begin{array}{c c}
    \U_{\bar\C} & \U_{\bar\C}^\perp \\
  \end{array}
\right]
\left[
  \begin{array}{c c}
    \U_{\bar\C}^T \K \U_{\bar\C} & \0 \\
    \0 & \delta^{\textrm{ss}} \I_{n-c}\\
  \end{array}
\right]
\left[
  \begin{array}{c}
    \U_{\bar\C}^T \\
    (\U_{\bar\C}^\perp)^T \\
  \end{array}
\right].
\end{eqnarray*}
Here the block diagonal matrix is positive definite, and thus $\tilde\K$ is positive definite.


\section{Proof of Theorem~\ref{thm:superiority} and Theorem~\ref{thm:ss_nystrom_bound}} \label{sec:app:ss}

Directly analyzing the theoretical error bound of the SS model is not easy,
so we formulate a variant of SS called the inexact spectral shifting (ISS) model
and instead analyze the error bound of ISS.
We define ISS in Section~\ref{sec:proof:iss_pbs}
and prove Theorem~\ref{thm:superiority} and Theorem~\ref{thm:ss_nystrom_bound} in Section
\ref{sec:proof:superiority} and \ref{sec:proof:ss_pbs_bound}, respectively.
In the following we let $\Kss$ and $\Kiss$ be respectively the approximation formed by the two models.


\subsection{The ISS Model} \label{sec:proof:iss_pbs}

ISS is defined by
\begin{equation} \label{eq:def_iss_nystrom}
\Kiss \; = \; \bar{\C} \bar{\U} \bar{\C}^T + {\delta} \I_n,
\end{equation}
where ${\delta}>0$ is the spectral shifting term,
and $\bar{\C} \bar{\U} \bar{\C}^T$ is the prototype model of $\bar{\K} = \K - {\delta} \I_n$.
It follows from Theorem~\ref{thm:ss_closed_form} that ISS is less accurate than SS in that
\[
\big\| \K - \Kss \big\|_F
\; \leq \;
\big\| \K - \Kiss \big\|_F ,
\]
thus the error bounds of ISS still hold if $\Kiss$ is replaced by $\Kss$.

We first show how to set the spectral shifting term ${\delta}$.
Since $\bar{\C} \bar{\U} \bar{\C}^T$ is the approximation to $\bar\K$ formed by the prototype model,
it can holds with high probability that
\[
\big\|{\K} - \Kiss \big\|_F
\;= \;
\big\|{\K} - \bar\delta \I_n - \bar{\C} \bar{\U} \bar{\C}^T \big\|_F
\;= \;
\big\|\bar{\K} - \bar{\C} \bar{\U} \bar{\C}^T \big\|_F
\;\leq \;
\eta \big\|\bar{\K} - \bar{\K}_k \big\|_F
\]
for some error parameter $\eta$.
Apparently, for fixed $k$, the smaller the error $\|\bar{\K} - \bar{\K}_k\|_F$ is,
the tighter error bound the ISS has;
if $\|\bar{\K} - \bar{\K}_k\|_F \leq \|{\K} - {\K}_k\|_F$,
then ISS has a better error bound than the prototype model.
Therefore, our goal is to make $\|\bar{\K} - \bar{\K}_k\|_F $ as small as possible,
so we formulate the following optimization problem to compute ${\delta}$:
\[
\min_{\delta\geq 0} \; \big\| \bar{\K} - \bar{\K}_k \big\|_F^2;  \qquad \st \; \bar{\K} = \K - \delta \I_n.
\]
However, since $\bar{\K}$ is in general indefinite, it requires all of the eigenvalues of $\K$ to solve the problem exactly.
Since computing the full eigenvalue decomposition is expensive, we attempt to relax  the problem.
Considering that
\begin{eqnarray} \label{eq:delta_upper_bound}
\big\| \bar{\K} - \bar{\K}_k\big\|_F^2
\; = \; \min_{|\JM|=n-k} \sum_{j \in \JM} \big(\sigma_j (\K) - \delta\big)^2
\; \leq \; \sum_{j=k+1}^n \big(\sigma_j (\K) - \delta\big)^2,
\end{eqnarray}
we seek to minimize the upper bound of $\| \bar{\K} - \bar{\K}_k\|_F^2$,
which is the right-hand side of (\ref{eq:delta_upper_bound}),
to compute ${\delta}$, leading to
the solution
\begin{eqnarray} \label{eq:delta_solution}
\bar{\delta}
\; = \; \frac{1}{n-k}\sum_{j=k+1}^n \sigma_j (\K)
\; = \; \frac{1}{n-k} \bigg( \tr(\K) - \sum_{j=1}^k \sigma_j (\K) \bigg).
\end{eqnarray}
If we choose $\delta = 0$, then ISS degenerates to the prototype model.


\subsection{Proof of Theorem~\ref{thm:superiority}} \label{sec:proof:superiority}

Theorem~\ref{thm:ss_closed_form} indicates that ISS is less accurate than SS, thus
\[
\big\| \K - \Kss \big\|_F
\; \leq \;
\big\| \K - \Kiss \big\|_F
\; = \;
\big\| \K - \bar\delta \I_n - \bar\C {\bar\C}^\dag \bar\K ({\bar\C}^\dag )^T {\bar\C}^T\big\|_F
\; = \;
\big\| \bar\K - \bar\C {\bar\C}^\dag \bar\K ({\bar\C}^\dag )^T {\bar\C}^T\big\|_F .
\]
Theorem~\ref{thm:superiority} follows from the above inequality and the following theorem.

\begin{theorem} \label{thm:superiority_isspbs}
Let $\K$ be any $n \times n$ SPSD matrix,
$\bar{\delta}$ be defined in (\ref{eq:delta_solution}),
and $\bar{\K} = \K - \delta \I_n$.
Then for any $\delta \in (0, \bar{\delta}]$, the following inequality holds:
\[
\big\| \bar{\K} - \bar{\K}_k \big\|_F^2
\; \leq \;
\big\| \K - \K_k \big\|_F^2.
\]
\end{theorem}

\begin{proof}
Since the righthand side of (\ref{eq:delta_upper_bound}) is convex and $\bar{\delta}$ is the minimizer of the righthand of (\ref{eq:delta_upper_bound}),
for any $\delta \in (0, \bar{\delta}]$, it holds that
\[
\sum_{j=k+1}^n \big(\sigma_j (\K) - \delta\big)^2
\, \leq \,
\sum_{j=k+1}^n \big(\sigma_j (\K) - 0\big)^2
\, = \,
\big\| \K - \K_k\big\|_F^2.
\]
Then the theorem follows by the inequality in (\ref{eq:delta_upper_bound}).
\end{proof}


\subsection{Proof of Theorem~\ref{thm:ss_nystrom_bound}} \label{sec:proof:ss_pbs_bound}

Since $\| \K - \Kss \|_F \leq
\| \K - \Kiss \|_F$,
Theorem~\ref{thm:ss_nystrom_bound} follows from Theorem~\ref{thm:bound_iss_pbs}.

\begin{theorem} \label{thm:bound_iss_pbs}
Suppose there is a sketching matrix $\PP \in \RB^{n\times c}$ such that
for any $n \times n$ {\it symmetric} matrix $\S$ and target rank $k$ ($\ll n$),
by forming the sketch $\C = \S \PP$,
the prototype model satisfies the error bound
\[
\big\| \S - \C \C^\dag \S (\C^\dag)^T \C^T \big\|_F^2
\; \leq \; \eta \, \big\|\S - \S_k \big\|_F^2.
\]
Let $\K$ be any $n \times n$ SPSD matrix,
$\bar{\delta}$ defined in \eqref{eq:delta_solution} be the initial spectral shifting term,
and $\Kiss$ be the ISS approximation defined in \eqref{eq:def_iss_nystrom}.
Then
\[
\big\|\K - \Kiss \big\|_F^2
\: \leq \:
\eta \bigg( \big\|\K - \K_k \big\|_F^2 - \frac{\big[ \sum_{i=k+1}^n \lambda_{i} (\K) \big]^2}{n-k}  \bigg).
\]
\end{theorem}

\begin{proof}
The error incurred by SS is
\begin{eqnarray}
 \big\|\K - \tilde{\K}_{c}^{\textrm{ss}} \big\|_F^2
&= & \big\| \big(\bar{\K} + \bar{\delta} \I_n\big) - \big(\bar{\C} \bar{\U} \bar{\C}^T  + \bar{\delta} \I_n\big) \big\|_F^2
\; = \; \big\| \bar{\K} - \bar{\C} \bar{\U} \bar{\C}^T  \big\|_F^2 \nonumber \\
& \leq & \eta \big\| \bar{\K} - \bar{\K}_k \big\|_F^2
\;  = \; \eta \sum_{i=k+1}^n \sigma^2_i \big(\bar{\K} \big)
\; = \; \eta \sum_{i=k+1}^n \lambda_i \big(\bar{\K}^2 \big) . \nonumber
\end{eqnarray}
Here the inequality follows from the assumption.
The $i$-th largest eigenvalue of $\bar{\K}$ is $\lambda_i (\K) - \bar{\delta}$,
so the $n$ eigenvalues of $\bar{\K}^2$ are all in the set $\{ (\lambda_i (\K) - \bar{\delta})^2 \}_{i=1}^n$.
The sum of the smallest $n-k$ of the $n$ eigenvalues of $\bar{\K}^2$ must be less than or equal to the sum of any $n-k$ of the eigenvalues,
thus we have
\begin{eqnarray}
\sum_{i=k+1}^n \lambda_i \big(\bar{\K}^2 \big)
& \leq & \sum_{i=k+1}^n \Big( \lambda_i (\K) - \bar{\delta} \Big)^2 \nonumber \\
& \quad = & \sum_{i=k+1}^n \lambda_i^2 (\K) - 2 \sum_{i=k+1}^n \bar{\delta} \lambda_i (\K) + (n-k) (\bar{\delta})^2 \nonumber \\
& \quad = & \| \K - \K_k\|_F^2 - \frac{1}{n-k} \bigg[ \sum_{i=k+1}^n \lambda_i (\K) \bigg]^2 , \nonumber
\end{eqnarray}
by which the theorem follows.
\end{proof}


\section{Proof of Theorem~\ref{thm:delta}}

\begin{proof}
Let $\tilde{\K} = \Q (\Q^T \K)_k$,
where $\Q$ is defined in Line~\ref{alg:ss_nystrom:Y} in Algorithm~\ref{alg:ss_nystrom}.
\cite{boutsidis2011NOCshort} showed that
\begin{equation} \label{thm:delta:proof1}
\EB \| \K - \tilde{\K} \|_F^2 \; \leq \; (1 + k / l) \, \| \K - \K_k \|_F^2,
\end{equation}
where the expectation is taken w.r.t.\ the random Gaussian matrix $\Ome$.

It follows from Lemma~\ref{lem:delta} that
\[
\| \si_{\K} - \si_{\tilde{\K}} \|_2^2 \; \leq \; \| \K - \tilde{\K} \|_F^2,
\]
where $\si_{\K}$ and $\si_{\tilde{\K}}$ contain the singular values in a descending order.
Since $\tilde{\K}$ has a rank at most $k$, the $k+1$ to $n$ entries of $\si_{\tilde{\K}}$ are zero.
We split $\si_{\K}$ and $\si_{\tilde{\K}}$ into vectors of length $k$ and $n-k$:
\[
\si_{\K} \; = \; \left[
                   \begin{array}{c}
                     \si_{\K,k} \\
                     \si_{\K,-k} \\
                   \end{array}
                 \right]
\quad \textrm{ and } \quad
\si_{\tilde{\K}} \; = \; \left[
                   \begin{array}{c}
                     \si_{\tilde{\K},k} \\
                     \0 \\
                   \end{array}
                 \right]
\]
and thus
\begin{equation} \label{thm:delta:proof2}
\| \si_{\K,k} - \si_{\tilde{\K},k} \|_2^2 + \| \si_{\K,-k} \|_2^2 \; \leq \; \| \K - \tilde{\K} \|_F^2.
\end{equation}
Since $\| \si_{\K,-k} \|_2^2 = \| \K - \K_k\|_F^2$,
it follows from (\ref{thm:delta:proof1}) and (\ref{thm:delta:proof2}) that
\begin{equation} 
\EB \| \si_{\K,k} - \si_{\tilde{\K},k} \|_2^2 \; \leq \; \frac{k}{l} \, \| \si_{\K,-k} \|_2^2. \nonumber
\end{equation}
Since $\|\x\|_2 \leq \|\x\|_1 \leq \sqrt{k} \|\x\|_2$ for any $\x\in \RB^k$,
we have that
\begin{equation}
\EB \big\| \si_{\K,k} - \si_{\tilde{\K},k} \big\|_1 \; \leq \; \frac{k}{\sqrt{l}} \, \big\| \si_{\K,-k} \big\|_1. \nonumber
\end{equation}
Then it follows from (\ref{eq:delta_solution}) and Line~\ref{alg:ss_nystrom:delta_end} in Algorithm~\ref{alg:ss_nystrom} that
\begin{align}
& \EB \big|\bar{\delta} - \tilde{\delta}\big|
\; = \; \EB \Bigg[ \frac{1}{n-k} \bigg|  \sum_{i=1}^k \sigma_i (\K)  - \sum_{i=1}^k \sigma_i (\tilde{\K})  \bigg| \Bigg] \nonumber \\
& \qquad\qquad \leq \; \frac{1}{n-k} \, \EB \big\| \si_{\K,k} - \si_{\tilde{\K},k} \big\|_1
\; \leq \; \frac{k}{\sqrt{l}} \, \frac{1}{n-k}  \, \big\| \si_{\K,-k} \big\|_1
\; = \; \frac{k}{\sqrt{l}} \, \bar{\delta} . \nonumber
\end{align}
\end{proof}

\begin{lemma} \label{lem:delta}
Let $\A$ and $\B$ be $n\times n$ matrices
and $\si_{\A}$ and $\si_{\B}$ contain the singular values in a descending order.
Then we have that
\[
\| \si_{\A} - \si_{\B} \|_2^2 \; \leq \; \| \A - \B \|_F^2.
\]
\end{lemma}

\begin{proof}
It is easy to show that
\begin{eqnarray}
\| \A - \B \|_F^2
& = & \tr (\A^T \A) + \tr (\B^T \B) - 2 \tr(\A^T \B) \nonumber \\
& = & \sum_{i=1}^n \sigma^2_i (\A) + \sum_{i=1}^n \sigma^2_i (\B) - 2 \tr(\A^T \B) \label{eq:lem:delta1}.
\end{eqnarray}
We also have
\begin{eqnarray}
\tr(\A^T \B)
& \leq & \sum_{i=1}^n \sigma_i (\A^T \B)
\; \leq \; \sum_{i=1}^n \sigma_i (\A) \sigma_i (\B), \label{eq:lem:delta2}
\end{eqnarray}
where the first inequality follows from Theorem 3.3.13 of \citet{horn1991topics}
and the second inequality follows from Theorem 3.3.14 of \cite{horn1991topics}.
Combining (\ref{eq:lem:delta1}) and (\ref{eq:lem:delta2}) we have that
\begin{eqnarray}
\| \A - \B \|_F^2
& \geq & \sum_{i=1}^n \bigg( \sigma^2_i (\A) + \sigma^2_i (\B) - 2 \sigma_i (\A) \sigma_i (\B)  \bigg)
\; = \; \| \si_{\A} - \si_{\B} \|_2^2 , \nonumber
\end{eqnarray}
by which the theorem follows.
\end{proof}


\bibliography{spsd}

\end{document}